\documentclass[
]{article}

\usepackage[final,nonatbib]{neurips_2022}

\usepackage[utf8]{inputenc} 
\usepackage[T1]{fontenc}    
\usepackage{hyperref}       
\usepackage{url}            
\usepackage{booktabs}       
\usepackage{amsfonts}       
\usepackage{nicefrac}       
\usepackage{microtype}      
\usepackage{xcolor}         
\usepackage{tabularx}

\usepackage[style=alphabetic,sortcites]{biblatex}
\usepackage{mathtools,amssymb,amsthm}
\usepackage{epigraph}
\usepackage[nameinlink]{cleveref}
\usepackage{tikz-cd,extarrows}
\usepackage{graphicx}
\usepackage{subcaption}
\usepackage[shortlabels,inline]{enumitem}
\usepackage{listings}
\hypersetup{
colorlinks=true,
bookmarks=true,
bookmarksdepth=2
}
\usetikzlibrary{decorations.pathmorphing}
\usepackage{array}

\newcommand{\RR}{\mathbb{R}}

\newcommand{\WW}{\mathcal{W}}

\newcommand{\FF}{\mathcal{F}}

\newcommand{\nrm}[1]{\lvert #1 \rvert}

\DeclareMathOperator{\relu}{ReLU}
\DeclareMathOperator{\gelu}{GeLU}
\DeclareMathOperator{\leakyrelu}{LeakyReLU}
\DeclareMathOperator{\diag}{diag}
\DeclareMathOperator{\Acc}{Acc}
\DeclareMathOperator{\sort}{sort}
\DeclareMathOperator{\SoftSort}{SoftSort}
\DeclareMathOperator{\softmax}{softmax}
\DeclareMathOperator{\CKA}{CKA}
\DeclareMathOperator{\tr}{tr}
\DeclareMathOperator{\Mat}{Mat}
\DeclareMathOperator{\abs}{abs}

\numberwithin{equation}{section}

\theoremstyle{plain}
\newtheorem{theorem}[equation]{Theorem}
\newtheorem{lemma}[equation]{Lemma}
\newtheorem{corollary}[equation]{Corollary}
\newtheorem{proposition}[equation]{Proposition}

\theoremstyle{definition}
\newtheorem{definition}[equation]{Definition}

\theoremstyle{remark}
\newtheorem{remark}[equation]{Remark}

\newtheorem{example}[equation]{Example}

\newcommand{\omitt}[1]{}
\newcommand{\saveforcamready}[1]{}

\title{On the Symmetries of Deep Learning Models and their Internal Representations}

\author{%
Charles Godfrey$^{1,*}$, Davis Brown$^{1,*}$, Tegan Emerson$^{1,3,4}$, Henry Kvinge$^{1,2,3}$ \\
$^1$Pacific Northwest National Laboratory,\\
$^2$Department of Mathematics, University of Washington,\\
$^3$Department of Mathematics, Colorado State University,\\
$^4$Department of Mathematical Sciences, University of Texas, El Paso\\
$^*$Equal contribution\\
\texttt{first.last@pnnl.gov} \\
}

\begin{document}

\maketitle

\begin{abstract}

Symmetry is a fundamental tool in the exploration of a broad range of
complex systems. In machine learning symmetry has been explored in both models
and data. In this paper we seek to connect the symmetries arising from the architecture of a family of models with the symmetries of that family's internal representation of data. We do this by calculating a set of fundamental symmetry groups, which we call the {\emph{intertwiner groups}} of the model. 
We connect intertwiner groups to a model's internal representations of data through a range of experiments that probe similarities between hidden states across models with the same architecture. Our work suggests that the symmetries of a network are propagated into the symmetries in that network's representation of data, providing us with a better understanding of how architecture affects the learning and prediction process. Finally, we speculate that for ReLU networks, the intertwiner groups may provide a justification for the common practice of concentrating model interpretability exploration on the activation basis in hidden layers rather than arbitrary linear combinations thereof.

\end{abstract}


\section{Introduction}

Symmetry provides an important path to understanding across a range of disciplines. This principle is well-established in mathematics and physics, where it has been a fundamental tool (e.g., Noether's Theorem \cite{noether}). Symmetry has also been brought to bear on deep learning problems from a number of directions. There is, for example, a rich research thread that studies symmetries in data types that can be used to inform model architectures. The most famous examples of this are standard convolutional neural networks which encode the translation invariance of many types of semantic content in natural images into a network's architecture. In this paper, we focus on connections between two other types of symmetry associated with deep learning models: the symmetries in the learnable parameters of the model and the symmetries across different models' internal representation of the same data. 

The first of these directions of research starts with the observation that in modern neural networks there exist models with different weights that behave identically on all possible input. We show in \cref{sec:notation} that at least some of these equivalent models arise because of symmetries intrinsic to the nonlinearities of the network. We call these groups of symmetries, each of which is attached to a particular type of nonlinear layer $\sigma$ of dimension $n$, the {\emph{intertwiner groups}} $G_{\sigma_n}$ of the model. These intertwiner groups come with
a natural action on network weights for which the realization map to function space of \cite{jacotNeuralTangentKernel2020} is invariant (\cref{lem:comm-w-sig}). As such they provide a unifying framework in which to discuss well-known weight space symmetries such as permutation of neurons \cite{breaWeightspaceSymmetryDeep2019} and scale invariance properties of
\(\relu\) and batch-norm \cite{nairRectifiedLinearUnits2010, ioffeBatchNormalizationAccelerating2015}. 




Next, we tie our intertwiner groups to the symmetries between different model's internal representations of the same data. We do this through a range of experiments that we describe below; each builds on a significant recent advance in the field. 

\textbf{Neural stitching with intertwiner groups:}
The work of
\cite{Bansal2021RevisitingMS,csiszarikSimilarityMatchingNeural2021,lencUnderstandingImageRepresentations2015}
demonstrated that one can take two trained
neural networks, say A and B, with the same architecture but trained from
different randomly initialized weights, and connect the early layers of network
A to the later layers of network B with a simple ``stitching'' layer and achieve
negligible loss in prediction accuracy. This was taken as evidence of the similarity of strong model's representations of data. 
Though the original experiments use a fully
connected linear layer to stitch, we provide theoretical evidence in
\cref{thm:min-stitch} that much less is needed. Indeed, we show that the intertwiner group (which has far fewer parameters in general) is the minimal
viable stitching layer to preserve accuracy. We conduct experiments stitching networks
at \(\relu\) activation layers with the stitching layer restricted to elements of the group \(G_{\relu}\) showing in \cref{fig:stitch-penalties} that one
can stitch CNNs on CIFAR-10 \cite{Krizhevsky09learningmultiple} with only elements of
\(G_{\relu}\) incurring less than \(\approx 10 \%\) accuracy penalty at most activation layers. This is surprisingly close to the losses found when one allows for a much more expressive linear layer to be used to stitch two networks together. However, we see that there remains a significant gap between the stitching accuracies obtained using \(G_{\relu}\) and fully connected linear layers; this provides independent confirmation of earlier findings that neurons of networks trained with different random seeds (i.e. with independent initializations and different random batches) are not simply permutations of each other \cite{liConvergentLearningDifferent2015,wangUnderstandingLearningRepresentations2018}. It is also consistent with observed phenomena such as distributed representations in hidden features \cite[\S 15.4]{Goodfellow-et-al-2016} and perhaps also polysemantic neurons \cite{olah2020zoom}. 

\begin{figure}[h]
\centering
  \includegraphics[
    width=0.9\linewidth
    ]{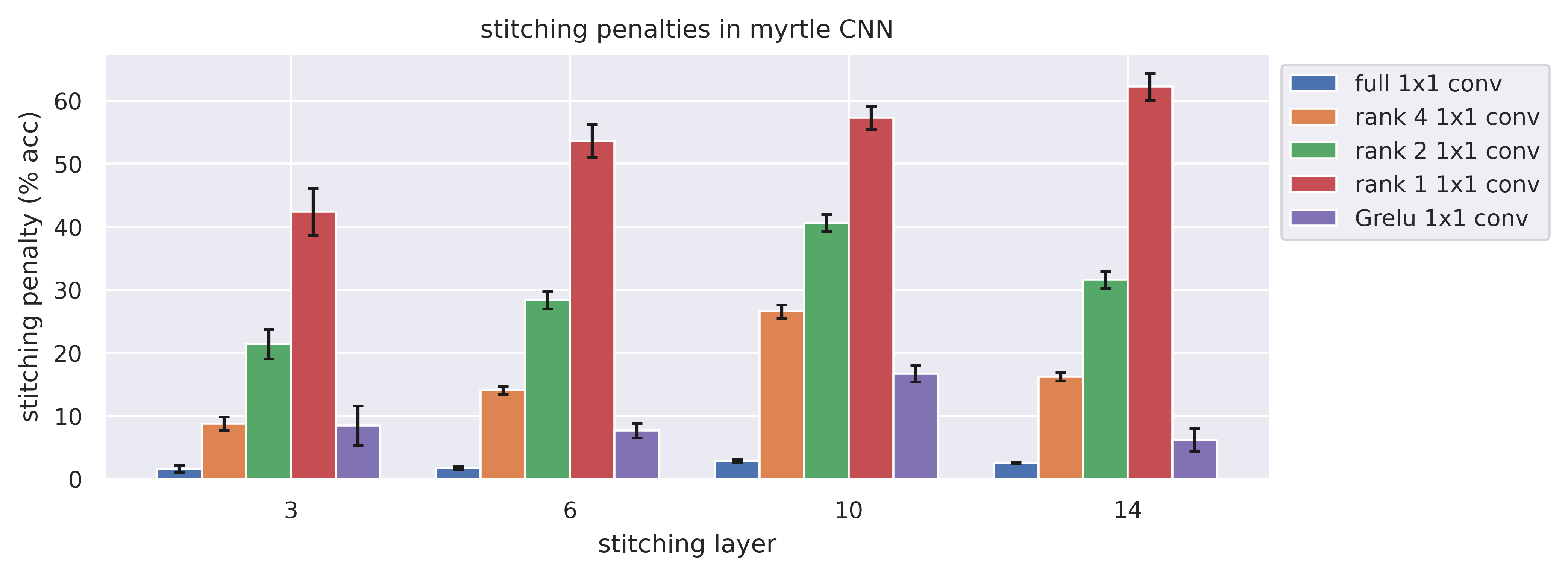}
  \caption[]{Full, reduced rank and \(G_{\relu}\)  1-by-1 convolution stitching penalties \eqref{eq-stitching-penalty} for Myrtle CNNs \cite{pageHowTrainYour2018} on CIFAR-10. Confidence intervals were obtained by evaluating stitching penalties for 32 pairs models trained with different random seeds. The accuracy of the models was  \(91.3 \pm 0.2 \%\). }\label{fig:stitch-penalties}
\end{figure}

\textbf{Representation dissimilarity measures for \(G_{\relu}\):} In \cref{sec:shapemetrics} we
present two statistical dissimilarity measures, \(G_{\relu}\)-Procrustes and \(G_{\relu}\)-CKA, for \(\relu\)-activated hidden
features in different networks, say \(A\) and \(B\). Our measures are
counterparts of orthogonal Procrustes distance (see e.g.
\cite{ding2021grounding}) and Centered Kernel Alignment (CKA)\footnote{with
linear or Gaussian radial basis function (RBF) kernel.} \cite{kornblithSimilarityNeuralNetwork2019}
respectively, which are invariant to orthogonal
transformations, and are maximized when the hidden features of networks A and B agree up to orthogonal transformations. In contrast \(G_{\relu}\)-Procrustes and \(G_{\relu}\)-CKA are invariant to \(G_{\relu}\) transformations, and are maximized when hidden features
agree up to \(G_{\relu}\) transformations. We compare and contrast our measures with their orthogonal counterparts, as well as with stitching experiment results. \Cref{fig:wreath-cka-resnet} shows a comparison of \(G_{\relu}\) and orthogonal CKA measures.

\begin{figure}
    \begin{subfigure}[h]{0.5\linewidth}
        \centering
        \includegraphics[width=\linewidth]{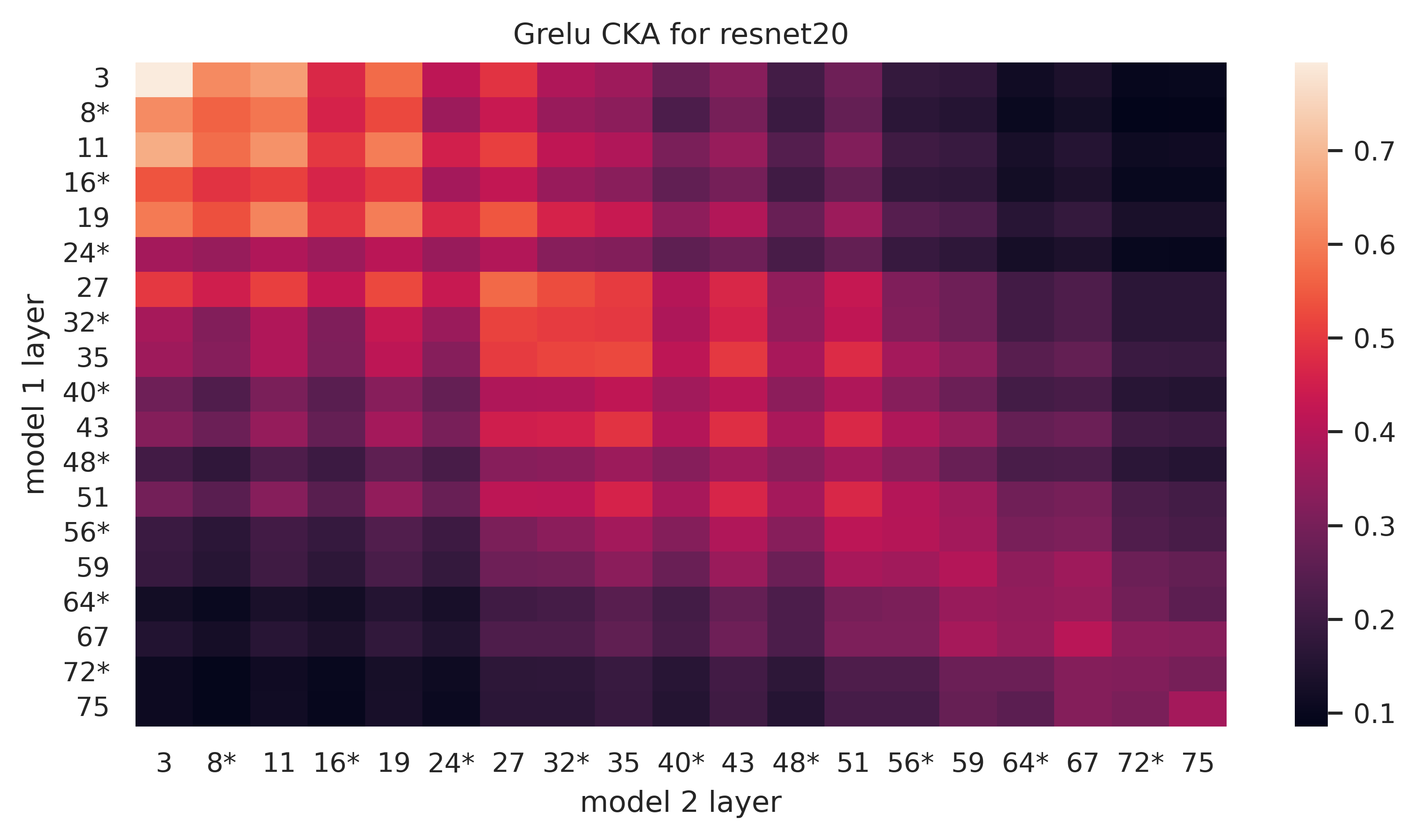}
    \end{subfigure}
    \begin{subfigure}[h]{0.5\linewidth}
        \centering
        \includegraphics[width=\linewidth]{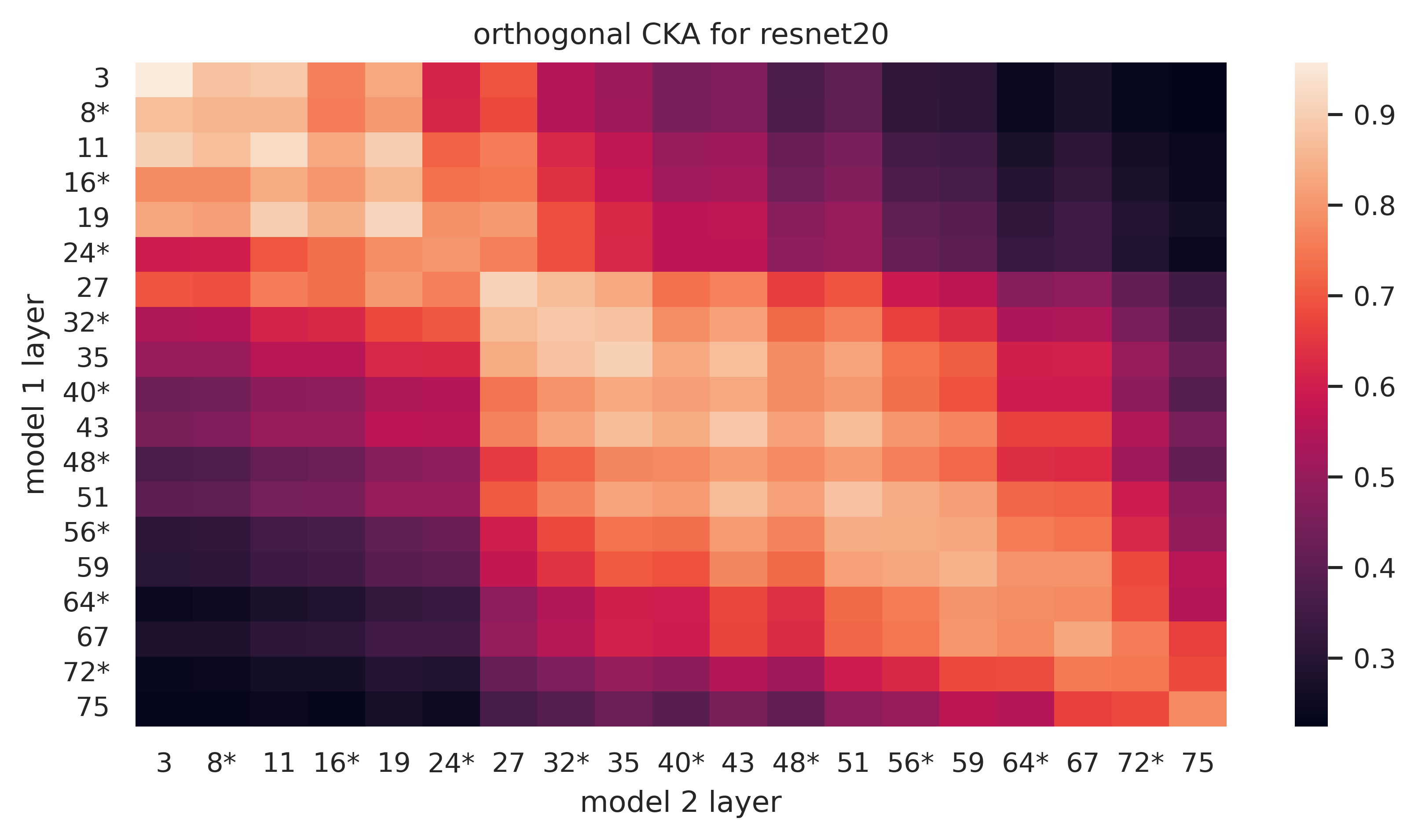}
    \end{subfigure}
    \caption{\(G_{\relu}\)-CKA and orthogonal CKA between layers of two ResNet20s trained on CIFAR-10. Results averaged over 16 pairs of models trained with different random seeds 
    . Layers marked with `*' occur inside residual blocks (\cref{rmk:rn-failure}). For further details see \cref{sec:shapemetrics}.
    }\label{fig:wreath-cka-resnet}
\end{figure}

\textbf{Impact of activation functions on interpretability:} An intriguing
finding from \cite[\S 3.2]{Bau2017NetworkDQ} was that individual neurons
are more interpretable than random linear combinations of neurons. Our results on intertwiner groups (\cref{thm:stabilizer1}) predict that this is a particular feature of ReLU networks. Indeed, \cref{fig:net-d1} shows that in the absence of \(\relu\)
activations interpretability does not decrease when one moves from individual neurons to linear combinations of neurons --- \cref{sec:net-dissect} describes this experiment in further detail. This result suggests that intertwiner groups provide a theoretical justification for the explainable AI community's focus on individual neuron activations, rather than linear combinations thereof \cite{erhan2009visualizing, zeiler2014visualizing, zhou2014object,
yosinski2015understanding,  na2019discovery}, but that this justification is only valid for layers with certain activation functions.

\begin{figure}
  \centering
  \includegraphics[width=4.2in]{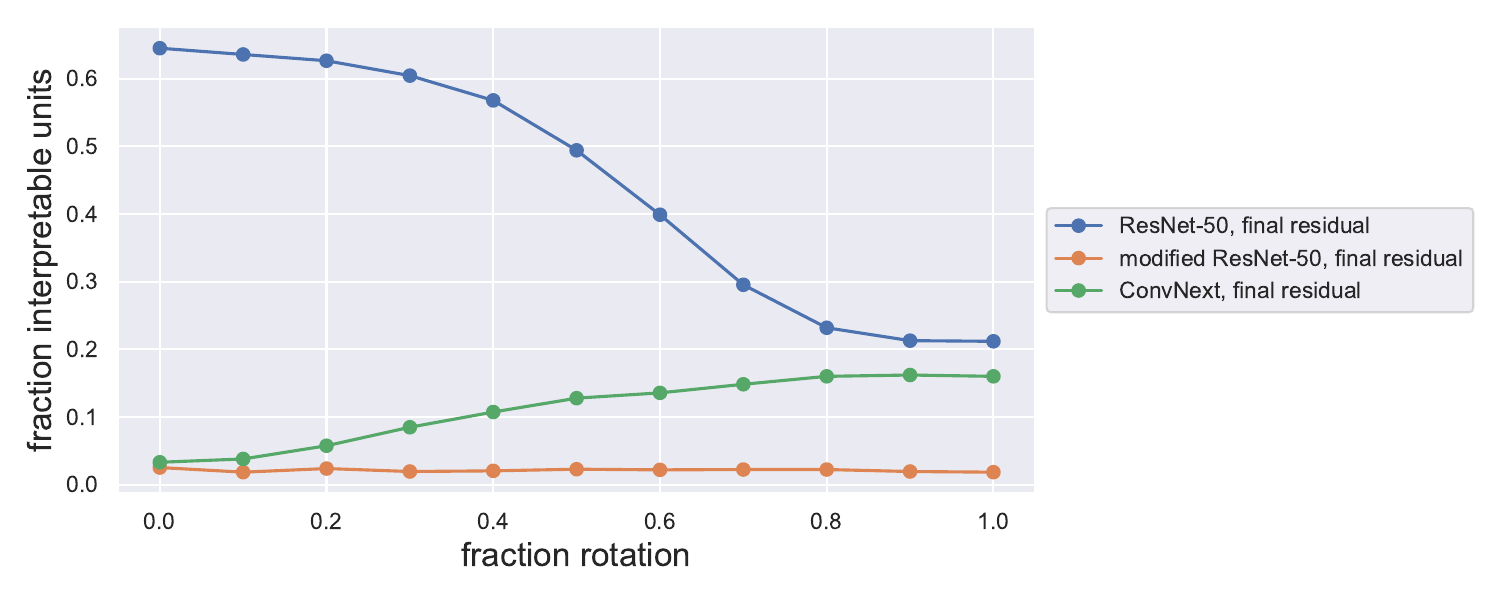}
  \caption[]{Fraction of network dissection interpretable units under rotations of
  the representation basis for a ResNet-50, as well as a modified ResNet-50 and a ConvNeXt model (both without an activation function on the residual output). \Cref{sec:net-dissect} contains details and further discussion.}\label{fig:net-d1}
\end{figure}

Taken together, our experiments provide evidence that a network's symmetries (realized through intertwiner groups) propagate down to symmetries of a model's internal representation of data. Since understanding how different models process the same data is a fundamental goal in fields such as explainable AI and the safety of deep learning systems, we hope that our results will provide an additional lens under which to examine these problems.  

\section{Related work}
\label{sec:related}

The research on symmetries of neural networks is extensive, 
hence we aim to provide a
representative sample knowing it will be incomplete.
\cite{Goodfellow-et-al-2016,breaWeightspaceSymmetryDeep2019,freemanTopologyGeometryHalfRectified2017,yiPositivelyScaleInvariantFlatness2019}
study the effect of weight space symmetries on the
loss landscape.
On the other hand, \cite{badrinarayananUnderstandingSymmetriesDeep2015,Goodfellow-et-al-2016,kuninNeuralMechanicsSymmetry2021,mengGSGDOptimizingReLU2019}
study the effect of weight space symmetries on training dynamics, while
\cite{Goodfellow-et-al-2016,rolnickReverseengineeringDeepReLU2020} show that
weight space symmetries pose an obstruction to model identifiability.

Neural stitching was introduced as a means of comparing learned representations
between networks in \cite{lencUnderstandingImageRepresentations2015}. In
\cite{Bansal2021RevisitingMS}
it was shown to have intriguing connections with the ``Anna Karenina
'' (high performance models share similar internal representations of
data) and ``more is better'' (stitching later layers of a weak model to
early layers of a model trained with more
data/parameters/epochs can improve performance) phenomena.
\cite{csiszarikSimilarityMatchingNeural2021} considered constrained stitching
layers by restricting the rank of the stitching matrix or by introducing an \(\ell_1\)
sparsity penalty. Our methods are distinct in that we explicitly optimize over
the intertwiner group for ReLU nonlinearities (permutations and scalings). Both
\cite{Bansal2021RevisitingMS, csiszarikSimilarityMatchingNeural2021}
compare their stitching results with statistical dissimilarity measures such as CKA. Our \(G_{\relu}\)-Procrustes measure is a close relative of the
permutation Procrustes distance introduced in
\cite{williamsGeneralizedShapeMetrics2022}, and our \(G_{\relu}\)-CKA is a an instance of CKA \cite{kornblithSimilarityNeuralNetwork2019}
in which the kernel is taken to be \(\max \{x_1 \cdot y_1, \dots, x_d \cdot y_d\}\).

\cite{liConvergentLearningDifferent2015} developed algorithms for obtaining a
permutation to align neurons, and \cite{wangUnderstandingLearningRepresentations2018}
introduced
\emph{neuron activation subspace matching} and used it to study similarity of hidden
feature representations.
\cite{tatroOptimizingModeConnectivity2020,git-rebasin,entezari2022the} all aligned neurons with
permutations with the goal of obtaining low-loss paths between the weights of networks. The objectives for neuron alignment used in these works include 
maximizing correlation (\cite{liConvergentLearningDifferent2015,tatroOptimizingModeConnectivity2020,git-rebasin}), maximizing a ``match'' (as defined in \cite{wangUnderstandingLearningRepresentations2018}), simulated annealing search algorithms (\cite{entezari2022the}), direct alignment of weights via a bilinear assignment problem and a ``straight-through estimator'' of back-propagated training loss (\cite{git-rebasin}). Each of these is distinct from our method,
which explicitly seeks a permutation minimizing training loss and searches for one using standard convex relaxation methods for permutation optimization.

Approaches to deep learning interpretability sometimes assume that the activation basis
is special \cite{erhan2009visualizing, zeiler2014visualizing, zhou2014object,
yosinski2015understanding,  na2019discovery}. Studying individual neurons rather
than linear combinations of neurons significantly reduces the complexity of low-level approaches to the understanding
of neural networks \cite{elhage2021mathematical}. In tension with this, many
different projections of hidden layer activations appear to be semantically
coherent \cite{szegedy2013intriguing}. However, \cite{Bau2017NetworkDQ} found evidence that
the hidden feature vectors closer to the coordinate basis align more with human
concepts than vectors sampled uniformly from the unit sphere.

\section{The symmetries of nonlinearities}
\label{sec:notation}

Let $\Mat_{n_1, n_0}(\RR)$ be the algebra of all $n_0 \times n_1$ real matrices and $GL_n(\mathbb{R})$ be the group of all invertible $n \times n$
matrices. Let $\sigma:\mathbb{R} \rightarrow \mathbb{R}$ be a
continuous function. For any $n \in \mathbb{N}$, we can build a
{\emph{nonlinearity}} $\sigma_n$ from $\mathbb{R}^n$ to $\mathbb{R}^n$ by
applying $\sigma$ coordinatewise, i.e., $\sigma_n(x_1,\dots,x_n) =
(\sigma(x_1),\dots, \sigma(x_n))$. Fix some \(k > 1 \) and for each $1 \leq i < k$ let  $\ell_i: \mathbb{R}^{n_{i-1}} \rightarrow \mathbb{R}^{n_{i}}$ be the composition of an affine layer and  a nonlinear layer, so that $\ell_i(x) := \sigma_{n_i}(W_ix + b_i)$, and let  $\ell_k(x) := W_kx + b_k$. Here $W_i \in \Mat_{n_i, n_{i-1}}(\RR)$ and $b_i \in \RR^{n_i}$ are the weights and bias of layer \(i\) respectively. We define $f: \mathbb{R}^{n_0} \rightarrow \mathbb{R}^{n_k}$ to be the neural network $f = \ell_k \circ \dots \circ \ell_1$. For each $1 \leq i < k-1$ we can then decompose $f$ as $f = f_{>i} \circ f_{\leq i}$ where 
\begin{equation*}
f_{\leq i} = \ell_i \circ \dots \circ \ell_1 \quad \text{and} \quad f_{> i} = \ell_k \circ \dots \circ \ell_{i+1}.
\end{equation*}
We define 
\begin{equation*}
    W:= (W_i, b_i \, | \, i = 1, \dots, k) \quad \text{and} \quad \WW := \prod_{i=1}^k (\Mat_{n_i,
n_{i-1}}(\RR) \times \RR^{n_i})
\end{equation*}
where the former is the collection of all weights of $f$ and the latter is the space of all possible weights for a given architecture. When we want to emphasize the dependence of $f$ on weights $W$, we write $f(-, W)$ (and similarly \(f_{\leq
i}(-, W), f_{>i}(-, W)\)). 

One of the topics this work will consider is vector space bases for $f$'s hidden spaces $\mathbb{R}^{n_i}$, for $1 \leq i \leq k-1$. We will investigate the legitimacy of
analyzing features \(f_{\leq i}(D)\) for dataset $D \subset \mathbb{R}^{n_0}$ with respect to the {\emph{activation basis}} for $\mathbb{R}^{n_i}$ which is simply the usual coordinate basis, $e_1, \dots, e_{n_i}$ where $e_j = [\delta_{j\ell}]^{n_i}_{\ell=1}$ is naturally parameterized by individual neuron activations. Note that $\mathbb{R}^{n_i}$ has an infinite number of other possible bases that could be chosen.



\subsection{Intertwiner Groups}

For any $0 \leq i < k$, elements of $GL_{n_i}(\mathbb{R})$ can be applied to the hidden activation space $\mathbb{R}^{n_i}$ both before and after the nonlinear layer $\sigma_{n_i}$. We define
\begin{equation*}
    G_{\sigma_{n_i}} := \{A \in GL_{n_i}(\RR) \, | \, \text{ there exists a } B\in
  GL_{n_i}(\RR)  \text{ such that } \sigma_{n_i} \circ A = B \circ \sigma_{n_i} \}.
\end{equation*}
Informally, we can understand $G_{\sigma_{n_i}}$ to be the set of all invertible linear transformations whose action on $\mathbb{R}^{n_i}$ prior to the nonlinear layer $\sigma_{n_i}$ has an equivalent invertible  transformation after $\sigma_{n_i}$. This is an instance of the common procedure of understanding a function by understanding those operators that commute with it. For any $A \in GL_{n_i}(\mathbb{R})$, we can write $\sigma(A)$ for the $n_i \times n_i$ matrix formed by applying $\sigma$ to all entries in $A$. 

\begin{lemma}
  \label{lem:intertwiner}
  Suppose \(\sigma(I_n) \) is invertible and for each $A \in GL_{n}(\RR)$ define $\phi_\sigma(A) = \sigma(A) \sigma(I_n)^{-1}$. Then \(G_{\sigma_n}\) is a group, $\phi_\sigma: G_{\sigma_n} \to GL_{n}(\RR)$ is a homomorphism  and \(\sigma_n \circ A = \phi_{\sigma}(A) \circ \sigma_n\). 
\end{lemma}

We defer all proofs to \cref{appendix-proofs}. We include concrete examples of $\sigma$ for small $n_i$ there as well. 


\begin{definition}
  \label{def:intertwiner}
  When the hypotheses of  \cref{lem:intertwiner} are satisfied (namely,
  \(\sigma(I_n) \) is invertible) we call \(G_{\sigma_n}\) the
  \textbf{intertwiner group of the activation \(\sigma_n\)}. We denote the image
  of the homomorphism $\phi_\sigma$ as
  \(\phi_{\sigma}(G_{\sigma_n})\).
\end{definition}

The intertwiner group $G_{\sigma_n}$ and $\phi_{\sigma}$ are concretely described for a range of activations in \cref{fig-intertwiner-groups} --- the last two examples motivate the generality of \cref{def:intertwiner}. Note also that in both of those cases \(A \mapsto \phi_\sigma(A)\) is \emph{not} a homomorphism on all of \(GL_{n}(\RR)\), but \emph{is} a homomorphism when restricted to the appropriate subgroup \(G_{\sigma_n}\). While a substantial part of \cref{fig-intertwiner-groups} can be found scattered in prior work, our calculations  in \cref{sec:calc-intertwiner-grps} deal with the different cases of \cref{fig-intertwiner-groups} in a uniform way, by what amounts to an algorithm that compute $G_{\sigma_n}$ and $\phi_{\sigma}$ given in terms of elementary properties of any (reasonable) activation function \(   \sigma \).\footnote{We defer further discussion of and references to this prior work to \cref{sec:calc-intertwiner-grps}.} As design of activation functions remains an active industry (for example \cite{elhage2022solu}), our techniques for computing $G_{\sigma_n}$ could be useful in future studies of network symmetries.



  \begin{table}
    \centering
    \begin{tabular}{ m{12em}  m{7cm} m{2cm} } 
    \toprule
         Activation &     $G_{\sigma_n}$ &     $\phi_\sigma(A)$ \\
    \toprule
    
    \(\sigma(x) = x\) (identity) & \(GL_n(\RR)\) & \(A\)\\
    \midrule[.05pt]
    \(\sigma(x) = \frac{e^x}{1+e^x}\) & $\Sigma_n$ & $A$\\
    \midrule[.05pt]
    $\sigma(x) = \relu(x)$ & Matrices 
    \(P D\), where \(D\) has positive entries & $A$\\
    \midrule[.05pt]
    $\sigma(x) = \text{LeakyReLU}(x)$ & Same as \(\relu \) as long as negative slope \(\neq 1\) & $A$ \\
    \midrule[.05pt]
    \(\sigma(x) = \frac{1}{\sqrt{2 \pi}} e^{-\frac{x^2}{2}}\) (RBF) & Matrices 
    \(P D\), where  \(D\) has entries in \(\{\pm 1\}\) & 
    $\mathrm{abs}(A)$\\
    \midrule[.05pt]
    \(\sigma(x) = x^d\) (polynomial) & Matrices \(P D\), where \(D\) has non-zero entries 
    & \(A^{\odot  d}\)
    \end{tabular}
    \caption{Explicit descriptions of $G_{\sigma_n}$ and $\phi_{\sigma}$ for six different activations. Here $P \in \Sigma_n$ is a permutation matrix, $D$ is a diagonal matrix, $\mathrm{abs}$ denotes the entrywise absolute value, and $A^{\odot d}$ denotes the entrywise $d$th power.}\label{fig-intertwiner-groups} 
  \end{table}


The following theorem shows that the activation basis is intimately related to the intertwiner group of \( \relu\): \( G_{\relu} \) admits a natural group-theoretic characterization in terms of the rays spanned by the activation basis, and dually the rays spanned by the activation basis can be recovered from \( G_{\relu} \) . While both its statement and proof are elementary, our interest in this theorem lies in the question of whether it could \emph{potentially} provide theoretical justification for focusing model interpretation studies on individual activations. We investigate this question further in \cref{sec:net-dissect}. 

\begin{theorem}
  \label{thm:stabilizer1}
  The group \(G_{\relu_n}\) is precisely the stabilizer of the set of
  rays \(\{ \RR_{\geq 0} e_i \subset \RR^n | i = 1, \dots, n \}\). Moreover if 
  \(\RR_{\geq 0} v_1, \dots, \RR_{\geq 0} v_N \subseteq \RR^n  \)
  is a finite set of rays stabilized by \(G_{\relu}\), then for each \( v_i = [v_{i1}, \dots, v_{in}]^T \), it must be that \(v_{ij} = 0 \) for all but one \(j \in \{1, \dots, n\}\). Equivalently up to multiplication by a positive scalar every \(v_i\) is of the form \(\pm e_j\) for some \(j\).
  
\end{theorem}

\subsection{Weight space symmetries}
\label{sect-weight-space}

The intertwiner group is also a natural way to describe the weight space symmetries of a neural network. We denote by \(\FF \subseteq C(\RR^{n_0}, \RR^{n_k})\) the space of continuous
functions that can be described by a network with the same architecture as
\(f\). As described in \cite{jacotNeuralTangentKernel2020} there is a
\textbf{realization map} \(\Phi: \WW \to \FF\) mapping weights \(W \in \WW\) to
the associated function \(f \in \FF\). $\Phi$ arises because there are generally multiple sets of weights that yield the same function. We will show that \(\Phi\) is invariant
with respect to an action of the intertwiner groups on \(\WW\) so that intertwiner groups form a set of ``built-in'' weight space symmetries of \(f\). This result, which encompasses phenomena including permutation
symmetries of hidden neurons, is well known in many particular cases (e.g., \cite[\S 8.2.2]{Goodfellow-et-al-2016}, \cite[\S 3]{breaWeightspaceSymmetryDeep2019}, \cite[\S 2]{freemanTopologyGeometryHalfRectified2017}, \cite[\S 3]{mengGSGDOptimizingReLU2019}, \cite[\S 3, A]{rolnickReverseengineeringDeepReLU2020}). From \Cref{lem:comm-w-sig} we can also derive corollaries regarding symmetries of the loss landscape --- these are included in \cref{sec:sym-ll}.
\begin{proposition}
  \label{lem:comm-w-sig}
  Suppose $A_i \in G_{\sigma_{n_i}}$ for $1 \leq i \leq k-1$, and let 
  \begin{equation*}
  W' = (A_1 W_1, A_1b_1, A_2 W_2 \phi_{\sigma}(A_1^{-1}), A_2 b_2 , \dots,
  W_{k}\phi_{\sigma}(A_{k-1}^{-1}), b_{k})
  \end{equation*}
  Then, as functions, for each $m$
  \begin{equation}
  \label{eq:equivariant-thing}
    f_{\leq m}(x, W') = \phi_\sigma(A_m) \circ f_{\leq m}(x, W) \text{  and  }f_{> m}(x, W') = f_{>m}(x, W) \circ \phi_{\sigma}(A_m)^{-1},
  \end{equation}
  In particular, \(f(x, W') = f(x, W)\) for all \(x \in \RR^{n_0}\).  Equivalently, we
  have \(\Phi(W') = \Phi(W) \in \FF\).
\end{proposition}
\begin{remark}
    \label{rmk:rn-failure}
    We will show in \cref{appendix-proofs} that the statement of this theorem must be modified  if the architecture of \(f\) contains residual connections. By placing suitable restrictions on the matrices \(A_i\)\footnote{Namely, that \(A_l = A_m\) if layers \(l\) and \(m\) are joined by a sequence of residual connections.} we can recover a form of \cref{eq:equivariant-thing} provided \(m\) occurs at the end of a residual block. However, there doesn't seem to be a way to obtain such an identity when \( m\) occurs \emph{inside} a residual block; we see empirical evidence consistent with this point in \cref{fig:rn-pro,fig:wreath-cka-resnet,fig:stitch-penalties-resnet} below.
\end{remark}

\subsection{A ``sanity test'' for intertwiners}
\label{sec:sanity-rotations}

To test \cref{lem:comm-w-sig} with a simple experiment, we begin with a Myrtle CNN \cite{pageHowTrainYour2018} network\footnote{This is a simple 5-layer CNN, with no residual connections described further in \cref{sect-experimental-details}.}  trained for 50 epochs on the CIFAR-10 dataset, fix a \emph{pre-activation layer} \(l\), and apply a transformation \(A \) to the weights \( W_l \) and biases \( b_l \) to obtain \( AW_l\) and \(  Ab_l\) (we only act on channels, hence in practice this is implemented by an auxiliary 1-by-1 convolution layer). We consider 2 choices of \(  A\):
\begin{enumerate*}[(i)]
    \item a random element of \(G_{\relu} \), where \(P\) is a random permutation and the diagonal entries of \( D \) are sampled from a lognormal distribution, and
    \item a random orthogonal matrix, obtained as the ``\(  Q\)'' in a \(QR\)-decomposition of a random matrix \(X \) with independent standard normal entries.
\end{enumerate*}

Next, we freeze layers up to and including \(l\) and finetune the later layers for another 50 epochs. We refer to the difference between the validation accuracy before and after applying the transformation \(A\) and finetuning as a \textbf{rotation penalty}. Based on \cref{lem:comm-w-sig}, when \(   A \in G_{\relu} \) the network should be able to recover reasonable performance even with the transformed features --- for example, by updating \(W_{l+1} \) to \(W_{l+1}\phi_\sigma(A)^{-1}\). On the other hand, with probability 1 there is no matrix \( B\) such that updating \(W_{l+1} \) to \(W_{l+1}B\) counteracts the effect of an orthogonal rotation \(A\) on all possible input. We see that this is indeed the case in \cref{fig:rotation-penalties}: transforming by \(   A \in G_{\relu} \) rather a random orthogonal matrix results in significantly smaller rotation penalties.

\begin{figure}
    \centering
    \includegraphics[width=0.9\linewidth]{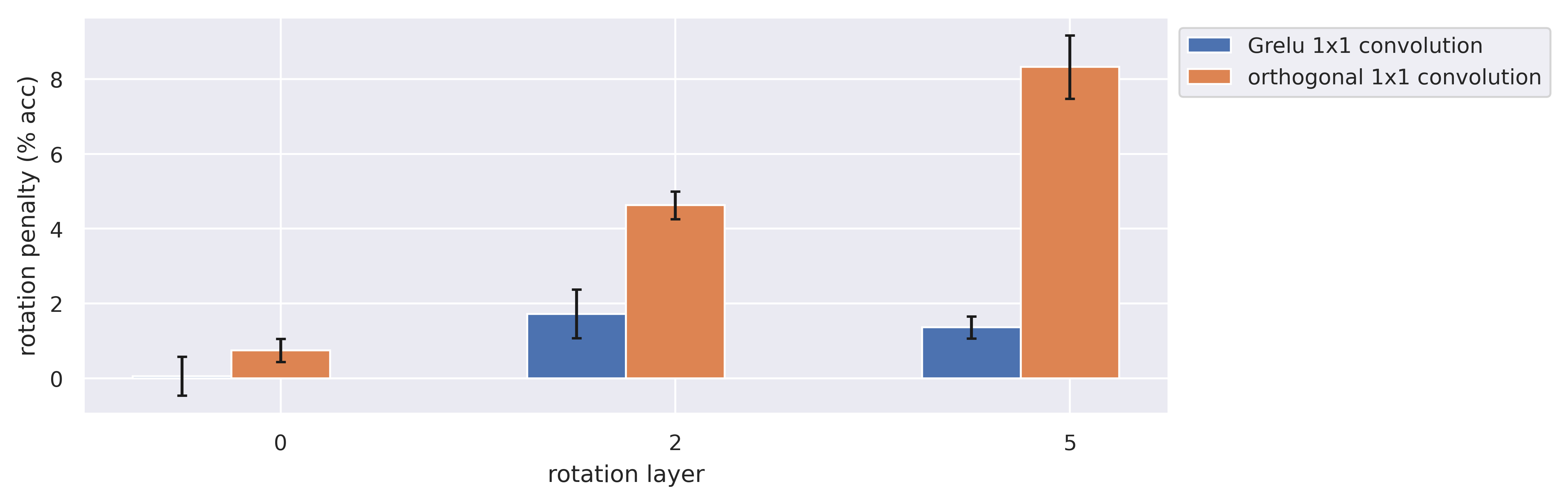}
    \caption{Rotation penalties for Myrtle CNNs on the CIFAR-10 dataset.  Confidence intervals were obtained by performing 10 independent trials of the experiment with different random seeds, and baseline accuracy was \(\approx 87\%\)..}
    \label{fig:rotation-penalties}
\end{figure}

\section{Intertwining group symmetries and model stitching}
\label{sect-stitching}

In this section we provide evidence that some of the differences between distinct model's internal representations can be explained in terms of symmetries encoded by intertwiner groups. We do this using the stitching framework from \cite{Bansal2021RevisitingMS,csiszarikSimilarityMatchingNeural2021}, and begin by reviewing the concept of network stitching. 

Suppose \(f, \tilde{f}\) are two networks as in \cref{sec:notation}, with
weights \(W \), \(\tilde{W}\) respectively. For any $1 \leq l \leq k-1$ we may form a \textbf{stitched network \(S(f, \tilde{f}, l, \varphi): \RR^{n_0} \to \RR^{n_k}\)} defined in the
notation of \cref{sec:notation} by $S(f, \tilde{f}, l, \varphi) = \tilde{f}_{>l}\circ\varphi \circ f_{\leq l}$ -- here \(\varphi: \RR^{n_l} \to \RR^{n_l}\) is a \textbf{stitching layer}.
In a typical stitching experiment one trains networks $f$ and $\tilde{f}$ from different initializations and freezes their weights, constrains
\(\varphi\) to some simple function class \(\mathcal{S}\) (e.g., affine maps in
\cite{Bansal2021RevisitingMS}), and trains \(S(f,
\tilde{f}, l, \varphi)\) by optimizing \(\varphi\) alone. The final validation
accuracy \(\Acc S(f, \tilde{f}, l, \varphi)\) of \(S(f, \tilde{f}, l, \varphi)\)
is then considered a measure of similarity (or lack therof) of the internal representations of
\(f\) and \(\tilde{f}\) in \(\RR^{n_l}\) --- in this framework the situation 
\begin{equation}
  \label{eq:happy-stitch}
  \Acc S(f, \tilde{f}, l, \varphi) \approx \Acc f, \Acc \tilde{f}
\end{equation}
corresponds to high similarity since the hidden representations of model $f$ and $\tilde{f}$ could be related by a transformation $\mathcal{S}$. 

Recall that even though the networks \(f(W)\) and \(f(W')\) may be \emph{equal as functions}, their hidden representations need not be the same (an example of this is given in \cref{appendix-examples}). Our next result shows that in the case where $f$ and $\tilde{f}$ do only differ up to an element of $G_{\sigma_{n_l}}$,  \cref{eq:happy-stitch} is achievable
even when the stitching function class \(\mathcal{S}\) is restricted down to elements of
$\phi_{\sigma}(G_{\sigma_{n_l}})$ (see \cref{def:intertwiner}). 

\begin{theorem}
  \label{thm:min-stitch}
  Suppose \(\tilde{W} = (A_1 W_1, A_1 b_1, A_2 W_2 \phi_{\sigma}(A_1^{-1}), A_2 b_2 , \dots,
  W_{k}\phi_{\sigma}(A_k^{-1}), b_{k}) \) where \(A_i \in
  G_{\sigma_{n_i}} \) for all \(i\).  Then \cref{eq:happy-stitch} is achievable
  with equality \emph{if} the stitching function class \(\mathcal{S}\) containing
  \(\varphi\) contains  \(\phi_{\sigma}(G_{\sigma_{n_l}})\).
\end{theorem}

Motivated by \cref{thm:min-stitch}, we attempt to stitch various networks at
\(\relu\) activation layers using the group \(G_{\relu}\) described in
Figure \ref{fig-intertwiner-groups}. Every matrix \(A \in G_{\relu}\) can be written as
\(PD\), where \(P\) is a permutation matrix and \(D\) is diagonal with positive
diagonal entries --- hence optimization over \(G_{\relu}\) requires optimizing
over permutation matrices. We use the well-known convex
relaxation of permutation matrices to doubly stochastic matrices and describe our  optimization procedure in greater detail in \ref{appendix-optimization-perm}.


\Cref{fig:stitch-penalties} gives the difference between the average test error of Myrtle CNN networks $f$ and $\tilde{f}$ and the network $S(f,\tilde{f},l, \varphi)$, which we call the \textbf{stitching penalty}:
\begin{equation}
\label{eq-stitching-penalty}
\frac{\Acc(f) + \Acc(\tilde{f})}{2} - \Acc(S(f,\tilde{f},l,\varphi)).
\end{equation}
In our experiments $S(f,\tilde{f}, l, \varphi)$ was stitched together at layer \(l\) via a stitching transformation \( \varphi \)  that was either optimized over all affine transformations, reduced rank affine transformations as in \cite{csiszarikSimilarityMatchingNeural2021} or transformations restricted to $G_{\relu}$. We consider only the \(\relu\) activation layers, as these are the only layers where the theory of \cref{sec:notation} applies, and we only act on the channel tensor dimension --- in practice, this is accomplished by means of 1-by-1 convolution operations. In particular, with $G_{\relu}$  we are \emph{only permuting and scaling channels}. Lower values indicate that the stitching layer was sufficient to translate between the internal representation of $f$ at layer $l$ and the internal representation of $\tilde{f}$. 

We find that when we learn a stitching layer over arbitrary affine transformations of channels, we can nearly achieve the accuracy of the original models. When we only optimize over $G_{\relu}$ there is an appreciable increase in test error difference. This is consistent with findings in
\cite{liConvergentLearningDifferent2015,wangUnderstandingLearningRepresentations2018,csiszarikSimilarityMatchingNeural2021}
discussed in \cref{sec:related}, and also consistent with observations that
hidden features of neural networks exhibit distributed representations and
polysemanticism \cite{olah2020zoom}. Nonetheless, that $S(f,\tilde{f},l, \varphi)$ is able get within less than $10\%$ of the accuracy of $f$ and $\tilde{f}$ in all but one layer suggests that elements of $G_{\relu}$ can account for a substantial amount of the variation in the internal representations of independently trained networks. We include the reduced rank transformations as the dimension of their parameter spaces is greater than that of \( G_{\relu} \), and yet they incur significantly higher stitching penalties. If \(  n_l \) is the number of channels, we have \( \dim G_{\relu_{n_l}} = n_l \) whereas the dimension of rank \( r \) transformations is \( 2n_l \cdot r -r^2 \) (hence greater than \( \dim G_{\relu_{n_l}}\) even for \( r=1\)).\footnote{A valid concern is that the preceding analysis underestimates the size of \( G_{\relu} \) by ignoring a large discrete factor: \(G_{\relu_{n_l}} \) has \(n_l !\) connected components. In \cref{sec:epsnets} we carry out a comparison of the sizes of the parameter spaces of \(G_{\relu_{n_l}} \) and reduced rank transformations inspired by the machinery of \(\epsilon\)-nets, obtaining the same conclusion that even the space of rank 1 transformations is larger than \( G_{\relu} \).} 
Finally, in the specific case of the Myrtle CNNs the stitching penalties incurred when using any layer other than 1-by-1 convolution with a rank 1 matrix all follow similar trends: they increase up to the third activation layer, then decrease at the final activation layer. 

Further stitching results on the ResNet20 architecture can be found in \cref{sec:stitching-resnets},
including an experiment where we modify the architecture to have \(\mathrm{LeakyReLU}\) activation functions, vary the negative slope of the \(\mathrm{LeakyReLU}\), and find similar stitching penalties up to but not including a slope of 1.  This result is consistent with our calculations in \cref{fig-intertwiner-groups}, where we find that for any \(\mathrm{LeakyReLU}\) negative slope \(\neq 1\) the intertwiner is the same as \(G_{\relu} \) (when the negative slope is \(1\), \(\mathrm{LeakyReLU}(x) = x\) and so the intertwiner is all of \(\mathrm{GL}_n\)).

\omitt{The fact that we are unable
to recover full accuracy in our stitched networks is consistent with findings in
\cite{liConvergentLearningDifferent2015,wangUnderstandingLearningRepresentations2018,csiszarikSimilarityMatchingNeural2021}
discussed in \cref{sec:related}, and is also consistent with observations that
hidden features of neural networks exhibit distributed representations and
polysemanticism \cite{olah2020zoom}. }



\section{Dissimilarity measures for the intertwiner group of \(\relu\)}
\label{sec:shapemetrics}

Stitching penalties can be viewed as task oriented measures of hidden feature dissimilarity. From a different perspective, we can consider raw statistical measures of hidden feature dissimilarity. In the design of measures of dissimilarity, a crucial choice is the group of transformations under which the dissimilarity measure is invariant. For example, Centered Kernel Alignment (\(\CKA\)) \cite{kornblithSimilarityNeuralNetwork2019} with the dot product kernel
is invariant with respect to orthogonal
transformations and isotropic scaling. We ask
for a statistical dissimilarity metric \(\mu\) on datasets \(X, Y \in \RR^{N
\times d}\) with the properties that (0) \( 0 \leq \mu(X, Y) \leq 1  \),  (i) (\emph{\(G_{\relu}\)-Invariance}) If \(A, B \in G_{\relu_d}\) and \(v, w \in \RR^d\) then
  \(\mu(XA +\mathbf{1}v^T, YB + \mathbf{1}w^T)= \mu(X, Y)\), and
(ii) (\emph{Alignment Property}) \(\mu(X, Y) = 1\)  if \((\ast)\) \(Y = XA + \mathbf{1}v^T \text{  for some  } A \in
G_{\relu_d}, v\in \RR^d \). To motivate this question, we note that given such a metric \(\mu\), one can detect if \( X\) and \(Y\) do \emph{not} differ by an element of \(G_{\relu}\) by checking if \(\mu(X, Y) < 1\). 
Our basic tool for ensuring (i) is the next lemma.
\begin{lemma}
  \label{lem:iterated-quotient}
  Suppose \(\mu(XA, YB)= \mu(X, Y)\) if \(A, B\) are \emph{either} positive
  diagonal matrices or permutation matrices. Then, (i) holds.
\end{lemma}
In effect, this allows us to divide the columns of \(X\) and \(Y\) by their norms to achieve
invariance to the action of positive diagonal matrices and then apply dissimilarity measures
for the permutation group such as those presented in \cite{williamsGeneralizedShapeMetrics2022}. Ensuring (ii) seems to require
case-by-case analysis to determine an appropriate normalization constant.
\begin{definition}[\(G_{\relu}\)-Procrustes]
  Let \(D_X = \diag (\nrm{ X_{[:, i]}})\) and \(D_Y = \diag (\nrm{Y_{[:,
  i]}})\). Assuming these are invertible, let  $\tilde{X} = X D_X^{-1}$ and $\tilde{Y} = Y D_Y^{-1}$.
  Let \(\delta\) be the permutation Procrustes distance between \(\tilde{X},
  \tilde{Y}\), defined by 
  \(\delta := \min_{P\in \Sigma_d} \nrm{\tilde{X}- \tilde{Y}P}\)
  (as pointed out in \cite{williamsGeneralizedShapeMetrics2022} this can be
  computed via the linear sum assignment problem). Then the \textbf{$G_{\relu}$-Procrustes measure} is
  \[ \mu_{\text{Procrustes}}(X, Y) := 1 - \frac{\delta}{2\sqrt{d}}.\] 
\end{definition}
The factor
of \(2\sqrt{d}\) ensures this lies in \([0,1]\), and equals \(1\) if (and only if)
the condition \(\ast\) of (ii) holds. 

\begin{table}[h] 
    \centering
    \begin{tabular}{lllll}
\toprule
{} &          layer 3 &          layer 6 &         layer 10 &         layer 14 \\
\midrule
$G_{\relu}$       &  0.6208 \(\pm\) 0.008 &  0.5106 \(\pm\) 0.005 &  0.4432 \(\pm\) 0.004 &  0.4899 \(\pm\) 0.002 \\
Orthogonal &  0.7724 \(\pm\) 0.028 &  0.5743 \(\pm\) 0.040 &  0.5087 \(\pm\) 0.016 &  0.5825 \(\pm\) 0.019 \\
\bottomrule
\end{tabular}
    \caption[]{\(G_{\relu}\) and orthogonal Procrustes similarities for Myrtle CNNs trained on CIFAR-10. Confidence intervals were obtained by evaluating similarities for 32 pairs models trained with different random seeds. 
    }\label{fig:wreath-procrustes}
\end{table}

We apply $G_{\relu}$-Procrustes and orthogonal Procrustes similarities to $4$ different hidden representations from Myrtle CNNs  in \cref{fig:wreath-procrustes} and many more layers of ResNet20s in \cref{fig:rn-pro}, all trained on CIFAR-10 \cite{Krizhevsky09learningmultiple}.  In keeping with the discussion of \cref{sect-stitching}, we only consider permutations or orthogonal transformations of channels (for details on how this is implemented we refer to \cref{sec:dissimilarity-details}). We see that distinct representations register less similarity in terms of $G_{\relu}$-Procrustes than they do in terms of orthogonal Procrustes. This makes sense as similarity up to $G_{\relu}$-transformation requires a greater degree of absolute similarity between representations than is required of similarity up to orthogonal transformation (the latter being a higher-dimensional group containing all of the permutations in $G_{\relu}$). Otherwise patterns in $G_{\relu}$-Procrustes similaritiy largely follow those of orthogonal Procrustes, with similarity between representations decreasing as one progresses through the network, only to increase again in the last layer. This correlates with the stitching penalties of \cref{fig:stitch-penalties}, which increase with depth only to decrease in the last layer.

\begin{definition}[\(G_{\relu}\)-CKA]
  Assume that \(X\) and \(Y\) are data matrices that have been centered by subtracting means of rows: $X \gets X - \frac{1}{d}\mathbf{1}\mathbf{1}^T X$ and $Y \gets Y - \frac{1}{d}\mathbf{1}\mathbf{1}^T Y$. Let \( \tilde{X} = X D_X^{-1} \text{ and } \tilde{Y} = Y D_Y^{-1}
  \). Let \(\tilde{x}_1, \dots, \tilde{x}_N\) be the rows of \(\tilde{X}\), and
  similarly for \(\tilde{Y}\). Form the matrices \(K, L \in \RR_{\geq 0}^{N
  \times N}\) defined by $K_{ij} = \max(\tilde{x}_i \odot \tilde{x}_j)$
  and $L_{ij} = \max(\tilde{y}_i \odot \tilde{y}_j)$
  where $\odot$ is the Hadamard product. Then the $G_{\relu}$-CKA for $X$ and $Y$ is defined as: 
  \begin{equation}
  \label{eq:muCKA}
  \mu_{\text{CKA}}(X, Y) := \frac{\mathrm{HSIC}_1(K, L)}{\sqrt{\mathrm{HSIC}_1(K, K)}\sqrt{\mathrm{HSIC}_1(L, L)}}.
  \end{equation}
  where \( \mathrm{HSIC}_1 \) is the unbiased form of the Hilbert-Schmidt independence criterion of \cite[eq. 3]{nguyenWideDeepNetworks2021}.
  \end{definition}
  
  Symmetry of the \(\max\) function ensures (i), the
  Cauchy-Schwarz inequality ensures \(\mu_{\text{CKA}}(X, Y) \in [0,1]\), and we
  claim that \(\mu_{\text{CKA}}(X, Y) = 1\) if the condition \(\ast\) of
  (ii) is met. We do not claim `if and only if', however we point
  out the following in \cref{lem:max-symmetries}: if \(A\) is
  a matrix such that \(\max(Ax_1 \odot Ax_2) = \max(x_1 \odot x_2)\) for all \(x_1, x_2
  \in \RR^d\), then \(A \) is of the form \(PD\) where \(P\) is a permutation
  matrix and \(D\) is diagonal with diagonal entries in \(\{\pm 1\}\). In fact, \(\mu_{\text{CKA}}\) is simply an instance of CKA for a the ``\( \max \) kernel.''

  \begin{remark}
    \label{rmk:oops}
    In a previous version of this paper, it was incorrectly claimed that ``the
    function \(\kappa : \RR^d \times \RR^d \to \RR\) defined by \(\kappa(x, y) =
    \max(x\odot y)\) is a positive semi-definite (PSD) kernel.'' In fact it is not
    positive semi definite; a small counterexample courtesy of Derek Lim can be
    found in \cref{ex:oops}. While it is still possible for us to experimentally
    use the framework of CKA with the (not PSD) kernel
    \(\kappa(x, y) = \max(x\odot y)\), the theoretical underpinnings of CKA
    (e.g. \cite{Gretton2005MeasuringSD,cortes2012algorithms}) have been
    developed for PSD kernels, and in this sense our \(G_{\relu}\)-CKA dissimilarity
    measure is a somewhat non-standard and rule-bending instance of CKA. 

    It is worth noting that there are plenty of PSD kernel functions with the
    same symmetry properties as the \(\max\) kernel used in this paper. One
    example is described as follows:
    denoting the \(\ell^1\)-norm by \(\nrm{(x_1,\dots, x_d)}_1 =
    \sum_i \nrm{x_i} \), define 
    \begin{equation}
      \kappa(x, y) = \exp(-c \nrm{x-y}_1).
    \end{equation}
    The symmetry properties of \(\kappa\) are inherited from those of
    \(\nrm{x}_1\), and the PSD-ness of \(\kappa\) is
    well-known.\footnote{A neat proof is sketched
    \href{https://math.stackexchange.com/questions/2412421/is-the-matrix-of-exponential-kernel-with-l1-norm-positive-definite}{here}
   (see also \cite{kmml,micchelli06a} for further details). The function \(\kappa\) is often referred to as the ``Laplace
    kernel'' in the machine learning literature.} We suspect that replacing the
    \(\ell^1\)-norm with the \(\ell^p\)-norm \(\nrm{(x_1,\dots, x_d)}_p =
    (\sum_i \nrm{x_i}^p)^{\frac{1}{p}} \) for any \(p\neq 2\) provides a larger
    family of examples, and these do indeed have the same symmetry properties as
    the \(\max\) kernel and the \(p=1\) case above, but we were unable to locate
    a proof that these kernels are PSD, and we don't attempt a proof here.
  \end{remark}
  
  As with CKA \cite{kornblithSimilarityNeuralNetwork2019}, this metric makes
  sense even if \(X, Y\) are datasets in \(\RR^d, \RR^{d'}\) respectively with
  \(d \neq d'\). Results for a pair of Myrtle CNNs trained on CIFAR-10 with different random seeds, as well as standard orthogonal CKA for comparison, are shown in \cref{fig:wreath-cka}. Analogous results for ResNet20s  are shown in \cref{fig:wreath-cka-resnet}. We find that $G_{\relu}$-CKA respects basic trends found in their orthogonal counterparts: model layers at the same depth are more similar, early layers are highly similar, and the metric surfaces the block structure of the ResNet in \cref{fig:wreath-cka-resnet} (layers inside residual blocks are less similar than those at residual connections). One notable difference for $G_{\relu}$-CKA in \cref{fig:wreath-cka-resnet,fig:wreath-cka} is that the similarity difference between early and later layers in the orthogonal CKA (discussed for ResNets in \cite{raghu2021vision}) shown in (b) is less pronounced in (a), and in fact later layers are found to be less similar between runs. We found similar results for stitching in \cref{fig:stitch-penalties,fig:stitch-penalties-resnet}.
  

\begin{figure}
    \begin{subfigure}[h]{0.5\linewidth}
        \centering
        \includegraphics[width=\linewidth]{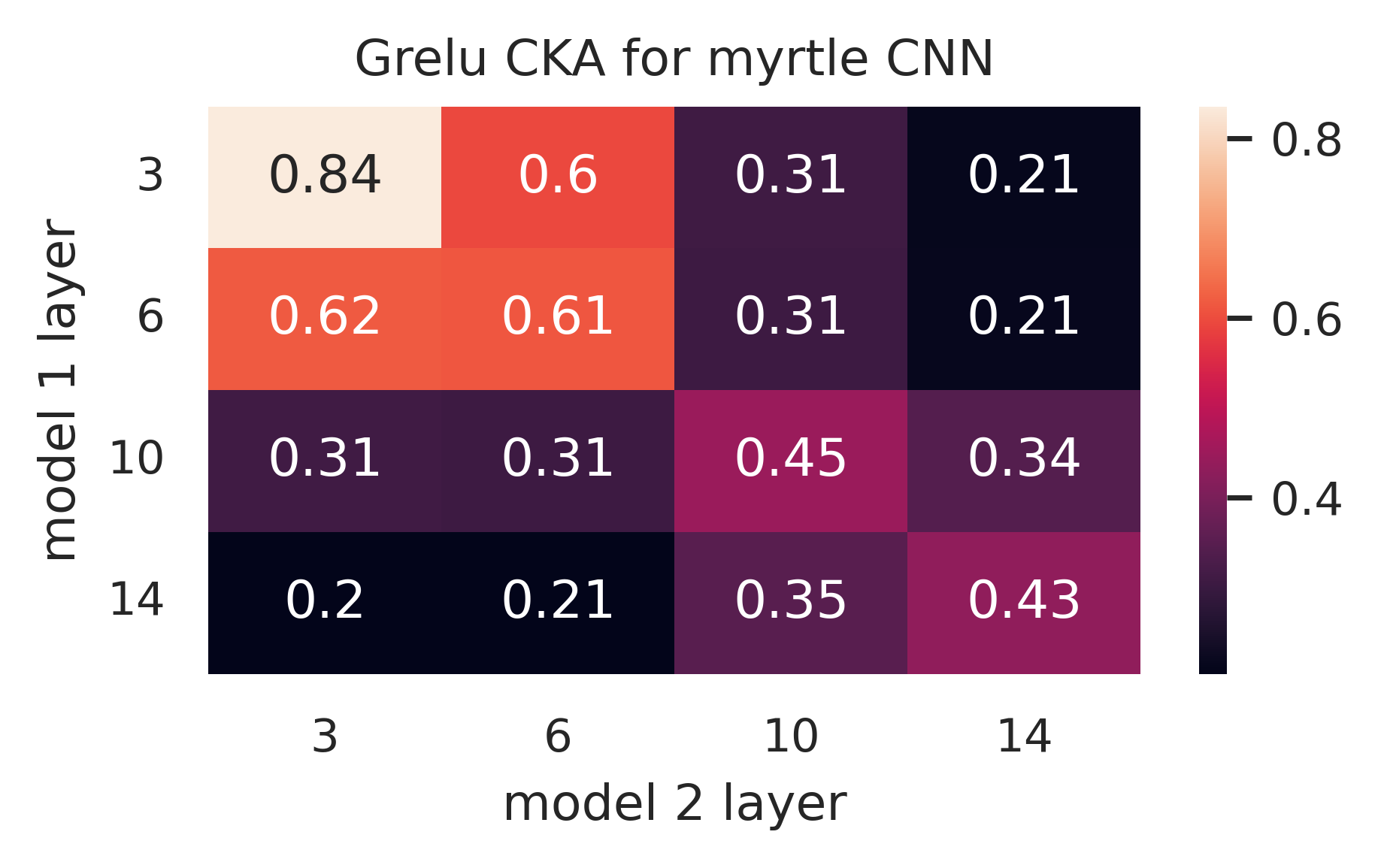}
    \end{subfigure}
    \begin{subfigure}[h]{0.5\linewidth}
        \centering
        \includegraphics[width=\linewidth]{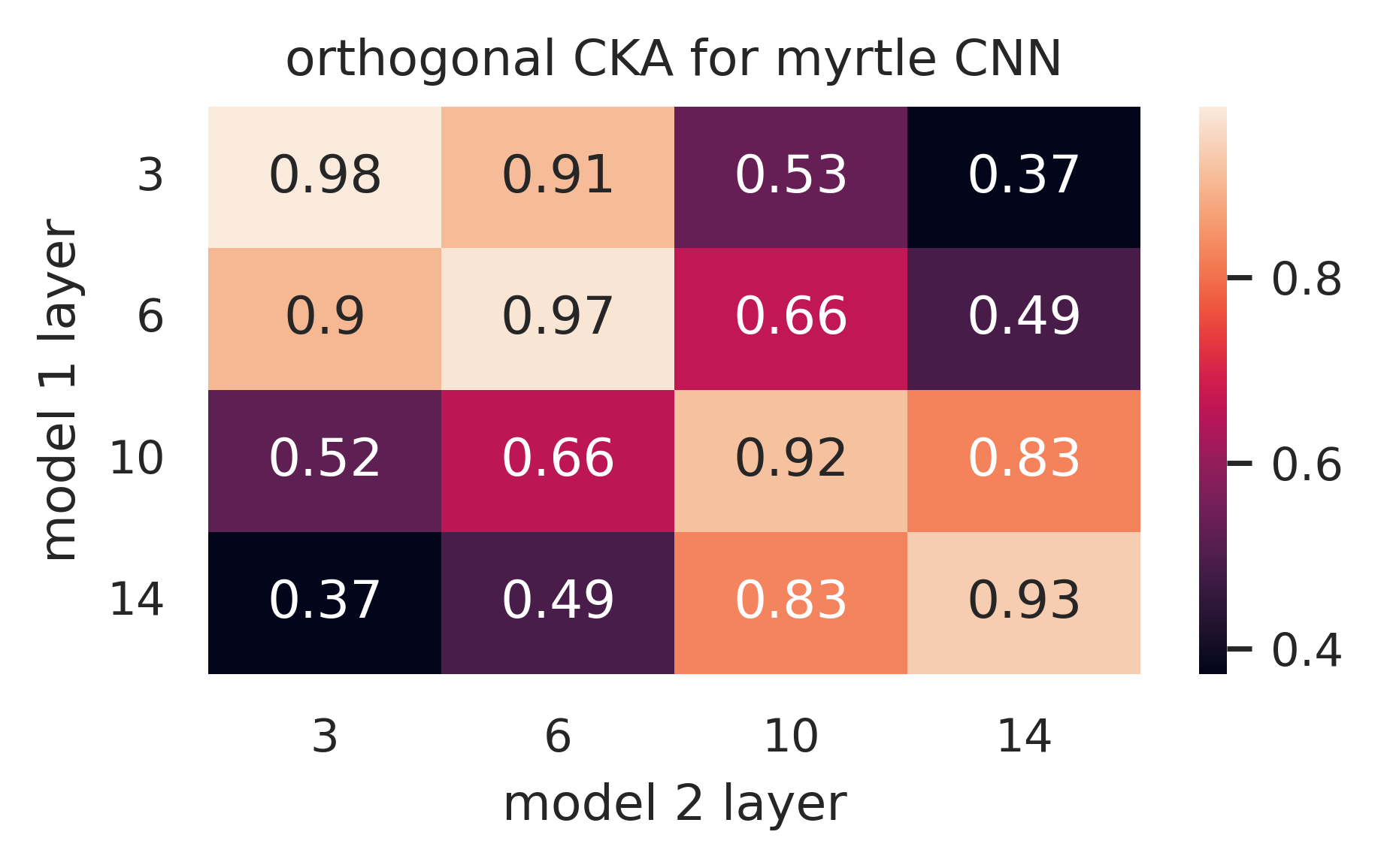}
    \end{subfigure}
    \caption{\(G_{\relu}\)-CKA and orthogonal CKA for two Myrtle CNNs with different random seeds trained on CIFAR-10. Results averaged over 16 such pairs of models 
    .}\label{fig:wreath-cka}
\end{figure}

\section{Interpretability of the coordinate basis}
\label{sec:net-dissect}
In this section we explore the confluence of model interpretability and intertwiner symmetries using \textit{network dissection} from
\cite{Bau2017NetworkDQ}. Network dissection measures
alignment between the individual neurons of a hidden layer and single,
pre-defined concepts (see \cref{sec:net-d-methodology} for the methodology). We adapt an experiment from \cite{Bau2017NetworkDQ} to compare the axis-aligned interpretability of hidden activation layers with and without an activation function. Bau et al. compares the interpretability of individual neurons, measured via network dissection, with that of random orthogonal rotations of neurons. We likewise rotate the hidden layer representations and then measure their interpretability. Using the methodology from \cite{diaconis2005random}, we define a random orthogonal transform $Q$ drawn uniformly from $SO(n)$ by using Gram-Schmidt to orthonormalize the normally-distributed $Q R=A
\in \RR^{n^{2}}$. Like in \cite{Bau2017NetworkDQ}, we also
consider smaller rotations $Q^{\alpha} \in S O(n)$ where $0 \leq \alpha \leq 1$,
where $\alpha$ is chosen to form a minimal geodesic rotating from \(I\) to
\(Q\). \cite{Bau2017NetworkDQ} found that the number of interpretable units decreased away from the activation basis as $\alpha$ increased for $\text{layer5}$ of an AlexNet.

We compare three models trained on ImageNet: a ResNet-50, a modified ResNet-50 where we remove the ReLU on the residual outputs (training details in \cref{sec:net-d-train}), and a ConvNeXt \cite{liu2022convnet} analog of the ResNet-50, which also does not have an activation function
before the final residual output. We give results in \cref{fig:net-d1}, and provide sample unit detection outputs and full concept labels for the figures in \cref{sec:net-d-further}. As was shown in \cite{Bau2017NetworkDQ}, interpretability decreases as we rotate away from the axis for the normal ResNet-50 in \cref{sec:net-d-train}. On the other hand, with no activation function, neuron interpretability does
not drop with rotation for the modified ResNet-50 and the ConvNeXt. 
We note that the models without residual activation functions also have far fewer concept covering units for a given basis.
Interestingly, while the number of interpretable units remains constant for the residual output of the modified ResNet-50, for the ConvNeXt model it actually \textit{increases}. We find similar results, where the number of interpretable units increase with rotation, for the convolutional layer inside the residual block for the modified ResNet-50 in \cref{fig:net-d2}. 


\section{Limitations}
\label{sect-limitations}


Our theoretical analysis in \cref{sec:notation} does not account for standard regularization techniques that are known to have symmetry-breaking effects (for example weight decay reduces scaling symmetry). More generally, we do not account for any implicit regularization of our training algorithms. As illustrated in \cref{fig:stitch-penalties,fig:stitch-penalties-resnet}, stitching  with intertwiner groups appears to have significantly more architecture-dependent behaviour than stitching with arbitrary affine transformations (however, since different architectures have different symmetries this is to be expected). Our empirical tests of the dissimilarity measures in \cref{sec:shapemetrics} are limited to what \cite{kornblithSimilarityNeuralNetwork2019} terms ``sanity tests''; in particular we did not perform the specificity, sensitivity and quality tests of \cite{dingGroundingRepresentationSimilarity2021}.

\section{Conclusion}
\label{sec:concl}

In this paper we describe groups of symmetries that arise from the nonlinear layers of a neural network, calculate these symmetry groups for a number of different types of nonlinearities, and explore their fundamental properties and connection to weight space symmetries. Next, we provide evidence that these symmetries induce symmetries in a network's internal representation of the data that it processes, showing that previous work on the internal representations of neural networks can be naturally adapted to incorporate awareness of the intertwiner groups that we identify. Finally, in the special case where the network in question has ReLU nonlinearities, we find experimental evidence that intertwiner groups justify the special place of the activation basis within interpretable AI research.

\section{Acknowledgements}
\label{sec:ack}

This research was supported by the Mathematics for Artificial Reasoning in Science (MARS) initiative at Pacific Northwest National Laboratory.
It was conducted under the Laboratory Directed Research and Development (LDRD) Program at at Pacific Northwest National Laboratory (PNNL), a multiprogram
National Laboratory operated by Battelle Memorial Institute for the U.S. Department of Energy under Contract
DE-AC05-76RL01830.

The authors would also like to thank Nikhil Vyas for useful discussions related
to this work and Derek Lim for pointing out that the \(\max \) kernel introduced
in \cref{sec:shapemetrics} is not positive definite.

\printbibliography

\newpage

\section*{Checklist}

\begin{enumerate}

  \item For all authors...
  \begin{enumerate}
    \item Do the main claims made in the abstract and introduction accurately reflect the paper's contributions and scope?
    \answerYes{}
    \item Did you describe the limitations of your work?
    \answerYes{See Section \ref{sect-limitations}.}
    \item Did you discuss any potential negative societal impacts of your work?
    \answerNA{This paper is largely focused on the mathematical aspects of deep learning so we do not think there are any immediate negative societal impact to the methods described. From a broader perspective though, we see this work helping to create a more principled groundwork for many interpretable AI techniques. We explain why this could have positive societal impacts in Section \ref{sect-societal-impacts}.}
    \item Have you read the ethics review guidelines and ensured that your paper conforms to them?
    \answerYes{}
  \end{enumerate}

  \item If you are including theoretical results...
  \begin{enumerate}
    \item Did you state the full set of assumptions of all theoretical results?
    \answerYes{}
    \item Did you include complete proofs of all theoretical results?
    \answerYes{All proofs can be found in \cref{appendix-proofs}.}
  \end{enumerate}

  \item If you ran experiments...
  \begin{enumerate}
    \item Did you include the code, data, and instructions needed to reproduce the main experimental results (either in the supplemental material or as a URL)?
    \answerTODO{} We are in the process of making code publicly available.
    \item Did you specify all the training details (e.g., data splits, hyperparameters, how they were chosen)?
    \answerYes{} See \cref{sect-experimental-details}.
    \item Did you report error bars (e.g., with respect to the random seed after running experiments multiple times)?
    \answerYes{} 
    \item Did you include the total amount of compute and the type of resources used (e.g., type of GPUs, internal cluster, or cloud provider)?
    \answerYes{} See \cref{sect-experimental-details} --- while we did not keep precise track of CPU/GPU hours, we do specify the hardware used.
  \end{enumerate}

  \item If you are using existing assets (e.g., code, data, models) or curating/releasing new assets...
  \begin{enumerate}
    \item If your work uses existing assets, did you cite the creators?
    \answerYes{}
    \item Did you mention the license of the assets?
    \answerYes{See Section \ref{sect-dataset-details}. }
    \item Did you include any new assets either in the supplemental material or as a URL?
    \answerNo{}
    \item Did you discuss whether and how consent was obtained from people whose data you're using/curating?
    \answerNA{}
    \item Did you discuss whether the data you are using/curating contains personally identifiable information or offensive content?
    \answerNA{}
  \end{enumerate}

  \item If you used crowdsourcing or conducted research with human subjects...
  \begin{enumerate}
    \item Did you include the full text of instructions given to participants and screenshots, if applicable?
    \answerNA{}
    \item Did you describe any potential participant risks, with links to Institutional Review Board (IRB) approvals, if applicable?
    \answerNA{}
    \item Did you include the estimated hourly wage paid to participants and the total amount spent on participant compensation?
    \answerNA{}
  \end{enumerate}

\end{enumerate}

\appendix

\section{Societal Impact}
\label{sect-societal-impacts}

Though deep learning models are in the process of being deployed for safety critical applications, we still have very little understanding of the structure and evolution of their internal representations. In this paper we discuss one aspect of these representations. We hope that by better illuminating the inner workings of these networks, we will be a small part of the larger effort to make deep learning more understandable, reliable, and fair.

\section{Code availability} 

Our code can be found at \href{https://github.com/pnnl/modelsym}{https://github.com/pnnl/modelsym}. 


\section{Examples}
\label{appendix-examples}

We first give an example of two networks with distinct weights which are functionally equivalent. Let \(f\) be a 2 layer network with \(  \relu \) activations and weight matrices 
\begin{equation*}
W_1 = \begin{bmatrix}
1 & 0 \\
0 & 2 
\end{bmatrix}
\quad \text{and} \quad 
W_2 = \begin{bmatrix}
3 & 0 \\
0 & 1 
\end{bmatrix}
\end{equation*}
(and biases = 0). Let \(\tilde{f}\) be a network with the same architecture, but with weights
\begin{equation*}
W_1 = \begin{bmatrix}
0 & 2 \\
1 & 0 
\end{bmatrix}
\quad \text{and} \quad 
W_2 = \begin{bmatrix}
0 & 3 \\
1 & 0 
\end{bmatrix}.
\end{equation*}
Then one can verify that \(\tilde{f}(x) = f(x) \) for all \(x \in \RR \), but that the weights of $f$ and $\tilde{f}$ differ.

We also work through a small example of $\phi_{\sigma_n}$ where $n = 2$.  Assume that $\sigma$ is the ReLU nonlinearity. Then,  
\begin{equation*}
A = \begin{bmatrix}
0 & 1 \\
2 & 0 
\end{bmatrix}
\end{equation*}
belongs to $G_{\sigma_2}$, and we can compute directly that
\begin{equation*}
    \relu \circ \begin{bmatrix}
    0 & 1 \\
    2 & 0
    \end{bmatrix}
    \begin{bmatrix}
    x_1 \\ 
    x_2
    \end{bmatrix} =\begin{bmatrix}
     \relu(x_2) \\
     \relu(2 x_1)
    \end{bmatrix}
    = \begin{bmatrix}
     \relu(x_2) \\
     2\relu(x_1)
    \end{bmatrix},
\end{equation*}
where in the last equality we used the fact that \(\relu(a x) = a \relu(x)\) when \(a\) is positive.
On the other hand, 
\begin{equation*}
    \begin{bmatrix}
    0 & 1 \\
    2 & 0
    \end{bmatrix} \circ\relu ( 
    \begin{bmatrix}
    x_1 \\ 
    x_2
    \end{bmatrix}
    )
    = 
    \begin{bmatrix}
    0 & 1 \\
    2 & 0
    \end{bmatrix}
    \begin{bmatrix}
    \relu (x_1) \\ 
    \relu (x_2)
    \end{bmatrix}
    = \begin{bmatrix}
     \relu(x_2) \\
     2\relu(x_1)
    \end{bmatrix}.
\end{equation*}

\section{Experimental Details}
\label{sect-experimental-details}

In this section we provide additional experimental results, as well as implementation details for the purposes of
reproducibility. All experiments were run on Nvidia GPUs using PyTorch \cite{torch}.

\subsection{Sampling pairs of models trained with different random seeds}
\label{sec:sampling-models}

We began by training 100 models with different random seeds (i.e. with independent initializations and different random batches) for each of the following architectures:
\begin{enumerate}[(i)]
  \item Myrtle CNN: a simple 5-layer feed-forward CNN with batch normalization.\footnote{With the exception of the rotation penalties experiment in \cref{fig:rotation-penalties}, where we omitted batch normalization to adhere closely to the theoretical framework of \cref{sec:notation}}
  \item ResNet20: a ResNet tailored to the CIFAR-10 dataset (numbers of channels
  are \(16, 32, 64\) respectively in the 3 residual blocks).
  \item ResNet18: an ImageNet-style ResNet adapted to the input size of CIFAR-10
  --- much wider than the above (numbers of channels are \(64, 128, 256\)
  respectively in the 3 residual blocks).
\end{enumerate}
More detailed architecture schematics are included in \cref{fig:arch-mCNN,fig:arch-rn20,fig:arch-rn18}.

All models were trained for 50 epochs using the Adam optimizer with PyTorch's
default settings. We use a batch size of 32, initial learning rate \(0.001\) and
4 evenly spaced learning rate drops with factor \(0.5\). We augment data with
translations of up to 2 pixels (padded as necessary with the mean RGB value for CIFAR-10) and left-right flips, and we save the weights with best
validation accuracy. In the rotation penalties experiment of \cref{fig:rotation-penalties} the fine-tuning stage uses the same hyperparameters as the initial training phase (though of course only a subset of parameters recieve gradient updaates during fine-tuning). Training this many CIFAR-10 models on a reasonable budget of time and computing resources was greatly aided by the excellent FFCV library \cite{leclerc2022ffcv}.

In the later stitching and dissimilarity measure experiments, we sample pairs of models from these ``zoos'' uniformly with replacement (but of course making sure that the two models in the pair are distinct). Thus the cost of training hundreds of models is amortized across many runs of stitching and dissimilarity measurement; this can be also viewed as bootstrap estimation of our experimental quantities of interest using empirical samples from certain distributions of CIFAR-10 models. 

\subsection{Stitching Experiments}
\label{appendix-optimization-perm}

For stitching layers, we train for 20 epochs with batch size 32 and learning
rate \(0.001\) (with no drops), however we use vanilla SGD with no momentum (we found the approximate second-order and/or momentum aspects of 
Adam interacted in complicated ways with the PGD algorithm described in
\cref{sec:approxoptperm} below, even after following some
\href{https://datascience.stackexchange.com/questions/31709/adam-optimizer-for-projected-gradient-descent}{helpful
advice from the Internet}\footnote{https://datascience.stackexchange.com/questions/31709/adam-optimizer-for-projected-gradient-descent}). Augmentation is described in the previous paragraph.

We parameterize reduced rank 1-by-1 convolutions as a composition of 2 1-by-1
convolutions, with \(\mathtt{in\_channels, out\_channels} = \mathtt{in\_channels,
rank}\) and \(\mathtt{rank, in\_channels}\) respectively. In contrast to
\cite{Bansal2021RevisitingMS} we omit both batch norm and bias from stitching
layers (to stick closely to the statement of \cref{thm:min-stitch}). 

\subsubsection{Approximate Optimization over Permutation Matrices}
\label{sec:approxoptperm}

By far the most complicated stitching layer is the one using \(G_{\relu}\), which
we describe here. Recall that $G_{\relu}$ is equal to the $n \times
n$ matrices of the form \(P D\), where \(P \in \Sigma_n\) is a permutation
matrix and \(D\) is a diagonal matrix with positive entries We parameterize
\(D\) simply as \(D = \diag(\lambda_i)\)  where \(\lambda_1, \dots,
\lambda_{n_l}\in \RR_{\geq 0}\) --- we preserve non-negativity during training
by a projected gradient descent step \(D \gets \relu(D)\).  During stitching
layer training, we parameterize \(P\) as a doubly stochastic matrix, that is, an
element of the Birkhoff polytope
\[\mathcal{B} = \{A = (a_{ij}) \in \Mat_{n_l, n_l}(\RR) \, | \, a_{ij}\geq 0
\text{ for all } i, j, \mathbf{1}^T A = \mathbf{1}^T \text{ and } A \mathbf{1} =
\mathbf{1}\} \] --- after each gradient descent step we project \(P\) back onto
\(\mathcal{B}\) by the operation \(P \gets \relu(P)\) followed by \(P \gets
\mathrm{sink}(P)\), where ``\(\mathrm{sink}\)'' denotes Sinkhorn iterations.
These consist of \(T\) iterations of \[ A \gets A \diag(\mathbf{1}^T A )^{-1} \text{
followed by } A \gets \diag(A \mathbf{1})^{-1}A  \] (it is a theorem of Sinkhorn
that this sequence converges to a doubly stochastic matrix of the form \( D A E
\)  with \(D, E\) positive diagonal matrices \cite{sink}). We use \(T = 16 \) in all
experiments (this choice drew on the work of \cite{mena2018learning}). In addition, we add a regularization term \( - \alpha \nrm{P}_2 \) to the stitching
objective, where \(\alpha >0\) is a hyperparameter (the motivation here is that
permutation matrices are precisely the elements of \(\mathcal{B}\) with maximal
\(\ell_2\)-norm). Unless stated otherwise in our experiments \(\alpha = 0.1\). We did experiment with choosing \(\alpha\) by cross validation and found the particular choice of \(\alpha\) was not crucial; see \cref{sec:crossval} for further details.

At evaluation time, we threshold \(P\) to an actual permutation matrix via the
Hungarian algorithm (specifically its implementation in
\lstinline{scipy.optimize.linear_sum_assignment} \cite{2020SciPy-NMeth}). This
amounts to 
\[ P_{\mathrm{eval}} = \mathrm{arg}\max_{Q \in \Sigma_{n_l}} \tr(P_{\mathrm{train}}Q^T) \] 

As stated above, we train for 20 epochs with batch size 32 and learning
rate \(0.001\) (with no drops), using SGD with momentum 0.9. However, we allow
the permutation factor to get a  ``head start'' by keeping
\(D\) fixed at the identity \(I\) for the first \(10\) epochs. This is probably not
essential, as shown in \cref{sec:crossval}.

Finally, before evaluating the stitched model on the CIFAR-10 validation set, we
perform a no-gradient epoch \emph{on the training data} with stitching layer \(
P_{\mathrm{eval}} \). This is critical as it allows the batch normalization
running means and variances in later layers to adapt to the thresholded
permutation matrix \( P_{\mathrm{eval}} \); observe that if we omitted this
step, during evaluation 
the ``batch normalization layers'' would not even be performing batch
normalization per se, since their running statistics would be computed from
features produced by a layer \(P_{\mathrm{train}}\) no longer in use.

As an aside, we also experimented with the differeniable relaxation of permutation
matrices SoftSort \cite{prilloSoftSortContinuousRelaxation2020}. Our final
results were comparable, however this method took far longer (\(>
10 \times\)) to optimize than the Birkhoff polytope method. It is
perhaps of interest that we used SoftSort on permutations far larger than
those of \cite{prilloSoftSortContinuousRelaxation2020} (e.g., the 512
channels of late layers of our Myrtle CNN). The next section (\cref{sec:sss})
contains some of our technical findings. 

We wish to aknowledge a couple articles, \cite{Fogel2013ConvexRF} and \cite{Lim2014BeyondTB}, that provided us with useful backround on optimization over doubly stochastic matrices.

\subsubsection{Stitching with SoftSort}
\label{sec:sss}

We parameterized \(D\) simply as \(D =  \diag(e^{\lambda_i})\)
where \(\lambda_1, \dots, \lambda_{n_l}\in \RR\). During stitching layer
training, we parameterized \(P\) using SoftSort
\cite{prilloSoftSortContinuousRelaxation2020}, a continuous relaxation of
permutation matrices given by the formula
\[ P = \SoftSort(s, \tau) := \softmax 
\big(-\frac{1}{\tau}(\sort(s)\mathbf{1}^T - \mathbf{1}s^T) \big), \text{ where }
s \in \RR^{n_l},\] \(\sort(s)\) denotes \(s\) sorted in descending order, and
\(\softmax\) is applied over rows. The parameter \(\tau > 0\) controls
\(\softmax\) temperature, and we were only able to obtain reasonable results
when tuning it according to \(\tau
\approx 1/n_l\) .
At validation time, we threshold \(P\) to an actual permutation matrix by
applying \(\arg \max\) over rows as in
\cite{prilloSoftSortContinuousRelaxation2020}. 

\subsection{Stitching and \(G_{\relu}\)-dissimilarity measures for ResNets}
\label{sec:stitching-resnets}

Here we include further results for ResNet20 and ResNet18 architectures. \Cref{fig:stitch-penalties-resnet} and \cref{fig:stitch-penalties-resnet18} include results for full 1-by-1 convolution, reduced randk 1-by-1 convolution and \(G_{\relu}\) 1-by-1 convolutions stitching in the ResNet20 and ResNet18 architectures respectively. Note that in general, layers inside residual blocks incur higher penalties, consisent with \cref{rmk:rn-failure}. This holds even in the full 1-by-1 convolution case, a finding that to the best of our knowledge is new. 

In the case of ResNet20 we also observe that the relative ranking of the different stitching constraints tends to change inside of residual blocks: whereas \(G_{\relu}\) stitching consistently outperforms rank 1 (and sometimes rank 2) stitching outside residual blocks, it consistently underperforms all strategies inside residual blocks. Lastly, we remark that the ResNet20 is significantly narrower than the Myrtle CNN (channels are 16, 32, 64 vs. 64, 128, 256, see \cref{fig:arch-rn20,fig:arch-rn18}), and hence the low-rank transformations account for a larger \emph{proportion} of the available total rank (for example, in early layers of the ResNet20 rank 4 is \(0.25\cdot \mathrm{full rank}\) whereas in the early layers of the Myrtle CNN rank 4 is \(0.0625\cdot \mathrm{full rank}\)). Heuristically, in the narrower network low-rank transformations may suffice to align for a larger fraction of the principal components of hidden features.

We also observe generally lower stitching penalties in the ResNet18 with the exception of the penultimate inside-a-residual-block layer --- we do not have a satisfactory explanation for random chance performance at that layer. We also remark that while the penalties in \cref{fig:stitch-penalties-resnet} are significantly higher than those in \cref{fig:stitch-penalties}, especially in later layers, we also saw significant dissimilarity in \cref{fig:wreath-cka-resnet} (a), especially in later layers. 

We also modify the ResNet20 to use the $\leakyrelu$ activation function and train models with different negative slopes $s$. The accuracy for two models trained with different random seeds at different $\leakyrelu$ is given in \cref{table:leaky-acc}. We perform $G_{\relu}$ stitching in \cref{fig:stitch-penalties-leaky}. Note that for a negative slope $s=1$, the activation function is the identity. We find the results difficult to interpret due to the significant decrease in CIFAR-10 accuracy for larger $s$. With this being said, unlike for $s<<1$, we note that the stitching penalties for $s=1$ (and to a lesser extent, $s=0.9$) are mostly constant throughout the layers of the network. This is most prominent for the final two ResNet20 layers ($72$ and $75$), where the stitching penalty for models with small $\leakyrelu$ slopes is the lowest.

\begin{figure}
    \centering
    \includegraphics[width=0.9\linewidth]{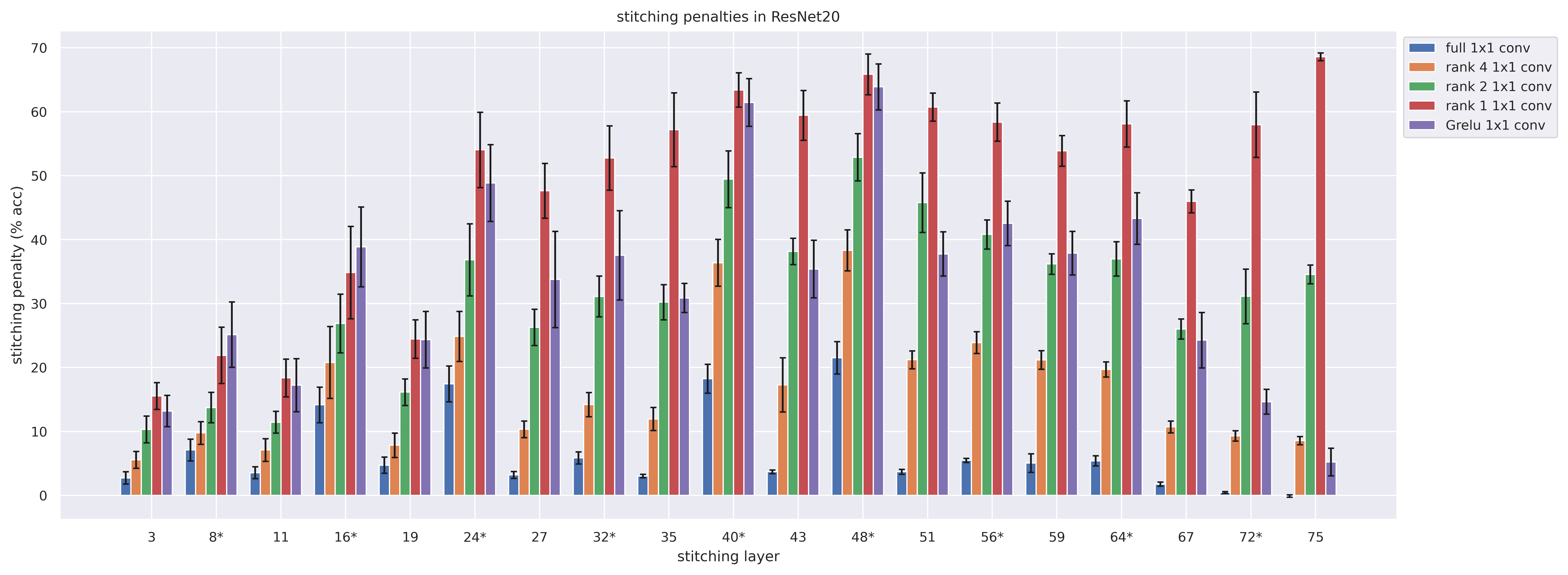}
    \caption{Full/reduced rank and \(G_{\relu}\) 1-by-1 convolution stitching penalties \eqref{eq-stitching-penalty} for ResNet20s on CIFAR-10. Confidence intervals were obtained by evaluating stitching penalties for 16 pairs of models trained with different random seeds. Accuracy of the models was 89.9 \( \pm \) 0.2 \%. Layers marked with `*' occur inside residual blocks  (\cref{rmk:rn-failure}).}
    \label{fig:stitch-penalties-resnet}
\end{figure}

\begin{figure}
    \centering
    \includegraphics[width=0.9\linewidth]{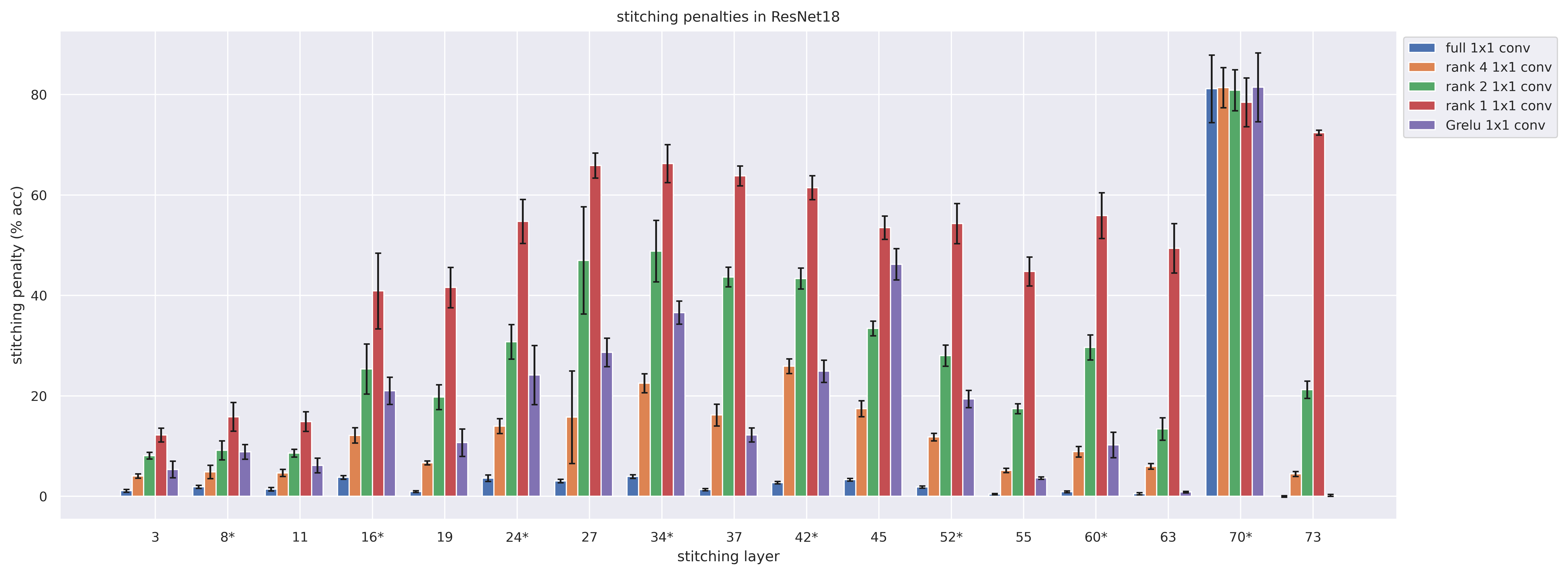}
    \caption{Full/reduced rank and \(G_{\relu}\) 1-by-1 convolution stitching penalties \eqref{eq-stitching-penalty} for ResNet18s on CIFAR-10. Confidence intervals were obtained by evaluating stitching penalties for 16 pairs of models trained with different random seeds. Accuracy of the models was 92.9 \( \pm \) 0.2 \%. Layers marked with `*' occur inside residual blocks  (\cref{rmk:rn-failure}).}
    \label{fig:stitch-penalties-resnet18}
\end{figure}

\begin{table}[th]
\caption{ResNet20 with $\leakyrelu$ CIFAR-10 accuracy}
\label{table:leaky-acc}
\begin{center}
\begin{tabular}{r|rrrrrrr}
\multicolumn{8}{c}{\textsc{$\leakyrelu$ slope}}\\\hline\hline
 & $1\mathrm{e}{-4}$ & $1\mathrm{e}{-3}$ & $1\mathrm{e}{-2}$ & 0.1 & 0.5 & 0.9 & 1.0
 \\ \hline 
 \% acc. & $89.3 \pm 0.2$ & $89.4 \pm 0.2$ & $89.2 \pm 0.2$ & $89.4 \pm 0.1$ & $86.6 \pm 0.1$ & $73.0 \pm 0.2$ & $41.8 \pm 0.1$\\
\end{tabular}
\end{center}
\end{table}

\begin{figure}
    \centering
    \includegraphics[width=0.95\linewidth]{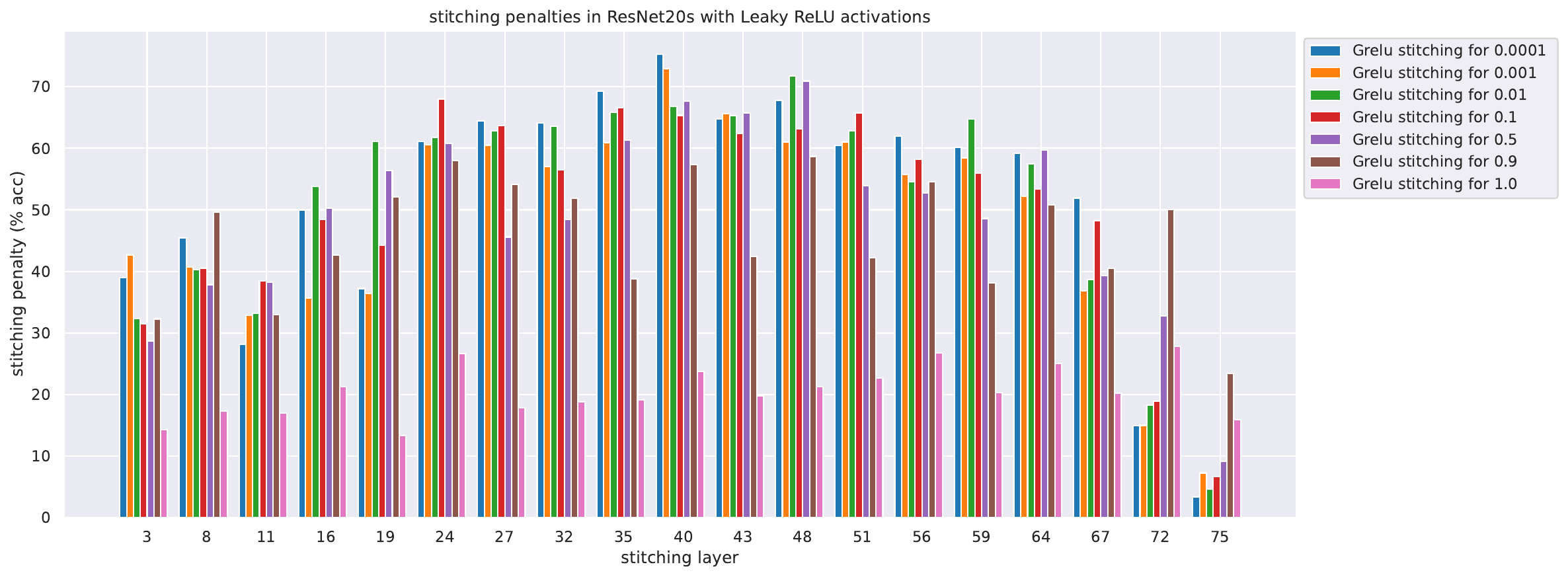}
    \caption{Stitching penalties \eqref{eq-stitching-penalty} for ResNet20s trained with different random seeds on CIFAR-10, where respective ResNet20 models are trained with $\leakyrelu$ activation functions with different slopes. Accuracy of the models with different $\leakyrelu$ slopes is given in \cref{table:leaky-acc}.}
    \label{fig:stitch-penalties-leaky}
\end{figure}

\Cref{fig:rn-pro} contains \(G_{\relu}\) and orthogonal Procrustes dissimilarities for the ResNet20. The 2 measures seem qualitatively quite similar in this case. For the most part the same applies to the ResNet18 in \cref{fig:rn-pro18}, with the exception of layer 70 (penultimate inside-a-residual-block layer), where we see high \(G_{\relu}\) \emph{similarity}, in conflict with both \cref{fig:stitch-penalties-resnet18} and \cref{fig:wreath-cka-resnet18} below. 

\begin{figure}[h]
    \centering
    \includegraphics[width=0.9\linewidth]{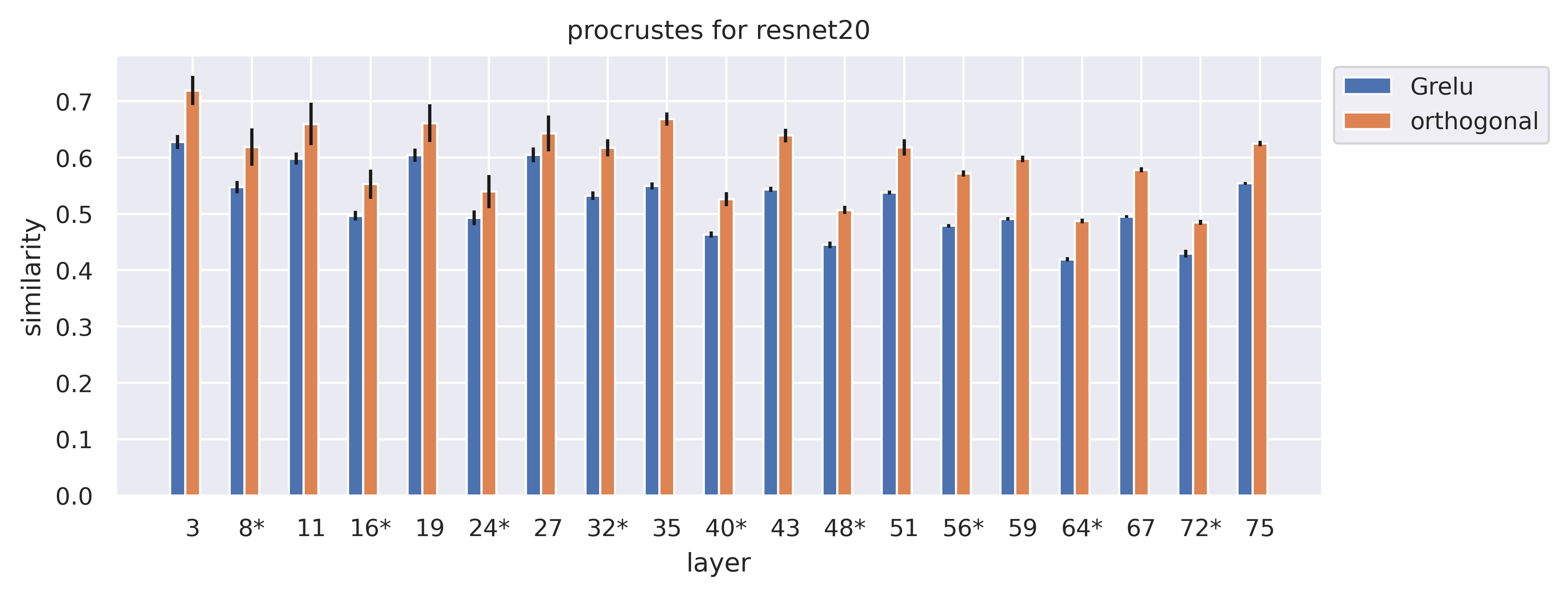}
    \caption{\(G_{\relu}\) and orthogonal Procrustes dissimilarities for two ResNet20s trained on CIFAR-10 with different random seeds. Layers marked with `*' occur inside residual blocks  (\cref{rmk:rn-failure}). Confidence intervals were obtained by evaluating similarities for 32 pairs of models trained with different random seeds.}
    \label{fig:rn-pro}
\end{figure}

\begin{figure}[h]
    \centering
    \includegraphics[width=0.9\linewidth]{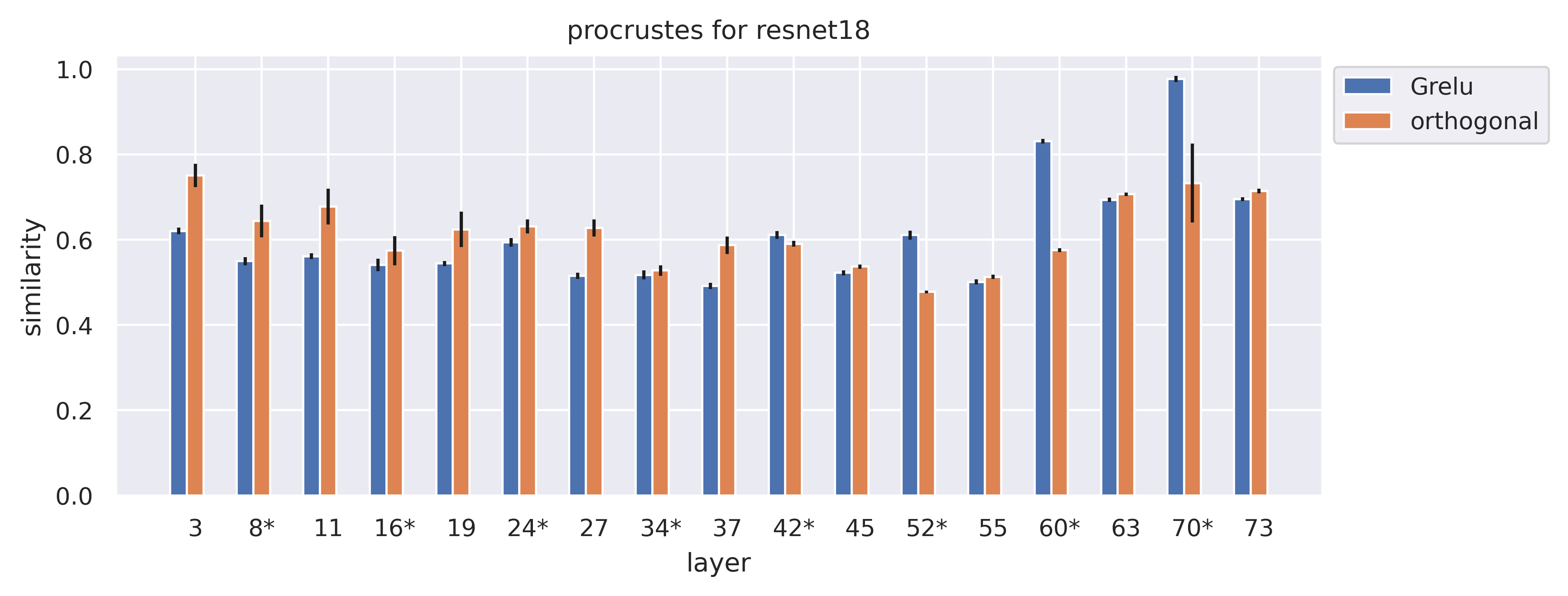}
    \caption{\(G_{\relu}\) and orthogonal Procrustes dissimilarities for two ResNet18s trained on CIFAR-10 with different random seeds. Layers marked with `*' occur inside residual blocks  (\cref{rmk:rn-failure}). Confidence intervals were obtained by evaluating similarities for 32 pairs of models trained with different random seeds.}
    \label{fig:rn-pro18}
\end{figure}

We include \(G_{\relu}\) and orthogonal CKA dissimilarities for the wider ResNet18 in \cref{fig:wreath-cka-resnet18}. For the most part the qualitative remarks on \cref{fig:wreath-cka-resnet} apply here as well --- note also the extreme dissimilarity in layer 70 (in both  \(G_{\relu}\) and orthogonal cases) consistent with \cref{fig:stitch-penalties-resnet18}.

\begin{figure}
    \begin{subfigure}[h]{0.5\linewidth}
        \centering
        \includegraphics[width=\linewidth]{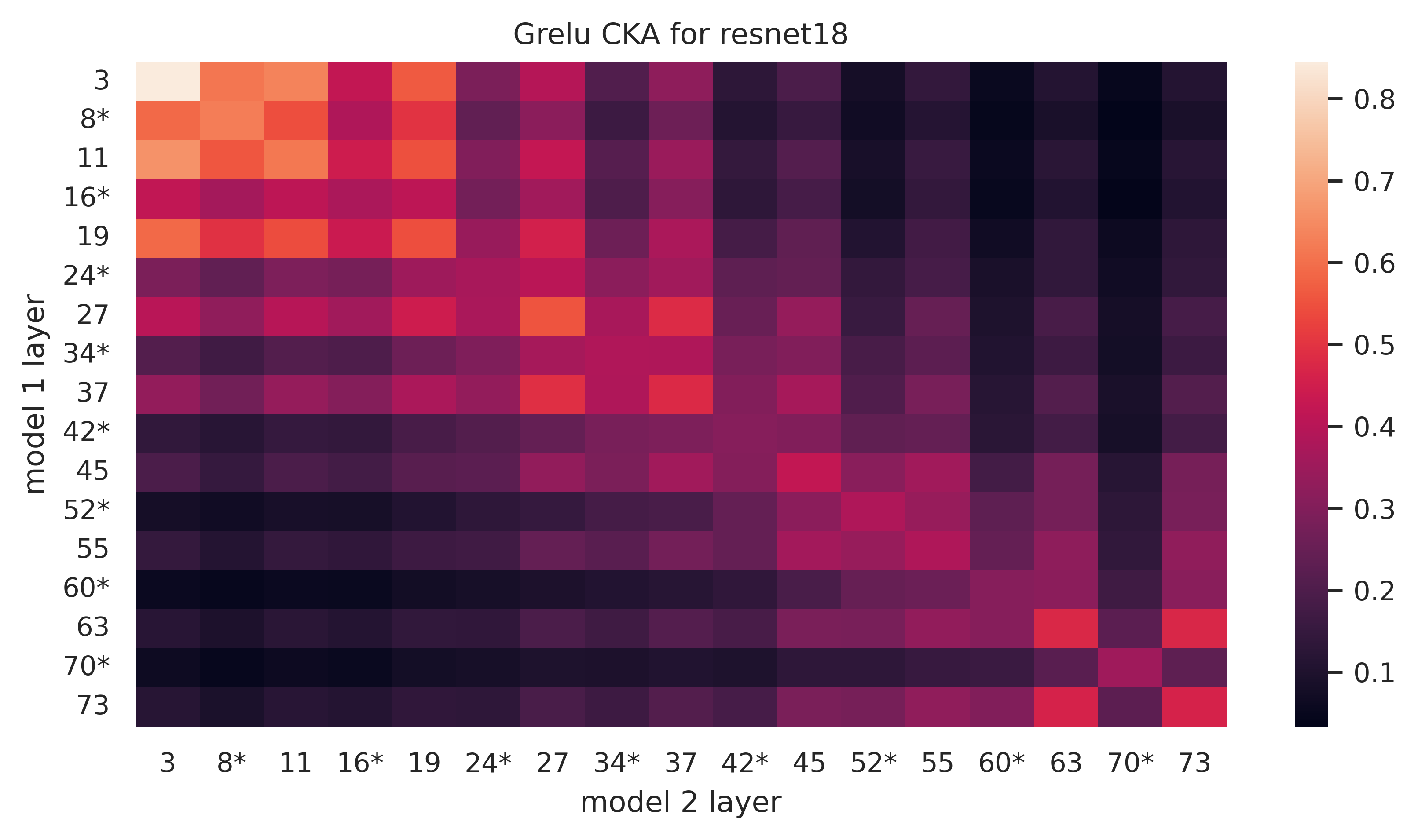}
        \subcaption[]{}
    \end{subfigure}
    \begin{subfigure}[h]{0.5\linewidth}
        \centering
        \includegraphics[width=\linewidth]{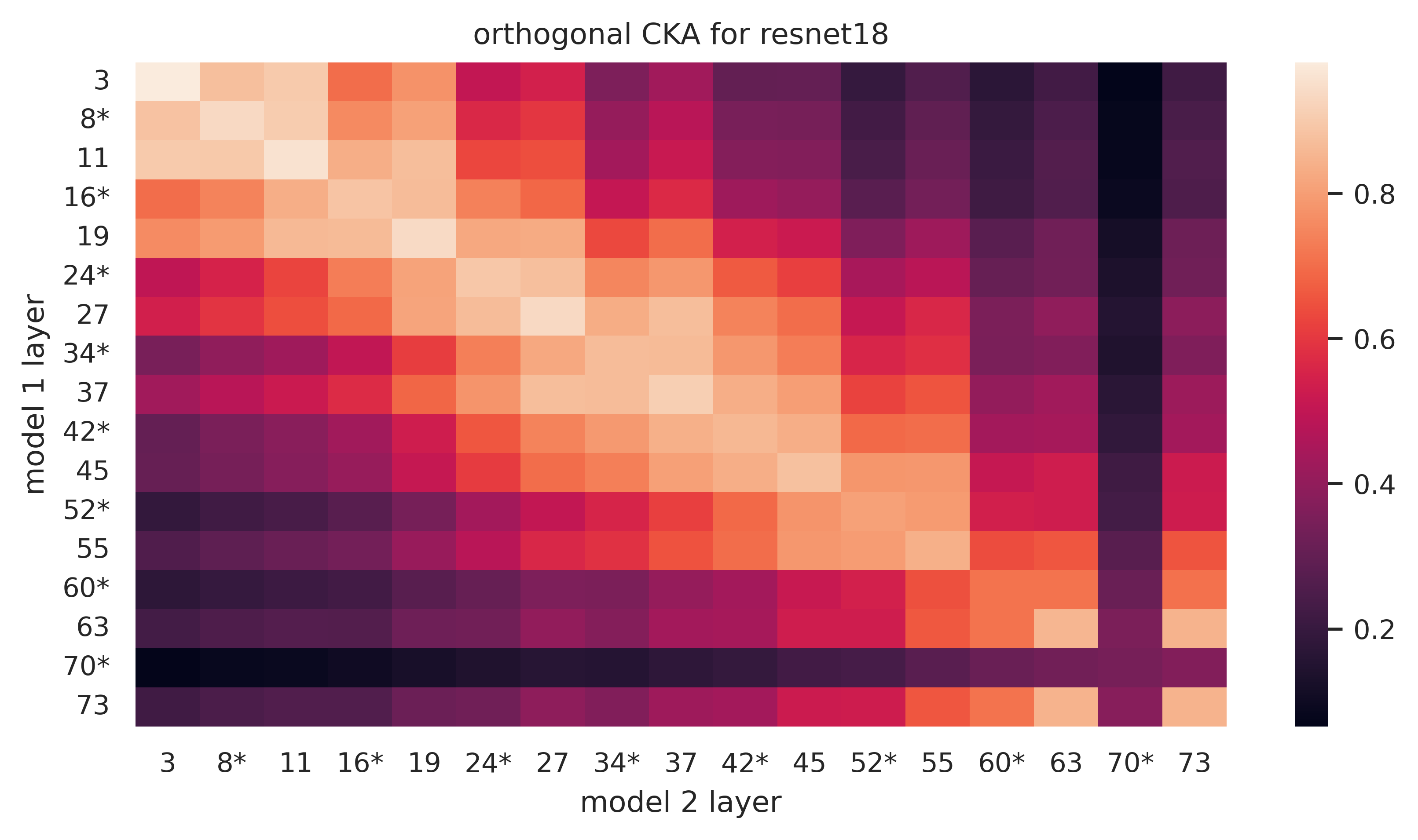}
        \subcaption[]{}
    \end{subfigure}
    \caption{\(G_{\relu}\)-CKA and orthogonal CKA for two ResNet18s trained on CIFAR-10 with different random seeds. Layers marked with `*' occur inside residual blocks  (\cref{rmk:rn-failure}). 
    Results averaged over 16 such pairs of models.
    }\label{fig:wreath-cka-resnet18}
\end{figure}

\subsection{Stitching for a Vision Transformer}
\label{sec:stitch-vit}

\begin{figure}
    \centering
    \includegraphics[width=0.95\linewidth]{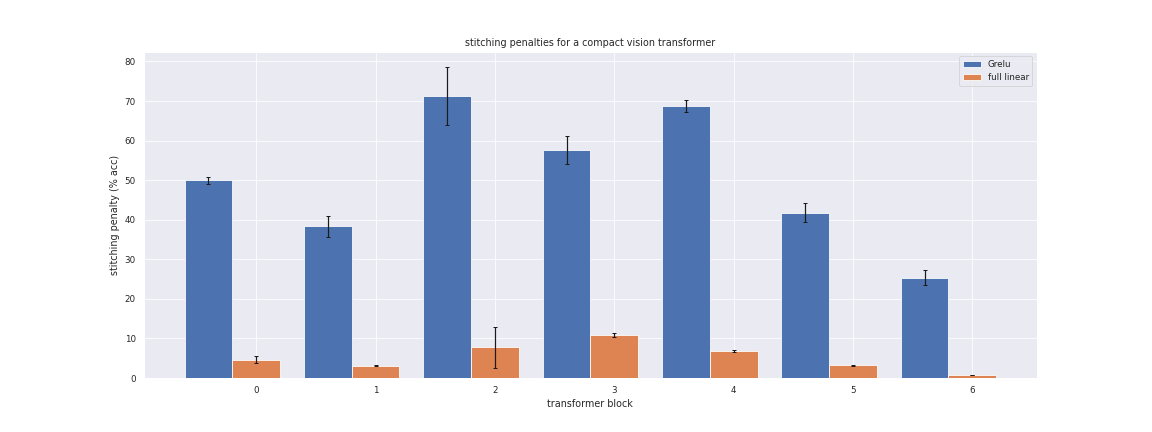}
    \caption{Linear and \(G_{\relu}\) stitching penalties \eqref{eq-stitching-penalty} for 5 pairs of vision transformers \cite{hassani2021escaping} trained on CIFAR-10 with different random seeds. Stitching was performed after every transformer block, and notably these blocks do not end in activation functions.}
    \label{fig:stitch-penalties-cvt}
\end{figure}

Here we include an additional stitching experiments with vision transformers from \cite{hassani2021escaping} trained on CIFAR-10. \Cref{fig:stitch-penalties-cvt} include results for linear stitching and \(G_{\relu}\) stitching after each transformer encoder layer. The large stitching penalties for \(G_{\relu}\) are expected due to the lack of activation functions after the linear (feedforward) layers for each encoder layer. 

We train 10 Compact Convolutional Transformers with sinusoidal positional encodings and six transformer blocks. The average model accuracy was $98\%$ using the distributed training-from-scratch recipe from \cite{hassani2021escaping}, which includes $6\mathrm{e}{-2}$ weight decay, augmentations (namely mixup \cite{zhang2018mixup} and CutMix \cite{yun2019cutmix}), label smoothing, and AdamW with a learning rate of $55\mathrm{e}{-5}$ with cosine scheduling.

\subsection{Choosing the negative-\(\ell_2\) regularization multiplier \(\alpha\) with cross validation}
\label{sec:crossval}

Here we briefly describe an experiment in which the multiplier \(\alpha\) of \cref{sec:approxoptperm} is chosen by cross validation. Most of the details are as in \cref{appendix-optimization-perm}. However, we create a random  80-20 split of the CIFAR10 training set into a smaller training and cross-validation set. We then learn \(G_{\relu}\) stitching layers  for each \(\alpha \in \{10^k \, | \, k = -3, -2, \dots, 1\}\), as in \cref{sec:approxoptperm}, with the exception that we only optimize over our training split for 5 epochs and do not give the permutations a head start. Then, the \(\alpha\) corresponding to highest accuracy on our cross validation set is selected, the corresponding model weights are loaded and we report accuracy on the regular CIFAR10 validation set. In \cref{fig:stitch-penalties-crossval} we obtain very similar results to those in \cref{fig:stitch-penalties}. Perhaps more interestingly, in \cref{fig:mCNN_best_alphas} we see that there is substantial variance in the \(\alpha\) selected by cross-validation, at all layers of our Myrtle CNN network --- for reference, \(\alpha=0.1\) is used in the rest of this paper. This suggests that the particular choice of \(\alpha\) is not essential to our method. Results for ResNet architectures are qualitatively similar and omitted for brevity.

\begin{figure}[h]
\centering
  \includegraphics[width=\linewidth]{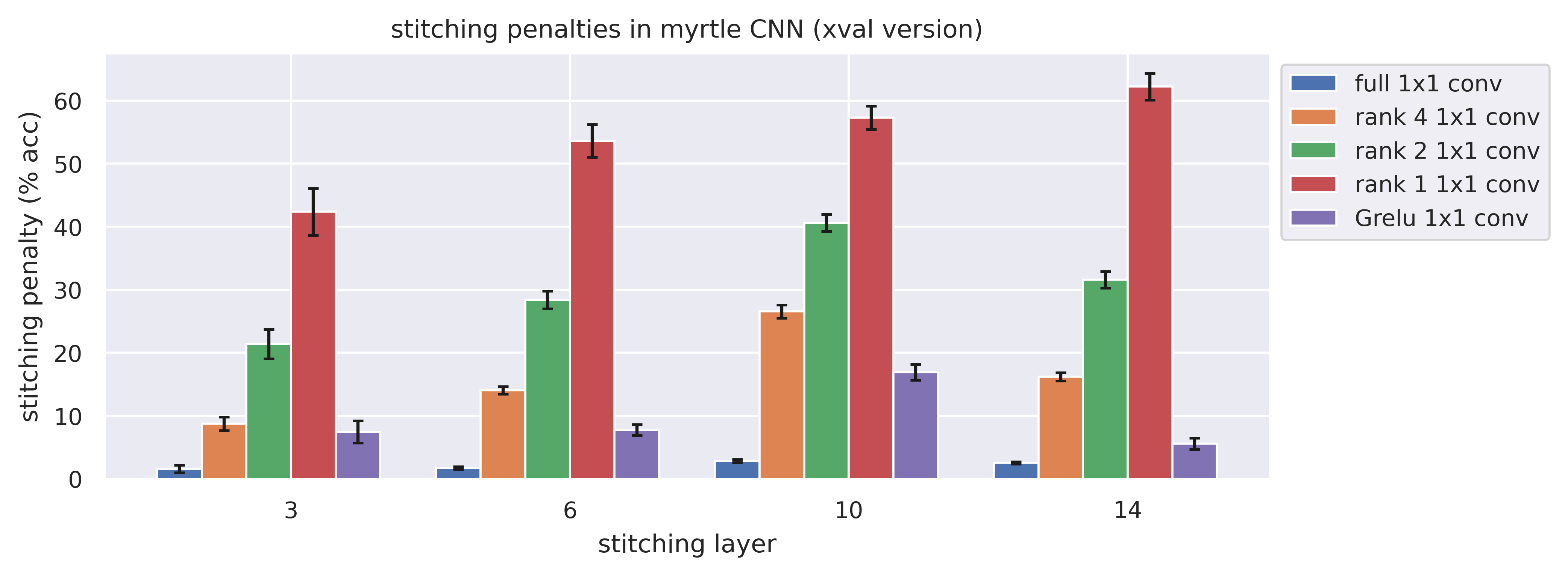}
  \caption[]{Full, reduced rank, and \(G_{\relu}\)  1-by-1 convolution stitching penalties \eqref{eq-stitching-penalty} for Myrtle CNNs \cite{pageHowTrainYour2018} on CIFAR-10, in which \(\alpha\) is chosen by cross-validation. Confidence intervals were obtained by evaluating stitching penalties for 32 pairs of models trained with different random seeds. The accuracy of the models was 91.3 \( \pm \) 0.2 \%. }\label{fig:stitch-penalties-crossval}
\end{figure}

\begin{figure}
    \centering
    \includegraphics[width=\linewidth]{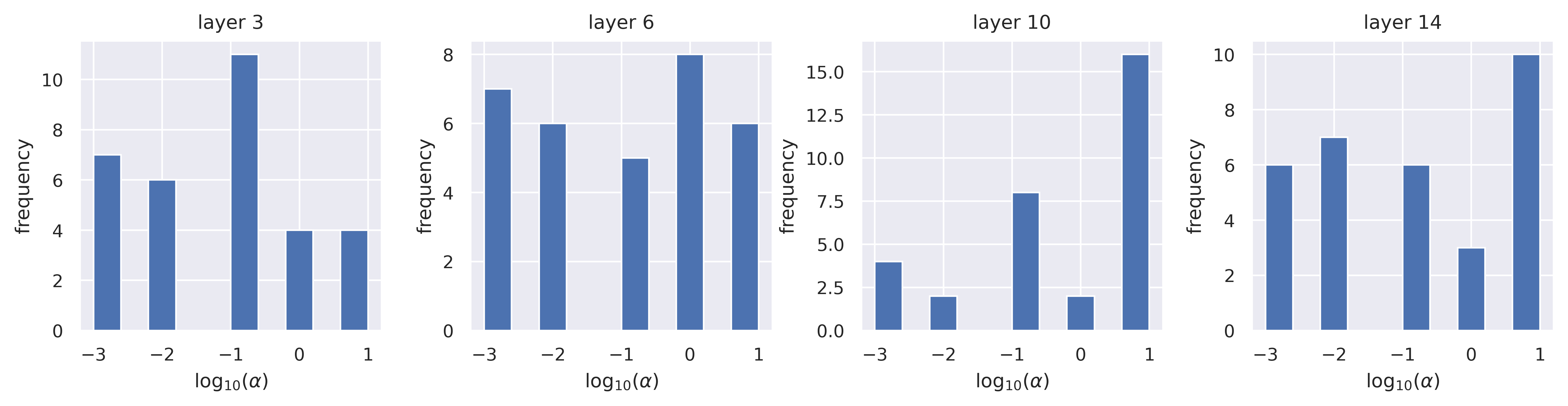}
    \caption{The histograms of \(\alpha\) selected by cross validation in the experiment of \cref{fig:stitch-penalties-crossval}}
    \label{fig:mCNN_best_alphas}
\end{figure}

\subsection{Stitching with \(\ell_1\)-regularized (a.k.a. LASSO) fully-connected layers}
\label{sec:lasso}

In this section we present results of a small experiment stitching with full 1-by-1 convolutional layers with \(\ell_1\) penalty \(\lambda \nrm{W}_1\), where \(\nrm{W}_1 = \sum_{ij} \nrm{W_{ij}}\), as in \cite{csiszarikSimilarityMatchingNeural2021}. We vary \(\lambda \in \{0.001, 0.01, 0.1\} \) and also tried \(\lambda = 1\) but found the stitching optimization to be unstable due the magnitude of the \(\ell_1\) penalty (possible this could have been counteracted by decreasing the learning rate). We also record the \emph{sparsity} of the stitching weights --- if \(n_l \) is the relevant channel dimension, and hence also the number of rows/columns in the square stitching matrix \(W\), we measure this as
\begin{equation}
    \label{eq:sparsity}
    \frac{\lvert \{ (i, j) \in \{1, \dots, n_l\}^2 \, | \, \nrm{W_{ij}} \leq \tau \} \rvert}{n_l^2}
\end{equation}
where \(\tau\) is a threshold, in our experiments chosen to be \(0.001\). Note that the sparsity of a  \(G_{\relu}\) is equal to \(\frac{n_l^2-n_l}{n_l^2} = 1-\frac{1}{n_l}\). \Cref{fig:mCNN-lasso} illustrates the results of these experiments, and seems to show that \(G_{\relu}\) layers achieve low stitching penalties for their sparsity levels. Also note that in the final layer the scatter points corresponding to \(G_{\relu}\) and \(\lambda = 0.01\) nearly overlap.

\begin{figure}
    \centering
    \includegraphics[width=\linewidth]{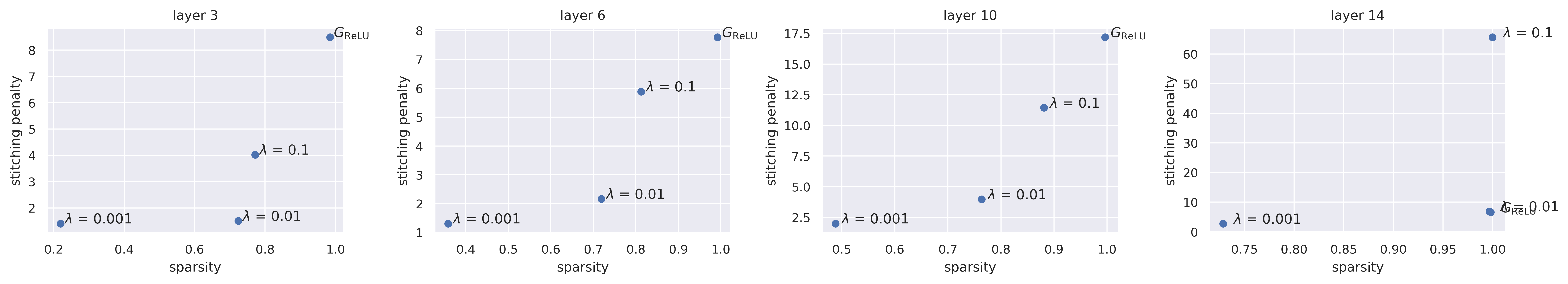}
    \caption{\(\ell_1\)-regularized stitching penalties versus sparsity for Myrtle CNNs, with \(G_{\relu}\) stitching penalties included for comparison. Penalties and sparsities are averaged over evaluations on 32 pairs models trained with different random seeds.}
    \label{fig:mCNN-lasso}
\end{figure}

\subsection{Implementing dissimilarity measures}
\label{sec:dissimilarity-details}

As mentioned in \cref{sec:shapemetrics}, we aim to capture invariants to
permuting and scaling channels, but not spatial coordinates. This requires some
care; practically speaking it means we cannot simply flatten feature vectors. 

In all cases we compute our measures over the entire CIFAR-10 validation set. In
particular, we do not require batched computations as in \cite{nguyenWideDeepNetworks2021}.

\subsubsection{Procrustes}

As in \cite{williamsGeneralizedShapeMetrics2022,ding2021grounding}
\[ \min_{P\in \Sigma_d} \nrm{\tilde{X}- \tilde{Y}P} = \sqrt{\min_{P\in \Sigma_d}
\nrm{\tilde{X}- \tilde{Y}P}^2} \] so it suffices to consider minimizing the
Frobenius norm-squared, and expanding as
\[\nrm{\tilde{X}- \tilde{Y}P}^2 = \nrm{\tilde{X}}^2 + \nrm{\tilde{Y}}^2 - 2
\tr(\tilde{X}^T \tilde{Y}P)\] we see that this is equivalent to
\emph{maximizing} \(\tr(\tilde{X}^T \tilde{Y}P)\). In our case \(X, Y\) have
shape \((N, C, H, W)\) where \(N\) is the size of the entire CIFAR-10 validation
set and \(C, H, W\) are the channels, height, and width at the given hidden
layer respectively. We want \(P\) to be a \(C\times C\) permutation matrix. Hence for
\(\tilde{X}^T \tilde{Y}\) we compute the tensor dot product
\begin{equation}
  \label{eq:tensordot}
  (\tilde{X}^T \tilde{Y})_{c, c'} = \sum_{n, h, w} \tilde{X}_{n, c, h, w} \tilde{Y}_{n,
c', h, w}
\end{equation}

The same method is used for orthogonal Procrustes, where instead of
\lstinline{scipy.optimize.linear_sum_assignment}  we use the nuclear norm of
\cref{eq:tensordot} as in \cite{dingGroundingRepresentationSimilarity2021}.

\subsubsection{CKA}

In this case for a set of hidden features \(X\) of shape \((N, C, H, W)\) as
above, we first subtract the mean over all but the channel dimension:
\[ X_{n, c, h, w} \gets X_{n, c, h, w} - \frac{1}{NHW} \sum_{n', h', w'} X_{n',
c, h', w'} \]
and divide by the norms over all but the channel dimension:\footnote{In
retrospect, it would arguably make more sense to use standard deviation rather than
\(\ell_2\) norm; however, for us the choice is irrelevant in the end since the 2
choices differ by a factor of \(\sqrt{NHW}\) which gets cancelled in \cref{eq:muCKA}.}
\[X_{n, c, h, w} \gets \frac{X_{n, c, h, w}}{\sqrt{\sum_{n', h', w'} X_{n',
c, h', w'}^2}}. \]
Next, we compute a tensor dot product of \(X\) with itself \emph{over spatial
dimensions}, to obtain the shape \((N, N , C)\) tensor
\[ J_{m, n, c} := \sum_{h, w}  X_{n, c, h, w} X_{n, c, h, w} \]
and finally we apply \(\max\) over the channel dimension to get 
\[ K_{m, n} =  \max_c J_{m, n, c}.  \]

\begin{remark}
  It could be interesting to refrain from applying a dot product over spatial
  dimensions, and thus measure not only similarity between hidden features of
  different images, but similarity between hidden features of
  different images \emph{at certain locations}. However, the memory requirements would have been far beyond our computational limits.
\end{remark}

\subsection{Dissimilarity measures for network with constant channel width}

A notable feature of our plots in \cref{fig:wreath-cka,fig:wreath-cka-resnet,fig:wreath-cka-resnet18} is that the \(G_{\relu}\)-CKA exhibits a much more significant decay with network depth than its orthogonal counterpart. From a skeptical perspective, we thought this could have something to do with dimensionality. All the networks we looked at up to this point had the feature that their channel dimension grows exponentially with depth (as seen in the last 3 figures of the appendix). When we compute the kernels $\mathrm{max} (\tilde{x}\_i \odot \tilde{x}\_j)$, we encounter maxima of larger and larger sets of random variables as the channel dimension increases. \emph{If} the products inside these maxima were independent normal random variables (we are not claiming this is a reasonable heuristic), we'd expect the max to grow like $\Phi^{-1}(1-\frac{1}{n_l})$ where $n_l$ is the channel dimension. It seemed possible that something along these lines could cause $G_{\mathrm{ReLU}}$-CKA to drift as depth (in our experiments correlated with channel dimension) increases. Note that the dot product kernel seems comparatively immune, since (with the same heuristics of normal distribution) the expected value of $\langle \tilde{x}, \tilde{y} \rangle $ is 0 regardless of dimension.

Motivated by this train of thought, we evaluated all 4 dissimilarity measures of \cref{sec:shapemetrics} on a variant of our Myrtle CNN with \emph{constant channel dimension}. The architecture of this network is identical to the one shown in \cref{fig:arch-mCNN} with the exception that all channel dimensions are 512. In \cref{fig:cw-wreath-procrustes,fig:cw-wreath-cka} we see that these constant width CNNs exhibit qualitatively very similar dissimilarity measures as their non-constant width counterparts. This suggests that the \(G_{\relu}\)-CKA  decay with network depth is \emph{not} an artifact of increasing channel dimension.

We speculate that it's possible that the decay of  $G_{\mathrm{ReLU}}$-CKA is due to something like the \emph{superposition hypothesis} for hidden layer features \cite{olah2020zoom,elhage2022solu}. Roughly, in overcomplete cases where the model can use more features than basis directions in a hidden layer, it may be encoding $m$ nearly orthogonal features across $n<m$ basis directions. If this encoding is not consistent across random seeds, we expect $G_{\mathrm{ReLU}}$-CKA to be smaller. Finally, polysemanticism may increase with depth. In a simple thought experiment, if each basis direction in layer $l$ has $a$ features encoded, layer $l+1$ will have $2a$ features per direction if it each neuron in $l+1$ simply sums over two neurons in $l$. Again assuming the combinations of features occuring in this polysemanticism vary accross random seeds, we would expect $G_{\mathrm{ReLU}}$-CKA to be smaller. Simply put, superposition and polysemanticism would seem to preclude alignment of the hidden features of different networks with permutations and scaling alone.

\begin{table}[]
    \begin{tabular}{lllll}
\toprule
{} &          layer 3 &          layer 6 &         layer 10 &         layer 14 \\
\midrule
\(G_{\relu}\)    &  0.8176 \(\pm\) 0.007 &  0.7602 \(\pm\) 0.005 &  0.5691 \(\pm\) 0.005 &  0.4971 \(\pm\) 0.003 \\
orthogonal &  0.8460 \(\pm\) 0.008 &  0.6735 \(\pm\) 0.005 &  0.5409 \(\pm\) 0.003 &  0.6050 \(\pm\) 0.002 \\
\bottomrule
\end{tabular}
\caption{\(G_{\relu}\) and orthogonal Procrustes similarities for \emph{constant channel width} Myrtle CNNs trained on CIFAR-10. Confidence intervals were obtained by evaluating similarities for 4 pairs of models trained with different random seeds. 
}\label{fig:cw-wreath-procrustes}
\end{table}

\begin{figure}
    \begin{subfigure}[h]{0.5\linewidth}
        \centering
        \includegraphics[width=\linewidth]{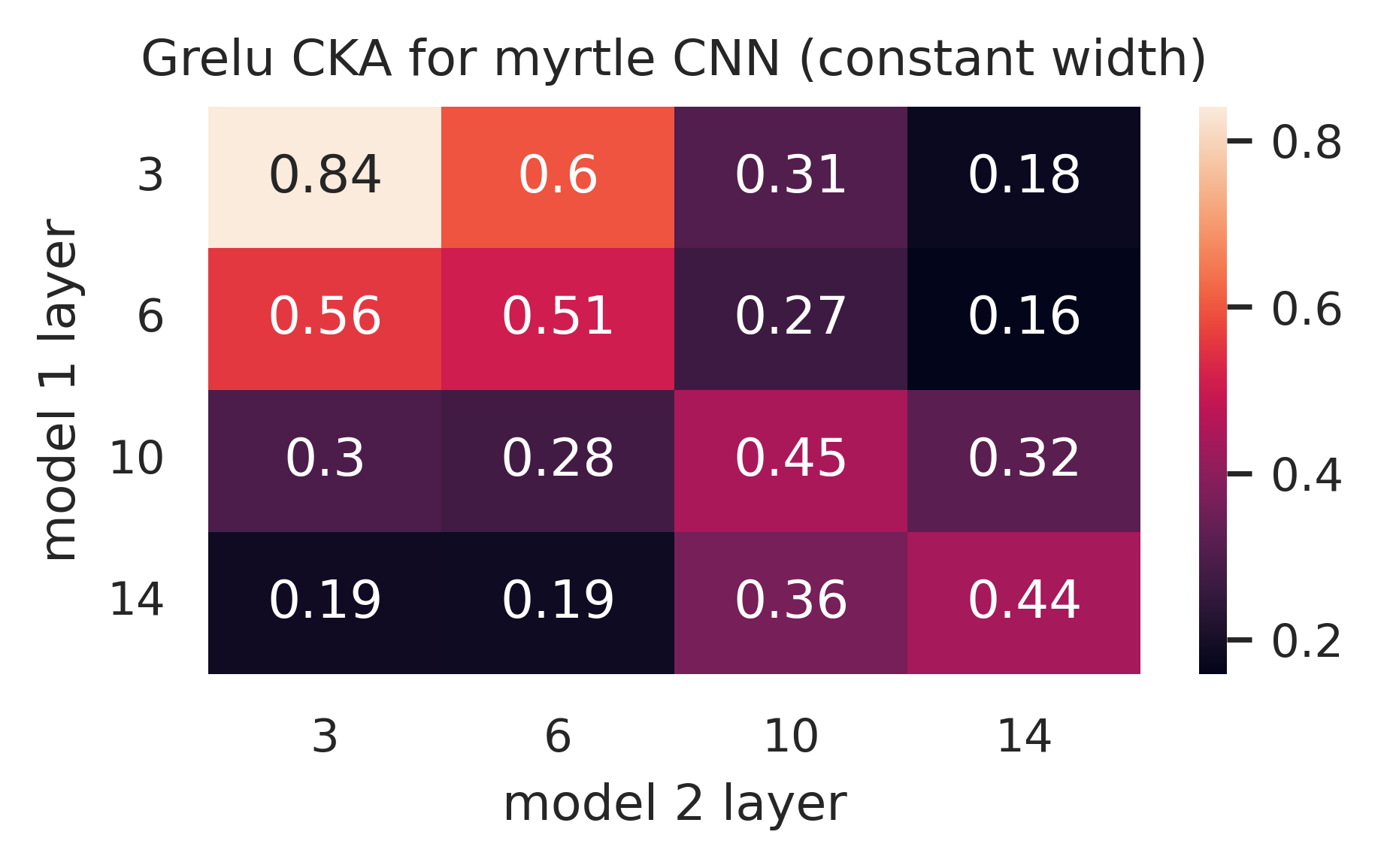}
        \subcaption[]{}
    \end{subfigure}
    \begin{subfigure}[h]{0.5\linewidth}
        \centering
        \includegraphics[width=\linewidth]{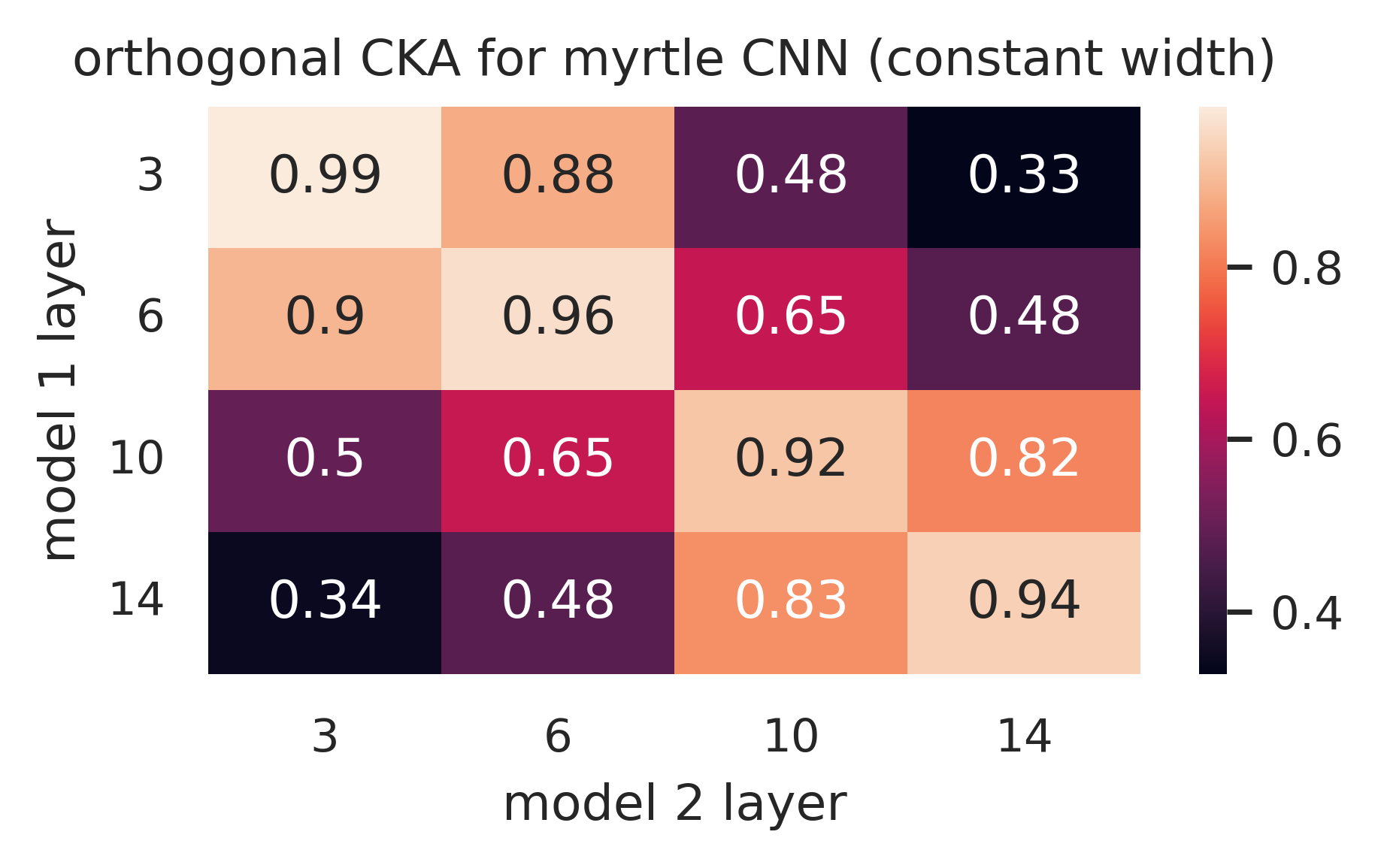}
        \subcaption[]{}
    \end{subfigure}
    \caption{\(G_{\relu}\)-CKA and orthogonal CKA for two \emph{constant channel width} Myrtle CNNs trained on CIFAR-10 with different random seeds. Results averaged over 4 such pairs of models 
    .}\label{fig:cw-wreath-cka}
\end{figure}

\section{Proofs}
\label{appendix-proofs}

\subsection{A proof of \cref{lem:intertwiner}, plus some abstractions thereof}

\begin{proof}[Proof of  \cref{lem:intertwiner}]
  Since by definition \(G_{\sigma_n} \subseteq GL_n(\RR)\), to prove
  \(G_{\sigma_n}\) is a subgroup it suffices to show that if \(A_1, A_2 \in
  G_{\sigma_n}\) then \(A_1 A_2^{-1} \in G_{\sigma_n}\). By hypotheses, there
  are matrices \(B_1, B_2 \in GL_n(\RR)\) so that 
  \begin{align}
    \sigma_{n} \circ A_1 &= B_1 \circ \sigma_{n} \label{eq:li1} \\
     \text{  and      }\quad  \sigma_{n}  \circ A_2 &= B_2 \circ \sigma_{n}. \label{eq:li2}
  \end{align}
  Applying \(A_2^{-1}\) on the right hand side of \cref{eq:li1} gives 
  \begin{equation}
    \label{eq:li1mod}
    \sigma_n  \circ (A_1 A_2^{-1}) = B_1 \circ \sigma_n \circ (A_2^{-1}).
  \end{equation}
  On the other hand, applying
  \(A_2^{-1}\) on the right hand side of \cref{eq:li2} gives \(\sigma_n  =
  B_2 \circ \sigma_n \circ (A_2^{-1})\) and hence  
  \begin{equation}
    \label{eq:li2mod}
    \sigma_n \circ  (A_2^{-1}) = B_2^{-1} \circ \sigma_n .
  \end{equation}
  Combining \cref{eq:li1mod,eq:li2mod} we obtain 
  \begin{equation}
    \label{eq:li3}
    \sigma_n \circ (A_1 A_2^{-1}) = B_1 B_2^{-1} \circ \sigma_n 
  \end{equation} 
  and hence \(G_{\sigma_n}\) is a subgroup. Next, we solve 
  \[ \sigma_{n} \circ A = B \circ \sigma_{n} \] for \(B\) in terms of \(A\) by
  evaluating both sides at \(e_1,\dots, e_n \in \RR^n\) (standard basis
  vectors). Letting \(A[:, j]\) denote the \(j\)-th column of \(A\) we obtain 
  \[ \sigma_n(A[:, j]) = B \sigma_n(e_j), \text{ for } j=1, \dots, n \] and
  stacking these columns to obtain the full \(n\times n\) matrix yields
  \[ \sigma(A) = B \sigma(I) \] where \(I \in GL_n(\RR)\) is the identity matrix
  and  \(\sigma(A)\) denotes \(\sigma\) applied to the coordinates of \(A\)
  (similarly for \(\sigma(I)\)). As \(\sigma(I)\) is invertible by hypotheses,
  this implies \( B = \sigma(A)  \sigma(I)^{-1} =: \phi_\sigma(A)\), and finally
  substituting \(B_i  = \phi_\sigma(A_i) \) for \(i =1, 2\) in \cref{eq:li3}
  shows that 
  \[\sigma_n \circ (A_1 A_2^{-1}) = \phi_\sigma(A_1) \phi_\sigma(A_2)^{-1} \circ
  \sigma_n \]
  while at the same time 
  \[ \sigma_n \circ (A_1 A_2^{-1}) = \phi_\sigma(A_1 A_2^{-1})  \circ \sigma_n
  \] so that \(\phi_\sigma(A_1 A_2^{-1})  \circ \sigma_n  = \phi_\sigma(A_1)
  \phi_\sigma(A_2)^{-1} \circ \sigma_n\). Using the invertibility of
  \(\sigma(I)\) one more time, we conclude 
  \[ \phi_\sigma(A_1 A_2^{-1})   = \phi_\sigma(A_1) \phi_\sigma(A_2)^{-1},   \]
  which implies \(\phi_\sigma\) is a homomorphism.
\end{proof}

\begin{remark}[for the mathematically inclined reader]
  Here is a more abstract definition of \(G_{\sigma_n}\) that makes
\cref{lem:intertwiner} appear more natural: let \(X\) be a topological space with a
continuous (left) action of a topological group \(G\). There
is a natural (right) action of \(G\) on \(C(X, \RR) \) by precomposition (\((f,
g) \mapsto f \circ g\)). For any subspace \(V \subseteq
C(X, \RR)\)  define
\[ G_V := \{g \in G \, | \, V \cdot g \subseteq V\} \] (that is, the elements of \(G\)
stabilize \(V\) \emph{as a subspace}, but not necessarily pointwise --- one can show this is always a subgroup of \(G\)). Then, for every such subspace \(V\), the group
\(G_V\) acts linearly on \(V\), and if we have a basis \(f_1, \dots, f_n \in V\), we
can obtain a matrix representation of \(G\) in \(GL_n(\RR)\). To obtain the special case in \cref{lem:intertwiner}, we take \(X = \RR^n\), \(G
= GL_n(\RR)\) with the usual action, and \(V \) to be the subspace spanned by
the functions \(f_i(x_1, \dots, x_n) = \sigma(x_i) \). The condition that
\(\sigma(I)\) is invertible is equivalent to the condition that the column space of the matrix  \((f_i(e_j)\) is \(n\)-dimensional, which in turn implies  \(V\) is \(n\)-dimensional.
\end{remark}

We end this section with a lemma that allows for easy verification
that \(\sigma(I)\) is invertible. We used this on all the activation functions considered
in \cref{fig-intertwiner-groups}.

\begin{lemma}
  \label{lem:inv-of-sigI}
  Let \(\sigma: \RR \to \RR\) be any function and let \(I \in GL_n(\RR)\) be the
  identity matrix. Then \(\sigma(I)\) is invertible provided 
  \begin{equation}
    \label{eq:inv-condition}
    \sigma(1) \neq \sigma(0) \text{  and  } \sigma(1) \neq -(n-1)\sigma(0).
  \end{equation}
\end{lemma}
\begin{proof}
  Let \(N = \mathbf{1}\mathbf{1}^T - I \). Then 
  \[ \sigma(I) = \sigma(1)I + \sigma(0) N.  \]
  Note that the eigenvalues of \(\mathbf{1}\mathbf{1}^T\) are \(n\)
  (corresponding to eigenvector \(\mathbf{1}\)) and 0's (corresponding to the
  orthogonal complement of \(\mathbf{1}\)). For any linear operator $A \in \Mat_n(\mathbb{R})$ with eigenvector/eigenvalue pair $(v,\lambda)$, $v$ is easily seen to be an eigenvector of $A - I$ with eigenvalue $\lambda - 1$. Hence the eigenvalues of \(N\) are
  \(n-1\)  and \(-1\), and it follows that the eigenvalues of
  \(\sigma(I)\) are 
  \[ \sigma(1) + \sigma(0)(n-1), \sigma(1)-\sigma(0), \dots,
  \sigma(1)-\sigma(0)\]
  which are all non-zero if and only if \cref{eq:inv-condition} holds.
\end{proof}

\begin{remark}
  \label{rmk:commuting}
  In particular \cref{eq:inv-condition} holds when \(\sigma(1)= 1, \sigma(0) =
  0\) (which holds for example when \(\sigma(x) = \relu(x)\) or \( \sigma(x) =
  x^d\)). In this situation, \(\sigma(I) = I\) and \(\phi_\sigma(A) =
  \sigma(A)\) (coordinatewise application of \(\sigma\)). For example, if
  \(\sigma \geq 0\) is non-negative, then \(\phi_\sigma(A) =
  \sigma(A)\) has non-negative entries.
\end{remark}

\subsection{Calculating intertwiner groups (for \cref{fig-intertwiner-groups})}
\label{sec:calc-intertwiner-grps}

We begin with two lemmas: the first puts a ``lower bound'' on \(G_{\sigma_n}\) 
and the second is a ``differential form'' of the definition of the
intertwiner group from \cref{sec:notation}. Together, these two results
effectively allow us to reduce calculation of intertwiner groups to the \(n=1\) case.

\begin{lemma}[{cf. \cite[\S 8.2.2]{Goodfellow-et-al-2016}, \cite[\S 3]{breaWeightspaceSymmetryDeep2019}}]
    \label{calc:perms}
    \(G_{\sigma_n}\) always contains the permutation matrices \(\Sigma_n\), and
    \(\phi_\sigma\) restricts to the identity on \(\Sigma_n\).
\end{lemma}
\begin{proof}
    If \(A \in \Sigma_n\) is a permutation matrix, so that \(Ae_i = e_{\pi(i)}\)
    where \(\pi\) is a permutation of \(\{1, \dots, n\}\) then for any \(x \in
    \RR^n\) we observe
    \[ A \sigma(x) = A(\sum_i \sigma(x_i) e_i) = \sum_i
    \sigma(x_i) A e_i = \sum_i \sigma(x_i) e_{\pi(i)} \] which is exactly
    \(\sigma\) applied coordinatewise to \(\sum_i x_i e_{\pi(i)} = Ax\).
\end{proof}

\begin{corollary}
  \label{cor:perms}
  If \(A \in GL_n(\RR), P \in \Sigma_n\), and \(AP \in G_{\sigma_n}\) or \(PA \in
  G_{\sigma_n}\), then \(A \in G_{\sigma_n}\).
\end{corollary}
\begin{proof}
  If \(B = AP \in G_{\sigma_n}\), then \(A = BP^{-1}\), where \(B \in
  G_{\sigma_n}\) by hypothesis and \(P \in G_{\sigma_n}\) by \cref{calc:perms}.
  The result follows as \(G_{\sigma_n}\) is a group (\cref{lem:intertwiner})
  and hence closed under multiplication. The other case is similar. 
\end{proof}

\begin{lemma}
  \label{lem:diff-form}
  Suppose \(A, B \in GL_n(\RR)\) and \( \sigma_{n} \circ A = B \circ
  \sigma_{n}\). Suppose \(x = (x_1,\dots,x_n)^T \in \RR^n\) and assume
  \(\sigma\) is differentiable at \(x_1, \dots, x_n\) as well as 
  \[ (Ax)_i = \sum_j a_{ij}x_j, \quad\text{ for } i=1,\dots, n. \]
  Then, 
  \[ \diag(\sigma'((Ax)_i)|i=1,\dots,n) A = B \diag(\sigma'(x_i)|i=1,\dots,n) \]
  (here \(\diag: \RR^n \to \Mat_{n\times n}(\RR)\) takes a vector to a diagonal
  matrix). Explicitly, for each \(i, j \in \{1, \dots, n\}\)
  \begin{equation}
    \label{eq:diff-form}
    \sigma'(\sum_k a_{ik}x_k) a_{ij} = b_{ij}\sigma'(x_j).
  \end{equation}
\end{lemma}

\begin{proof}
  By the chain rule \cite[Thm. 9.15]{rudinPrinciplesMathematicalAnalysis1976}, and since the
  differential of a matrix is itself,
  \[ d\sigma_n|_{Ax} A = B d\sigma_n|_x.\]
  Finally, by the definition of \(\sigma_n\)
  \[ \frac{\partial (\sigma_n(x))_i}{\partial x_j} = \frac{\partial
  \sigma(x_i)}{\partial x_j} = \begin{cases}
    \sigma'(x_j) & \text{  if  } i=j \\
    0 & \text{  otherwise}.
  \end{cases} \]
\end{proof}

\begin{theorem}
  \label{thm:struct-intertwiners}
  Suppose \(\sigma\) is non-constant, non-linear, and differentiable on a dense
  open set with finite complement.\footnote{The differentiability assumption is
  probably not necessary, however it holds in all of the examples we consider
  and allows us to safely use \cref{lem:diff-form}.} Then,
  \begin{enumerate}[(i)]
    \item \label{item:PD} Every \(A \in G_{\sigma_n}\) is of the form \(PD\), where \(P \in
    \Sigma_n\) and \(D\) is diagonal.
    \item \label{item:res-to-diag} For a diagonal \(D = \diag(\lambda_1, \dots, \lambda_n) \in
    G_{\sigma_n}\), we have \(\lambda_i \in G_{\sigma_1}\) for \(i=1, \dots, n\) and
    \[ \phi_\sigma (\diag(\lambda_1, \dots, \lambda_n)) = \diag
    (\phi_\sigma(\lambda_1), \dots, \phi_\sigma(\lambda_1)) \] where we make a
    slight abuse of notation: on the right hand side \(\phi_\sigma\) is the
    homomorphism \(G_{\sigma_1}\to GL_1(\RR)\).
  \end{enumerate}
  In particular, \(\phi_{\sigma}\) is determined by \cref{calc:perms} and its
  behavior for \(n=1\).
\end{theorem}

\begin{proof}
  For any \(A \in G_{\sigma_n}\) we observe that the differentiability hypotheses
  of \cref{lem:diff-form} holds for \(x \in U\) where \(U\) is a dense open set with
  measure-0 complement. Indeed, if \(t_1, \dots, t_M \in \RR\) are the points
  where \(\sigma\) fails to be differentiable, we can take \(U\) to be the
  complement of the hyperplane arangement given by
  \[ (\bigcup_{ij} \{x \in \RR^n \;|\; x_i =t_j\})\cup (\bigcup_{ij} \{x \in \RR^n \;|\; (Ax)_i \in =t_j \}) \subseteq \RR^n \]
  Fix a row \(i\) --- the matrix \(A \) is
  invertible by hypotheses, and so there must be some \(j\) such that \(a_{ij}
  \neq 0\) (otherwise the \(i\)-th row of \(A\) is 0).  For any \(x \in U\), we have
  by \cref{lem:diff-form}
  \begin{equation}
    \label{eq:diff-form-revisited}
    \sigma'(\sum_k a_{ik}x_k) a_{ij} = b_{ij}\sigma'(x_j)
  \end{equation}
  and we claim that this cannot hold unless \(a_{ik} = 0 \) for \(k \neq j\).
  First, there is a \((x_1,\dots, x_n) \in U\) such that \(\sigma'(x_j) \neq 0\)
  (otherwise \(\sigma\) would be constant). Next, fixing \(x_j\) at a value with
  \(\sigma'(x_j) \neq 0\) and rearranging \cref{eq:diff-form-revisited} we have 
  \begin{equation}
    \label{eq:diff-form-revisited2}
    \frac{\sigma'(a_{ij} x_j +\sum_{k \neq j} a_{ik}x_k)}{\sigma'(x_j)} a_{ij} = b_{ij}= \text{ constant}.
  \end{equation}
  By hypothesis \(\sigma\) is non-linear and so \(\sigma'\) is non-constant ---
  hence if there were some \(a_{ik} \neq 0\) for \(k \neq j\), the left hand
  side of \cref{eq:diff-form-revisited2} would be non-constant.

  We have shown each row of \(A\) has at most one non-0 entry \(a_{ij}\) and that
  \(a_{ij} \neq 0\). For \(A\) to be invertible, it must be that these non-0 entries
  land in distinct \emph{columns}. This is exactly the form described in
  \cref{item:PD}.
  
  Next, we note that for any \(ij\) (without assuming \(a_{ij}\neq 0\))
  \cref{eq:diff-form-revisited,eq:diff-form-revisited2} tell us \[a_{ij}=0
  \implies b_{ij}=0, \] and hence if \(D = \diag(\lambda_1, \dots, \lambda_n)
  \in G_{\sigma_n}\) and \(\sigma_n \circ D = E \circ \sigma_n\) (i.e. \(E =
  \phi_\sigma(D)\)), it must be that \( E = \diag(\mu_1, \dots, \mu_n) \) for
  some \(\mu_1, \dots, \mu_n \in \RR\). Now the equation \(\sigma_n \circ D = E
  \circ \sigma_n\) is equivalent to 
  \[ \sigma(\lambda_i x_i) = \beta_i \sigma(x_i) \text{  for  } i = 1, \dots, n
  \]
  which in turn is equivalent to \(\lambda_i \in G_{\sigma_1}\) and \(\beta_i =
  \phi_\sigma(\lambda_i)\) for \(i = 1, \dots, n\), proving \cref{item:res-to-diag}.
\end{proof}

In light of \cref{thm:struct-intertwiners}, to fill in the table of
\cref{fig-intertwiner-groups} it will suffice to deal with the \(n=1\) cases,
which we do below.

\begin{proof}[Calculation of \(G_{\relu}\)]  
  We remark that this is just the ``positive homogeneous'' property of
  \(\relu\), which is quite well known (cf. \cite[\S 8.2.2]{Goodfellow-et-al-2016},
  \cite[\S 2]{freemanTopologyGeometryHalfRectified2017}, \cite[\S
  3]{kuninNeuralMechanicsSymmetry2021}, \cite[\S 3]{mengGSGDOptimizingReLU2019},
  \cite[\S 3, A]{rolnickReverseengineeringDeepReLU2020}, \cite[\S
  2-3]{yiPositivelyScaleInvariantFlatness2019}). Using \cref{rmk:commuting} if
  \(a \in G_{\sigma_1}\)
  \[ \max \{0, ax\} = \max\{0, a\} \max \{0, x\}. \]
  If \(a <0\) then setting \(x = -1\) results in \(a = 0\), a contradiction. So
  \(a >0\) and \(\max \{0, ax\} = a \max \{0, x\}\), showing \(\phi_\sigma(a) = a\).
\end{proof}

\begin{proof}[Modifications for \(\leakyrelu\)]
  By \href{https://paperswithcode.com/method/leaky-relu}{definition}, for \(0 <
  s \ll 1\).
  \[ \leakyrelu(x, a) := \begin{cases}
    sx & \text{  for  } x <0 \\
    x & \text{  for  } x \geq 0
  \end{cases} \]
  which we may simplify to \(\leakyrelu(x) = sx + (1-s)\relu(x)\). 
  Suppose now that 
  \begin{equation}
    \label{eq:leaky}
    \begin{split}
      \leakyrelu(ax) &= b \leakyrelu(x), \text{  or using our simplification  } \\
    sax + (1-s)\relu(ax) &= b(sx + (1-s)\relu(x)).
    \end{split}
  \end{equation}
  If \(a<0\), we may choose \(x = -1\) to obtain
  \[-a =  -sa -(1-s)a = -bs \text{  and \(x=1\) to obtain} \]
  \[ sa = b \]
  showing that \( a =as^2 \), which is impossible when \(0 <
  s \ll 1\). So it must be \(a >0\), and then evaluating \cref{eq:leaky} at \(x
  =1\) gives \(a = b\).
\end{proof}

\begin{proof}[The sigmoid case: \(\sigma(x) = 1/(1+e^x)\)]
  We will leverage of a useful fact about the sigmoid function: 
  \begin{equation}
    \label{item:supp-not-cpct}
    \text{\(\sigma'(x)\) is a smooth probability distribution function on
    \(\RR\), with \(\sigma'(x) > 0\) for all \(x\in \RR\).} \tag{*}
  \end{equation}
  If \(\sigma(ax) = b\sigma(x)\), differentiating with respect to \(x\) gives
  \begin{equation}
    \label{eq:df-sig3}
    \sigma'(a x) a = b \sigma'(x).
  \end{equation}
  Using \cref{item:supp-not-cpct} and the fact that to probability distribution
  functions are proportional if and only if they are equal, we get
  \(\sigma'(ax) = \sigma'(x)\). Then integrating from \(-\infty \) to
  \(x\) tells us \(\sigma(ax)/a = \sigma(x)\); setting \(x
  = 0\) we see \(\frac{1}{2a} = \frac{1}{2}\), hence \(a = 1\).

  To show \(\phi_{\sigma} = \mathrm{id}\), backtracking to \cref{eq:df-sig3}  and
  setting \(x=0\) shows \(b = a\).
\end{proof}

\begin{proof}[The Gaussian RBF case: \(\sigma(x) = \frac{1}{\sqrt{2 \pi}} e^{-\frac{x^2}{2}}\)]
  We make use of several properties of this \(\sigma(x)\):
  \begin{enumerate}[(i)]
    \item \label{item:gaussian} For any \(a>0\) the function \(\sigma(ax)\) is a
    probability distribution function with mean \(0\) and variance
    \(\frac{1}{a^2}\),and with \(\sigma(ax) >0\) for all \(x \in \RR\) and
    \item \label{item:evenness} \(\sigma\) is an even function (\(\sigma(-x) = \sigma(x)\)).
  \end{enumerate}
  Now if \(\sigma(ax) = b\sigma(x)\), then the pdfs \(\sigma(ax)\) and
  \(\sigma(x)\) are proportional hence equal by \cref{item:gaussian}. Therefore they have the same means
  and variances --- since these are \(0, \frac{1}{a^2}\) and \(0, 1\)
  respectively we conclude \(a = \pm 1\).

  Finally, we explain why \(\phi_{\sigma}(A) = \abs(A)\) (entrywise absolute
  value). Differentiating with respect to \(x\) gives
  \begin{equation}
    \label{eq:df-rbf}
    \sigma'(ax) a = b \sigma'(x).
  \end{equation}
  This implies \(b = 1\) when \(a =1\). On the other hand differentiating
  \cref{item:evenness} tells us \(-\sigma'(-x) = \sigma'(x)\), so when \(a =
  -1\)
  \[ b \sigma'(x) = -\sigma'(-x) =  \sigma'(x)\] and hence \(b =
  1\).
\end{proof}

\begin{proof}[The polynomial case: \(\sigma(x) = x^d\)]
  We remark that the description given in \cref{fig-intertwiner-groups} is
  implicit in \cite{kileelExpressivePowerDeep2019}. By
  \cref{thm:struct-intertwiners} we only need to describe \(\phi_\sigma:
  G_{\sigma_1} \to GL_1(\RR)\); for any \(a \neq 0\)
  \[ (ax)^d = a^d x^d\]  
  and this shows \(G_{\sigma_1} = \RR \setminus \{0\}\) and \(\phi_\sigma(a) = a^d\).
\end{proof}

\subsubsection{Gaussian error linear units (\(\gelu\)s)}
\label{sec:gelu}

Introduced and first studied in \cite{hendrycksGeLU}, these are defined as
\(\gelu(x) = x\Phi(x)\) where \(\Phi\) is the standard normal cummulative
distribution function: 
\[ \Phi(x) = \int_{-\infty}^x \frac{e^{-\frac{t^2}{2}}}{\sqrt{2\pi}} \, dt. \]
By inspecting plots in \cref{fig:relugelu}, we see that \(\gelu\) and \(\relu\) are   globally quite
similar (they converge as \(\nrm{x} \to \infty\)) but that they differ when
\(x\) within a few standard normal standard deviations of 0. 
One can show that \(G_{\gelu_n} = \Sigma_n\): indeed, by
\cref{thm:struct-intertwiners} it suffices to show that the only \(\lambda \in
\RR\setminus \{0\}\) such that \(\gelu(\lambda x) = \phi(\lambda) \gelu(x)\) for
all \(x\) (where \(\phi\) is some non-zero function of \(\lambda \)) is \(\lambda=1\).
Expanding, we see that 
\begin{equation}
  \lambda x \Phi(\lambda x) = \phi(\lambda) x \Phi(x),
\end{equation}
and rearranging this becomes
\begin{equation}
  \frac{\Phi(\lambda x)}{\Phi(x)} = \frac{\phi(\lambda)}{\lambda} =: c,
\end{equation}
that is, the right hand side is constant as a function of \(x\). Then
\(\Phi(\lambda x) = c\Phi(x)\), and since \(\Phi\) is positive it must be \(c\)
is too. Moreover it must be \(\lambda > 0\), as otherwise \(\Phi(\lambda x) \)
is monotonically decreasing while \(\Phi(x)\) is increasing. Finally, letting
\(x \to \infty \) we see that \(c=1 \), and from there we conclude \(\lambda = 1
\) by an argument similar to the use of \cref{item:gaussian} in the Gaussian RBF
case.

Despite the above calculation, it seems natural to ask how far \(G_{\gelu}\)
is from \(G_{\relu}\), in other words how badly \(\gelu\) fails to be positive
homogeneous. One measure of this is obtained by letting \(X \sim \mathcal{N}(0,
1)\) be a standard normal variable and computing the root-mean-square error
\begin{equation}
  \label{eq:rms-non-pos-hom}
  \xi(\lambda) := \sqrt{E[ \lvert \gelu(\lambda X) - \lambda \gelu(X) \rvert^2] }
\end{equation}
as a function of \(\lambda > 0\), where the expectation is over \(X\). Here our
choice of a standard normal \(X\) is motivated by the same reasoning as
discissed in \cite{hendrycksGeLU}, namely that activation inputs are roughly
standard normal, especially in the presence of batch normalization. Evaluating \cref{eq:rms-non-pos-hom}
doesn't seem particularly tractible analytically, but it does simplify to 
\begin{equation}
  \label{eq:rms-non-pos-hom2}
\xi(\lambda) = \lambda \sqrt{E[ \lvert x (\Phi(\lambda x) - \Phi(x)) \rvert^2] }.\end{equation}
In \cref{fig:rms-non-pos-hom} we estimate \(\xi(\lambda)\) by sampling \(X\) and replacing the expectation
with the corresponding average.
\begin{figure}
\begin{subfigure}{0.5\linewidth}
  \centering
  \includegraphics[width=\linewidth]{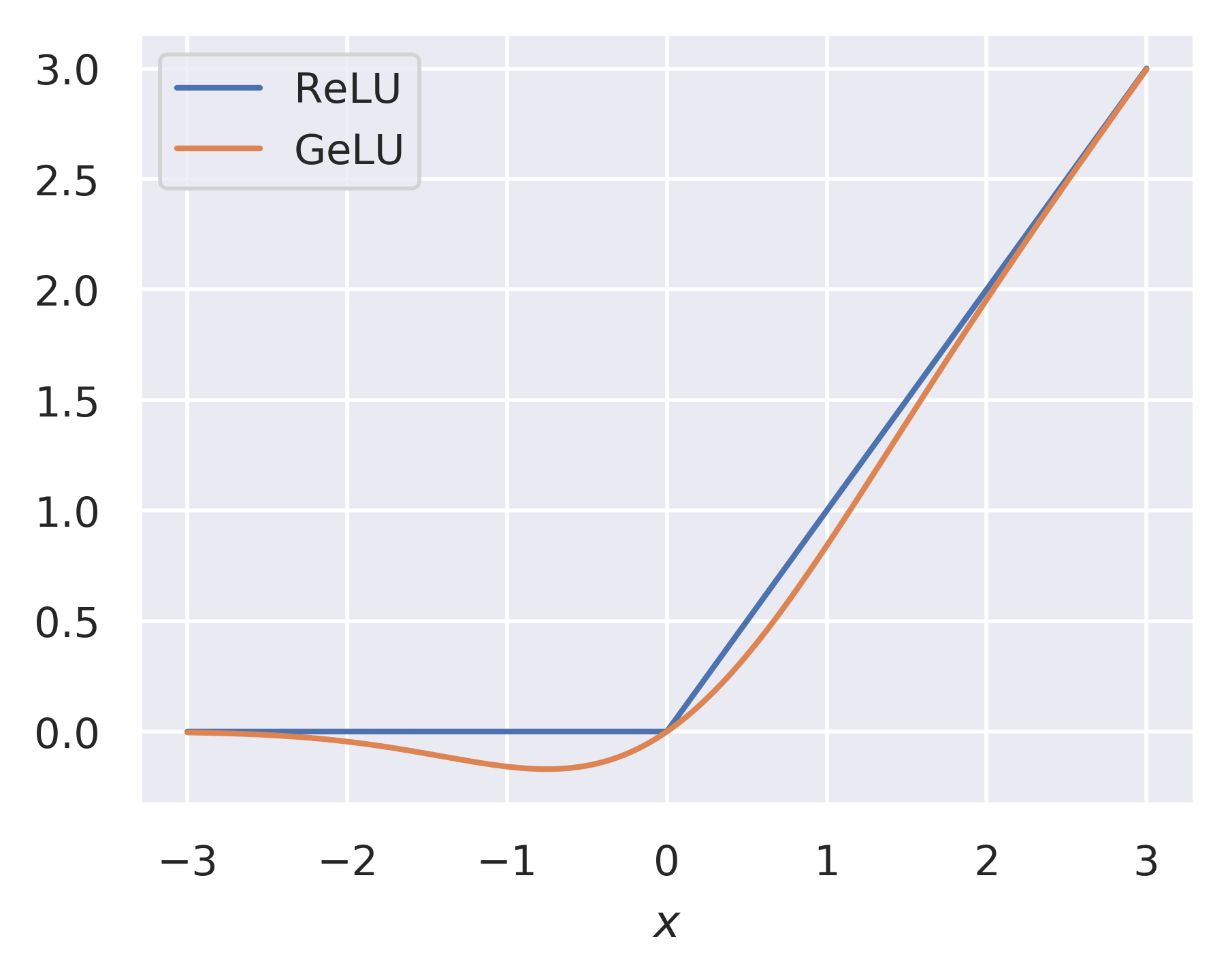}
  \caption[]{}\label{fig:relugelu}
\end{subfigure}
  \begin{subfigure}{0.5\linewidth}
  \centering
  \includegraphics[width=\linewidth]{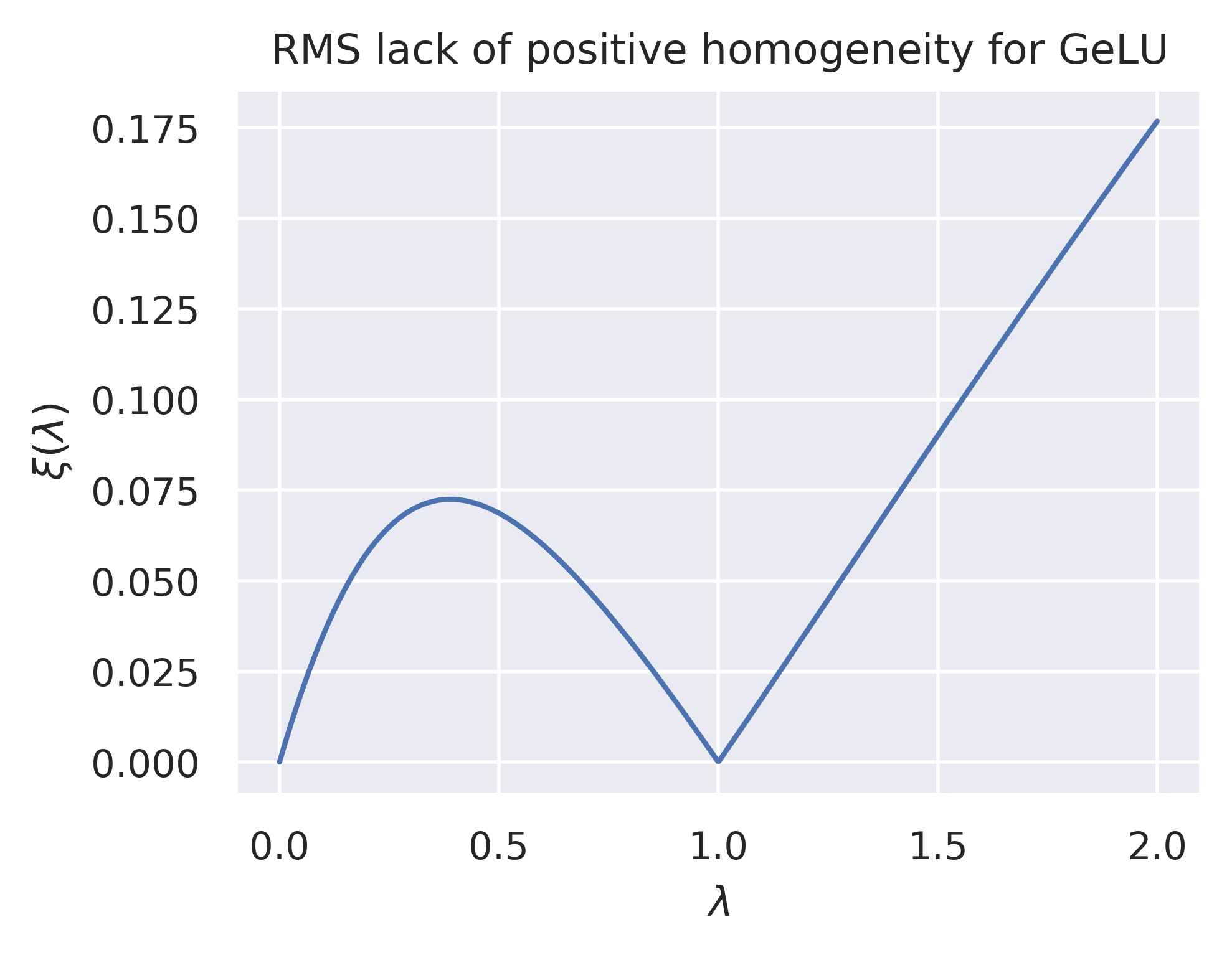}
  \caption[]{}\label{fig:rms-non-pos-hom}
  \end{subfigure}
  \caption{\textbf{(a)} The \(\relu\) and \(\gelu\) functions. \textbf{(b)} Root-mean-square lack of positive homogeneity for the \(\gelu\) function, estimated using \(10^5\) samples of \(X\).}
\end{figure}
Evidently, as \(\lambda \to \infty\) the function \(\xi(\lambda) \) becomes
linear: since \(\Phi(\lambda x) \to \mathbf{1}_{x\geq 0} \) as \(\lambda \to
\infty \) (here \(\mathbf{1}_{x\geq 0}\) is the indicator of
\(x>0\), also known as the Heaviside or unit-step function), the asymptotic
slope is \(\sqrt{E[ \lvert x (\mathbf{1}_{x\geq 0} - \Phi(x)) \rvert^2] }
\approx 0.127\).

\subsection{Proof of \cref{thm:stabilizer1}}

\begin{proof}
    The explicit description of  \(G(\relu , n  )\) in \cref{fig-intertwiner-groups} is enough
    to show \(G(\relu , n  )\) \emph{stabilizes} \(\{ \RR_{\geq 0} e_i | i = 1,
    \dots, n \}\). Indeed, any \(A \in G(\relu , n  )\) may be written as \(PD\)
    where \(D = \diag(a_i)\) for some \(a_i >0\)  and \(P\) is a permutation matrix associated
    to a permutation \(\pi\). It suffices to show that \(P\) and \(\diag(a_i)\)
    each preserves \(\{ \RR_{\geq 0} e_i | i = 1, \dots, n \}\). First, 
    \[ \diag(a_i) \{ \RR_{\geq 0} e_i | i = 1, \dots, n \} = \{ \RR_{\geq 0} a_i
    e_i | i = 1, \dots, n \}  = \{ \RR_{\geq 0} e_i | i = 1, \dots, n \}\] since
    the \(a_i >0\) so scaling by \(a_i\) preserves the ray \(\RR_{\geq 0} e_i\).
    Second,
    \[P\{ \RR_{\geq 0} e_i | i = 1, \dots, n \} = \{ \RR_{\geq 0} Pe_i | i = 1,
    \dots, n \} = \{ \RR_{\geq 0} e_{\pi(i)} | i = 1, \dots, n \}= \{ \RR_{\geq
    0} e_i | i = 1, \dots, n \}\]
    (the last equality is due to the fact that we only consider the set of rays,
    not the ordered tuple of rays).    

    Conversely, if \(A \in GL(n)\) stabilizes \(\{ \RR_{\geq 0} e_i | i = 1,
    \dots, n \}\), then in particular for each \(j\)
    \(A e_j \in \RR_{\geq 0} e_i \text{ for some } i \) and as \(A\) is
    invertible, setting \(j = \pi(i)\) yields a permutation of \(\{1,\dots,
    n\}\). Moreover,since \(A e_j \neq 0\) (\(A \) is invertible) there
    must be some \(a_{j} > 0\) so that \(Ae_j = a_j e_i\). One can now verify
    that \(A = P D\) where \(P\) is the permutation matrix associated to \(\pi\)
   and \(D = \diag(a_j)\) which matches the
    description of \(G(\relu , n  )\) from \cref{fig-intertwiner-groups}.

    For the ``moreover,'' we prove the contrapositive, namely that if \(v =
    (v_1, \dots, v_n) \in \RR^n\) has at least 2 non-0 coordinates \(v_i, v_j\),
    then  the \(G_{\relu}\)-orbit of the ray \(\RR_{\geq 0} v\) that \(v\) generates
    cannot be finite. Indeed, suppose \(t \geq 0\) and let \(D
    = \diag(1, \dots, 1, t, 1, \dots, 1 )\) (\(t\) in the \(i\)th position). The
    \(n \times 2\) matrix \( (v \vert Dv)\) has a \(2\times 2\) minor 
    \begin{equation}
        \begin{pmatrix}
            v_i & t v_i \\
            v_j & v_j
        \end{pmatrix} \text{ with deteriminant  } (1-t)v_i v_j \neq 0 \text{  as long as  } t \neq 1.
    \end{equation}
    Thus \(v\) and \(Dv\) are linearly independent, and hence define distinct
    rays, for all \(t \neq 1\). It follows that any set of rays stabilized by
    \(G(\relu , n  )\) that contains \(\RR_{\geq 0} v\) is uncountable.
\end{proof}

\subsection{Proof of \cref{lem:comm-w-sig}}

To give a rigorous proof we use induction on the depth \(k\); since this to some
extent obscures the main point, we briefly outline an informal proof: consider a
composition of 2 layers of the network \(f\) with weights \(W'\):
\begin{equation}
  \label{eq:informal}
  \sigma(A_{i+1}W_{i+1}\phi_{\sigma}(A_i)^{-1}\sigma(A_{i}W_{i}\phi_{\sigma}(A_{i-1})^{-1}
x + A_i b_i) + A_{i+1}b_{i+1}).
\end{equation}
Using the defining properties of \(G_{\sigma}\) and \(\phi_\sigma\), we can
extract \(A_i \) like 
\begin{equation}
  \label{eq:informal2}
  \sigma(A_{i}W_{i}\phi_{\sigma}(A_{i-1})^{-1}
x + A_i b_i) = \phi_\sigma(A_{i}) \sigma(W_{i}\phi_{\sigma}(A_{i-1})^{-1}
x +  b_i).
\end{equation}
The resulting copy of \( \phi_\sigma(A_{i})\) on the right hand side of
\cref{eq:informal2} is cancelled by the copy of \(\phi_{\sigma}(A_i)^{-1}\) right-multiplying
\(W_{i+1}\) in \cref{eq:informal}, so that \cref{eq:informal} reduces to 
\[\sigma(A_{i+1}W_{i+1}\sigma(W_{i}\phi_{\sigma}(A_{i-1})^{-1}
x + b_i) + A_{i+1}b_{i+1}).\]
In this way, in between any two layers  \(\ell_{i+1} \circ \ell_i\) the \(A_i\)
in \(A_{i}W_{i}\phi_{\sigma}(A_{i-1})^{-1}\) and the \(\phi_{\sigma}(A_i)^{-1}\)
in \(A_{i+1}W_{i+1}\phi_{\sigma}(A_i)^{-1}\) \emph{cancel}. However, the factors
\(\phi_\sigma(A_m)\) and \(\phi_{\sigma}(A_m)^{-1}\) appear on endpoints of the
truncated networks \(f_{\leq m}, f_{>m}\) and so they are not cancelled.

To keep track of \(W, W'\) while using the notation from \cref{sec:notation}, we
write 
\[ \ell_i(x, W) = \sigma(W_i x + b_i) \text{  and  } \ell_i(x, W') = \sigma(W'_i
x + b'_i)  \text{  for  } i < k\] and so on.

\begin{proof}
  By induction on \(k\), the depth of the network. The case \(k = 1\) is
  trivial, since there \(W' = W\) and there is nothing to prove. For \(k >1\) we
  consider 2 cases:

  \textbf{Case \(m =1\):} In this case let \(V = (W_2, b_2, \dots, W_k, b_k)\)
  and \(V' = ( A_2 W_2 , A_2 b_2 , A_3 W_3 \phi_\sigma(A_2)^{-1}, A_3 b_3, \dots,
  W_{k}\phi_{\sigma}(A_{k-1}^{-1}), b_{k})\), and let
  \[g(x, V) = \ell_{k-1}(x, V) \circ \cdots \circ \ell_1(x, V)  \text{  and  }
  g(x, V' ) = \ell_{k-1}(x, V') \circ \cdots \circ \ell_1(x, V') \] where
  \(\ell_i(x, V) = \sigma(W_{i+1}x +b_{i+1})\) and similarly for \(V'\). In
  other words, the weights \(V, V'\) and function \(g\) represent the
  architecture obtained by removing the earliest layer of \(f\). Then,
  \( f_{\leq 1}(x, W) = \sigma(W_1 x + b_1)  \)
  and on the other hand 
  \[ f_{\leq 1}(x, W') =  \sigma(A_1 W_1 x
  +A_1 b_1)  =  \sigma(A_1 (W_1 x + b_1)). \] 
  Using the identity \(\sigma(A_1 z) = \phi_\sigma(A_1) \sigma(z)\) for any \(z
  \in \RR^{n_1}\), we obtain 
  \[ f_{\leq 1}(x, W') =  \phi_\sigma(A_1) \sigma(W_1 x + b_1) =
  \phi_\sigma(A_1) \sigma(W_1 x + b_1). \] 
  This shows \(f_{\leq 1}(x, W') = \phi_\sigma(A_1) \circ f_{\leq 1}(x, W)\).
  Next, \(f_{>1}(x, W ) = g(x, V)\) but
  \emph{because} \(V'_1 = A_2 W_2\) whereas \(W'_2 = A_2 W_2
  \phi_\sigma(A_1)^{-1} \)
  \[ f_{>1}(x, W') = g(x, V') \circ \phi_\sigma(A_1)^{-1}. \]
  By induction on \(k\), we may assume \(g(x, V) = g(x, V')\) and it follows
  that \( f_{>1}(x, W') = f_{>1}(x, W )  \circ \phi_\sigma(A_1)^{-1} \)

  \textbf{Case \(m>1\):} Defining \(V, V'\) and \(g\) as above, we observe that 
  \begin{equation}
    \begin{split}
      f_{\leq i} (x, W) &= g_{\leq i-1}(\sigma(W_1x +b_1), V)  \text{  and  } \\
  f_{\leq i}(x, W' ) &= g_{\leq i-1}(\phi_\sigma(A_1)^{-1} \sigma(A_1W_1 x +A_1
  b_1) , V') = g_{\leq i-1}(\sigma(W_1 x +
  b_1) , V').
    \end{split}
  \end{equation}
  By induction on \(k\) we may assume \(g_{\leq i-1}(x, V') = \phi_\sigma(A_i)
  g_{\leq i-1}(x, V)\) and so 
  \[f_{\leq i}(x, W' ) = \phi_\sigma(A_i)g_{\leq i-1}(\sigma(W_1x +b_1), V) =
  \phi_\sigma(A_i) f_{\leq i} (x, W). \] 
  Finally, \(f_{>i}(x, W) = g_{>i-1}(x, V)\) and \(f_{>i}(x, W') = g_{>i-1}(x,
  V')\) and we may assume by induction on \(k\) that \( g_{>i-1}(x,
  V') =  g_{>i-1}(x, V) \circ \phi_\sigma(A_i)^{-1}\), hence \(f_{>i}(x, W') =
  f_{>i}(x, W) \circ \phi_\sigma(A_i)^{-1}\). 
\end{proof}

\subsection{Proof of \cref{thm:min-stitch}}

\begin{proof}
  Observe that by \cref{lem:comm-w-sig},
  \(\tilde{f}_{>l} = f_{>l} \circ \phi_{\sigma}(A_l^{-1})\). Hence
  \[ S(f, \tilde{f}, l, \varphi) = \tilde{f}_{>l}\circ\varphi \circ f_{\leq i} = f_{>l} \circ \phi_{\sigma}(A_l^{-1})
  \circ\varphi \circ f_{\leq l}. \]
  If \(\mathcal{S}\) contains \(\phi_\sigma(G_{\sigma_{n_l}})\) we
  may choose \(\varphi = \phi_{\sigma}(A_l)\) to achieve \(S(f,
  \tilde{f}, l, \varphi)= f\) as functions. Similarly, \(\tilde{f}_{\leq l} = \phi_\sigma(A_l) \circ
  f_{\leq l}\)
  so if \(\mathcal{S}\) contains \(\phi_\sigma(G_{\sigma_{n_l}})\) we may choose
  $\varphi = \phi_{\sigma}(A_l^{-1})$ to achieve \(S(f, \tilde{f}, l, \varphi) =
  \tilde{f}\) as functions. In either case \cref{eq:happy-stitch} holds.
\end{proof}

\subsection{Symmetries of the loss landscape}
\label{sec:sym-ll}
Given that the intertwiner group describes a large set of symmetries of a network, it is not surprising that it also provides a way of understanding the relationship between equivalent networks. \Cref{lem:comm-w-sig} has an interpretation in terms of the loss landscape of model architecture. For any layer $n_i$ in $f$, the action of $G_{\sigma_{n_i}}$ on the weight space $\WW$, translates to the obvious group action on the loss landscape.

\begin{corollary}
For any $1 \leq i \leq k$, the group $G_{\sigma_{n_i}}$ acts on $\WW$ and for any test set $D_t \subset X \times Y$, model loss on $D_t$ is invariant with respect to this action. More precisely,  if $\ell(\Phi(W),D)$ is the loss of $\Phi(W)$ on test set $D_t$, then for any $g \in G_{\sigma_{n_i}}$, $\ell(\Phi(W),D) = \ell(\Phi(gW),D)$.
\end{corollary}

\subsection{Comparing capacities of stitching layers via discretization}
\label{sec:epsnets}

Let \( \Mat_{n \times n}(\RR) \) denote the space of \( n \times n \) matrices. For \(  r = 1, \dots, n  \) let \( \Mat_{n \times n}^r(\RR) \subseteq \Mat_{n \times n}(\RR)\) denote the rank \( r  \) matrices, and let \(G_{\relu_{n}} \) be as described in \cref{fig-intertwiner-groups}. Suppose that each real dimension of \(\Mat_{n \times n}(\RR) \) is replaced by a discrete grid \(N(M, \epsilon) = \{-M + i \epsilon \, | \, i = 0, \dots, \lfloor \frac{2M}{\epsilon} \rfloor-1\}\) --- here \(\epsilon\) could represent the limits of numerical precision in a floating point number system, and \( M \) could represent the maximum numerical magnitude. The size of each such grid is \( \lfloor \frac{2M}{\epsilon} \rfloor \), and so the number of points in the resulting mesh grid \( N(M, \epsilon)^{n^2}  \subset \Mat_{n \times n}(\RR) \) is \( \lfloor \frac{2M}{\epsilon} \rfloor^{n^2} \). We now estimate the number of points of \(G_{\relu_{n}} \)  and \( \Mat_{n \times n}^r(\RR) \) in such a mesh grid.

\(G_{\relu_{n}} \) is a disjoint union of \(n!\) irreducible components, corresponding to the \(n!\) possible permutations \(P\) in \cref{fig-intertwiner-groups}. Each of these components is \(n\)-dimensional, corresponding to the fact that the factor \(D \) in \cref{fig-intertwiner-groups} is an arbitrary positive diagonal matrix. Hence we obtain 
\begin{equation}
    \lvert G_{\relu_{n}} \cap N(M, \epsilon)^{n^2} \rvert \approx n! \cdot (\frac{M}{\epsilon})^n.
\end{equation}
On the other hand, any matrix \(A \in \Mat_{n \times n}^r(\RR)\) can be written as \(A = UV\) where \(U \in \Mat_{n \times r}(\RR)\) and \(V \in \Mat_{r \times n}(\RR)\). These \(U \) and \(V\) are not unique: given any invertible \(r \times r\) matrix \(W \in GL_r(\RR)\), we have \(A = (UW)(W^{-1}V)\). From this we obtain the approximation\footnote{Here we ignore a significant subtlety: whether or not the multiplication map \( \Mat_{n \times r}(\RR) \times\Mat_{r \times n}(\RR) \to \Mat_{n \times n}(\RR) \) induces a map \( N(M, \epsilon)^{nr} \times N(M, \epsilon)^{rn} \to N(M, \epsilon)^{n^2} \) (with our naive setup it probably doesn't) and moreover whether the fibers of this map, which in the non-discretized case are generically isomorphic to \(GL_r(\RR)\), have intersection with \(N(M, \epsilon)^{nr} \times N(M, \epsilon)^{rn}\) of the expected size. We do not expect that these technical details will impact the takeaway of this analysis.}
\begin{align}
    \lvert \Mat_{n \times n}^r(\RR)  \cap N(M, \epsilon)^{n^2} \rvert &\approx \frac{\lvert \Mat_{n \times r}(\RR)  \cap N(M, \epsilon)^{nr} \rvert \cdot \lvert \Mat_{r \times n}(\RR)  \cap N(M, \epsilon)^{rn} \rvert}{\lvert GL_r(\RR)  \cap N(M, \epsilon)^{r^2} \rvert } \\
    &\approx \frac{(\frac{2M}{\epsilon})^{nr} (\frac{2M}{\epsilon})^{rn}}{(\frac{2M}{\epsilon})^{r^2}}\\
    & \approx (\frac{2M}{\epsilon})^{2 nr - r^2}.
\end{align}
It follows that 
\begin{align}
   &\log \lvert G_{\relu_{n}} \cap N(M, \epsilon)^{n^2} \rvert -  \log \lvert \Mat_{n \times n}^r(\RR)  \cap N(M, \epsilon)^{n^2} \rvert \\
   & = \log (n!) + n \log(\frac{M}{\epsilon}) - (2nr - r^2) \log (\frac{M}{\epsilon} + \log 2).
\end{align}
Ignoring the term \((2nr - r^2) \log 2\), which is independent of \(M, \epsilon \), we get the approximation
\begin{equation}
    \label{eq:log-card-diff}
    \log (n!) + n \log(\frac{M}{\epsilon}) - (2nr - r^2) \log (\frac{M}{\epsilon} + \log 2) \approx \log (n!) - ((2r-1)n - r^2) \log(\frac{M}{\epsilon}).
\end{equation}
Next, we make the coarse approximation 
\begin{equation}
    \log(n!) = \sum_{k=1}^n \log k \approx \int_1^n \log x \, dx = n\log n -n;
\end{equation}
with this approximation the expression of \cref{eq:log-card-diff} is approximated as 
\begin{equation}
    \log (n!) - ((2r-1)n - r^2) \log(\frac{M}{\epsilon}) \approx n\log n -n - ((2r-1)n - r^2) \log(\frac{M}{\epsilon}).
\end{equation}
From this we conclude that as long as 
\begin{enumerate}
    \item \label{item:enough-dims} \( r \geq 1\) (we actually already assumed this when defining \( \Mat_{n \times n}^r(\RR)\)) and
    \item \label{item:enough-bits} \(\frac{M}{\epsilon} \gg n\), which roughly says that the number of grid points per dimension is greater than the number of rows (equivalently columns) in \( \Mat_{n \times n}(\RR)\),
\end{enumerate}
\begin{align}
    n\log n -n - ((2r-1)n - r^2) \log(\frac{M}{\epsilon}) & \leq  n\log n -n - (n-1) \log(\frac{M}{\epsilon}) \text{ using \cref{item:enough-dims}} \\
    & = (n-1)(\log n -  \log(\frac{M}{\epsilon}) ) + \log n - n \\
    & < (n-1)(\log n -  \log(\frac{M}{\epsilon}) )  \text{ for \(n > 1\)} \\
    & < 0 \text{ using \cref{item:enough-bits}.}
\end{align}
The upshot is that our approximations and \cref{item:enough-bits,item:enough-dims} imply 
\begin{align}
    &\log \lvert G_{\relu_{n}} \cap N(M, \epsilon)^{n^2} \rvert -  \log \lvert \Mat_{n \times n}^r(\RR)  \cap N(M, \epsilon)^{n^2} \rvert  < 0, \text{  and hence }\\
    & \lvert G_{\relu_{n}} \cap N(M, \epsilon)^{n^2} \rvert < \lvert \Mat_{n \times n}^r(\RR)  \cap N(M, \epsilon)^{n^2} \rvert.
\end{align}

\subsection{Calculations related to dissimilarity measures (for
\cref{sec:shapemetrics})}

\begin{example}[courtesy of Derek Lim]
  \label{ex:oops} 
  Let \(x_1 = (1, 1), x_2 = (10, 1) \) and \(x_3 =  (0, 5) \), and let \(c =
  (10, -1, -1)\). Then letting \(\kappa(x_i, x_j) = \max(x_i \odot x_j)\), by
  direct computation
  \begin{equation}
    \sum_{i, j} c_i \kappa(x_i, x_j) c_j = -65 <0,
  \end{equation}
  and hence \(\kappa\) is not positive semi-definite.
\end{example}

\begin{proof}[Proof of \cref{lem:iterated-quotient}]
  By \cref{fig-intertwiner-groups} any \(A \in G_{\relu}\) can be factored as
  \(A=PD\) with \(P\) a permutation matrix and \(D\) a positive diagonal matrix,
  and we can obtain a similar factorization \(B = Q E\). Then 
  \[ \mu(XA, YB) = \mu(XPD, YQE) = \mu(XP, YQ)
  = \mu(X, Y) \] where the second equality uses the hypothesis
  that \(\mu\) is invariant to right multiplication by positive diagonal
  matrices, and the third uses the hypothesis that \(\mu\) is invariant to
  right multiplication by permutation matrices.
\end{proof}

\begin{lemma}
  \label{lem:max-symmetries}
  Suppose \(A\) is a matrix such that \(\max(Ax_1 \odot Ax_2) = \max(x_1 \odot x_2)\) for all \(x_1,
  x_2 \in \RR^d\). Then, \(A \) is of the form \(PD\) where \(P\) is a
  permutation matrix and \(D\) is diagonal with diagonal entries in \(\{\pm
  1\}\).
\end{lemma}

\begin{proof}
  We only need the special case where \( x = y \): observe that 
  \[ \max(x \odot x) = \max \{x_1^2, \dots, x_d^2\} = (\max \{\lvert x_1 \rvert,
  \dots, \lvert x_d \rvert\})^2 = \lvert x \rvert_\infty^2 \]
  This means that if \(\max(Ax_1 \odot Ax_2) = \max(x_1 \odot x_2)\), then
  \(A\) preserves the \(\ell_\infty\) norm on \(\RR^d\), hence in particular
  preserves the unit hypercube in \(\RR^d\), and it is known that symmetries of
  the hypercube have the form \(PD\) where \(P\) is a permutation
  matrix and \(D\) is diagonal with diagonal entries in \(\{\pm 1\}\) (see for example \cite[\S 5.9]{serre1977linear}).
\end{proof}

\subsection{Intertwiners and more general architecture features (justification of \cref{rmk:rn-failure})}

Here we briefly discuss how ubiquitous architecture features like batch normalization and residual connections interact with intertwiner groups. For simplicity in this section we only consider \( \sigma =\relu \).

\subsubsection{Batch normalization}

A batch normalization layer that takes as input \( X \in \RR^{b\cdot n} \) (where \(b\) is the batch size and \(n\) is the dimension of the layer) and returns
\[ \tilde{X}\diag(\tilde{X}^T \tilde{X})^{-1}\diag(\gamma) +\beta \text{  where } \tilde{X} = X - \mathbf{1}\mathbf{1}^T X \]  and where \(\gamma, \beta \in \RR^n\) are the ``gain'' and ``bias'' parameters of the batch normalization layer,
is \emph{invariant} under independent scaling of coordinates, that is transformations of the form \(X \gets XD \) where \(D\) is an \(n\times n\) positive diagonal matrix (see e.g. \cite{badrinarayananUnderstandingSymmetriesDeep2015}). Hence a \(k\)-layer \(\relu\) MLP as in \cref{sec:notation} enhanced with batch normalization (\emph{pre}-activation, as
is standard) is invariant under the action of the slightly larger group \(\prod_{l=1}^{k-1} (\RR^{n_l}_{>0}\rtimes
G_{\relu_{n_l}})\), where the action is given by\footnote{As is best practice we omit the biases on the linear layers \(\ell_l\) for \( l<k\), since they would be redundant now that we have biases on batch norm layers.} 
\begin{equation}
    \label{eq:bn}
    \begin{split}
        &(c_1, A_1, \dots, c_{k-1}, A_{k-1}) \cdot (W_1, \gamma_1, \beta_1, \dots, W_{k-1},
        \gamma_{k-1}, \beta_{k-1}, W_{k}, b_{k}) \\ 
        &= (A_1 W_1, \pi(A_1)\diag c_1 \gamma_1 , \pi(A_1)\diag c_1 \beta_1, \\
        & A_2 W_2 (\pi(A_1)\diag c_1)^{-1}, \pi(A_2)\diag c_2  \gamma_2, \pi(A_2)\diag c_2 \beta_2, \\
        & \dots, A_{k-1} W_{k-1} (\pi(A_{k-2})\diag c_{k-2})^{-1}, \pi(A_{k-1})\diag c_{k-1}  \gamma_{k-1}, \pi(A_{k-1})\diag c_{k-1} \beta_{k-1}, \\
        & W_{k}(\pi(A_{k-1})\diag c_{k-1})^{-1}, b_{k}).
    \end{split}
\end{equation}
Here,
\begin{itemize}[nosep]
    \item \(c_l \in \RR^{n_l}_{>0}\) and \( A_l \in G_{\relu_{n_l}}\), for all \(  l = 1, \dots, k-1\)
    \item \(\pi: G_{\relu_{n_l}} \to \Sigma_{n_l}\) is the homomorphism setting
    the positive entries to 1.
\end{itemize}
The key point is that we get another factor of \(\RR^{n_l}_{>0}\) at each layer. We also note that the space of matrices of the form \( \pi(A_l)\diag c_l\) is, incidentally, exactly \( G_{\relu_{n_l}} \), and that using \cref{eq:bn} one can generalize \cref{lem:comm-w-sig,thm:min-stitch} to the case of networks with batch normalization. 

\subsubsection{Residual connections}

We expand on  \cref{rmk:rn-failure}  and explain what exactly transpires with residual connections below.

Suppose we have a \(k\)-layer MLP as in \cref{sec:notation} (again for simplicity with
\(\sigma = \relu\)), together with residual connections between a set of layers \(R = \{r_1, \dots, r_m\} \subseteq \{2,\cdots, k-1 \} \)\footnote{In particular we assume there is at least one linear layer before the first outgoing/after the last incoming residual connection, as occurs in e.g. ResNets.}:
\begin{equation}
    \begin{tikzcd}[column sep = small]
        \RR^{n_0} \arrow[r,"\sigma W_1"] & \cdots \arrow[r, "\sigma W_{r_{i-1}}"] &\RR^{n_{r_{i-1}}} \arrow[r, "\sigma W_{r_{i-1}+1}"] \arrow[rr, bend right, "\mathrm{id}"] & \cdots \arrow[r, "W_{r_{i}}"] & \RR^{n_{r_{i}}} \arrow[r, "\sigma"] & 
        \cdots  
        \RR^{n_{k-1}} \arrow[r, "W_{k}"] & \RR^{n_{L+1}}
    \end{tikzcd}
\end{equation}
(for legibility biases \(b_l\) are suppressed). In addition we assume that the depth of each residual block is some fixed, that is \(r_i - r_{i-1} = b = \) constant for all \(i\). 

First, we claim that a  \((A_1, \dots A_{k-1}) \in \prod_{l = 1}^L G_{\relu_{n_l}} \) stabilizes the function \(f\) \emph{if} (not claiming if and only if) \(A_{r_i} = A_{r_j}\) for all \(r_i, r_j \in R\). To see this suppose \(g_i(x, W)\), \(i = 1, \dots, m\) are the depth \(b\) feedforward networks of the residual blocks, so that the \( f_{\leq r_i} \) is given by 
\begin{equation}
    \label{eq:basic-block}
    \begin{split}
        f_{\leq r_i}(x, W) &= f_{\leq r_{i-1}}(x, W) + g_i(f_{\leq r_{i-1}}(x, W) , W) \text{  where  }\\ g_i(z, W) &= \sigma(W_{r_{i}+b} \sigma( \cdots \sigma(W_{r_{i-1}+1} z)\cdots )).
    \end{split}
\end{equation}
Assuming by induction on \(i\) that \cref{lem:comm-w-sig} applies \emph{at the residual connections in \(R\)} we note that with weights \(W' \) \cref{eq:basic-block} turns into 
\begin{equation}
    \label{eq:basic-block2}
    \begin{split}
        f_{\leq r_i}(x, W') &= f_{\leq r_{i-1}}(x, W') + g_i(f_{\leq r_{i-1}}(x, W') , W') \\
    &= A_{r_{i-1}}f_{\leq r_{i-1}}(x, W) + g_i(A_{r_{i-1}} f_{\leq r_{i-1}}(x, W), W').
    \end{split}
\end{equation}
Note that \cref{lem:comm-w-sig} applies directly to the \(g_i\), so we may compute \( g_i(z, W') = A_{r_i}g_i(A_{r_{i-1}}^{-1}z, W) \). Hence 
\begin{equation}
    \label{eq:basic-block3}
    \begin{split}
        f_{\leq r_i}(x, W') &= A_{r_{i-1}}f_{\leq r_{i-1}}(x, W) + A_{r_i}g_i(A_{r_{i-1}}^{-1}A_{r_{i-1}} f_{\leq r_{i-1}}(x, W), W) \\
        &= A_{r_{i-1}}f_{\leq r_{i-1}}(x, W) + A_{r_i}g_i( f_{\leq r_{i-1}}(x, W), W).
    \end{split}
\end{equation}
and we see that the only way \( f_{\leq r_i}(x, W') = B f_{\leq r_i}(x, W)  \) for some matrix \(B\) is if \(  A_{r_{i-1}} =  A_{r_{i}} \), proving our claim. This also shows that if \(  A_r\) denotes the common value of the \(A_{r_{i}}\) for \(r_i \in R\), we have \( f_{\leq r_i}(x, W') = A_r f_{\leq r_i}(x, W)  \) for all \(i\). It is also true that \(  f_{>r_i}(x, W') = f_{>r_i}(A_r^{-1}x, W) \): observe that 
\begin{equation}
    \label{eq:bb4}
    f_{>r_i}(x, W) = f_{>r_{i+1}}(  x + g_{i+1}( x, W) , W).
\end{equation}
By \emph{descending} induction on \(k\), we may assume \(f_{r_{i+1}}(x, W') = f_{r_{i+1}}(A_{r}^{-1}x, W) \), and as above \(g_{i+1}(z, W') = A_{r}g_{i+1}(A_r^{-1}z, W)\), so that with weights \(W'\) \cref{eq:bb4} becomes 
\begin{equation}
    \label{eq:bb5}
    \begin{split}
        f_{>r_i}(x, W') &= f_{>r_{i+1}}(  x + g_{i+1}( x, W'), W' )\\
        &= f_{>r_{i+1}}( A_{r}^{-1} (x + A_r g_{i+1}( A_r^{-1}x, W)), W) \\
        &= f_{>r_{i+1}}( A_{r}^{-1} x + g_{i+1}( A_r^{-1}x, W), W) = f_{>r_i}(A_{r}^{-1} x, W).
    \end{split} 
\end{equation}
as claimed. 

Finally, we describe how stitching fails inside a residual block. Suppose we use weights \(W_l\) for \(l \leq r_{i} +j \) where \(0< j < b\)) (recall  \(b = \) depth of our basic block) and weights \(W'_l\) for \(l > r_{i} +j \). The resulting stitched network is (attempting to use indentation to increase legibility)
\begin{equation}
    \label{eq:stitch-in-block}
     \begin{split}
         f_{r_{i+1}}(& \\
         &f_{\leq r_i}(x, W) + g_{>j}( \\
         &\quad \quad \quad \quad \quad \quad \quad \quad \varphi  g_{\leq j}(f_{\leq r_i}(x, W), W), \\
         &\quad \quad \quad \quad \quad \quad \quad \quad W'), \\
         & W') .
     \end{split}
\end{equation}
By \cref{lem:comm-w-sig} \(  g_{>j}(z, W' ) = A_r g_{>j}(A_{r_{i} +j}^{-1}z, W)\), and we have shown \(f_{>r_{i+1}}(z, W') = f_{>r_{i+1}}(A_r^{-1}z, W)\). Combining these facts \cref{eq:stitch-in-block} becomes
\begin{equation}
    \label{eq:stitch-in-block2}
     \begin{split}
         f_{r_{i+1}}(& \\
         &A_r^{-1} f_{\leq r_i}(x, W) + A_r^{-1} A_r g_{>j}( \\
         &\quad \quad \quad \quad \quad \quad \quad \quad A_{r_{i} +j}^{-1} \varphi  g_{\leq j}(f_{\leq r_i}(x, W), W), \\
         &\quad \quad \quad \quad \quad \quad \quad \quad W), \\
         & W) .
     \end{split}
\end{equation}
\emph{Even after cancelling to remove the \(A_r^{-1} A_r \) and in the ideal case where \(\varphi = A_{r_{i} +j}\), we are left with an extra factor of \(A_r^{-1}\) left multiplying \(f_{\leq r_i}(x, W) \)}:
\begin{equation}
    \label{eq:stitch-in-block3}
     \begin{split}
         f_{r_{i+1}}(& \\
         &A_r^{-1} f_{\leq r_i}(x, W) + g_{>j}( \\
         &\quad \quad \quad \quad \quad \quad \quad \quad   g_{\leq j}(f_{\leq r_i}(x, W), W), \\
         &\quad \quad \quad \quad \quad \quad \quad \quad W), \\
         & W) .
     \end{split}
\end{equation}

\section{Network dissection details}
\label{sec:net-d-extra}
Here we include some supplementary results and experiments for examining coordinate basis interpretability with network dissection \cite{Bau2017NetworkDQ}. 

\subsection{Network dissection methodology}\label{sec:net-d-methodology}
The Broden concept dataset, compiled by Bau et al.,
contains pixel-level annotations for hierarchical concepts including colors,
textures, objects, and scenes. For every channel activation, network dissection
assesses the binary segmentation performance with every visual concept from
Broden. The method first computes the channel activation for every Broden image.
The distribution of the activations for the channel is used to binarize the
activation (where we threshold by the top $0.5\%$ of all activations for the
channel) to define a segmentation mask for the channel activation which is
interpolated to the size of the input image. If the Intersection over Union
(IoU) of the activation segmentation mask and a concept mask is high enough
(namely where $\text{IoU} > 0.04$), network dissection labels the activation an
interpretable detector for the concept. 

\subsection{Additional experiments}\label{sec:net-d-further}
 \cref{fig:net-d3} breaks down the categories of interpretable units for the models and rotations examined in \cref{fig:net-d1}. The number of interpretable units tends to be dominated by the object and scene concept detectors for the ResNet-50 and the ConvNeXt models. 

\begin{figure}
  \centering
  \includegraphics[width=5.5in]{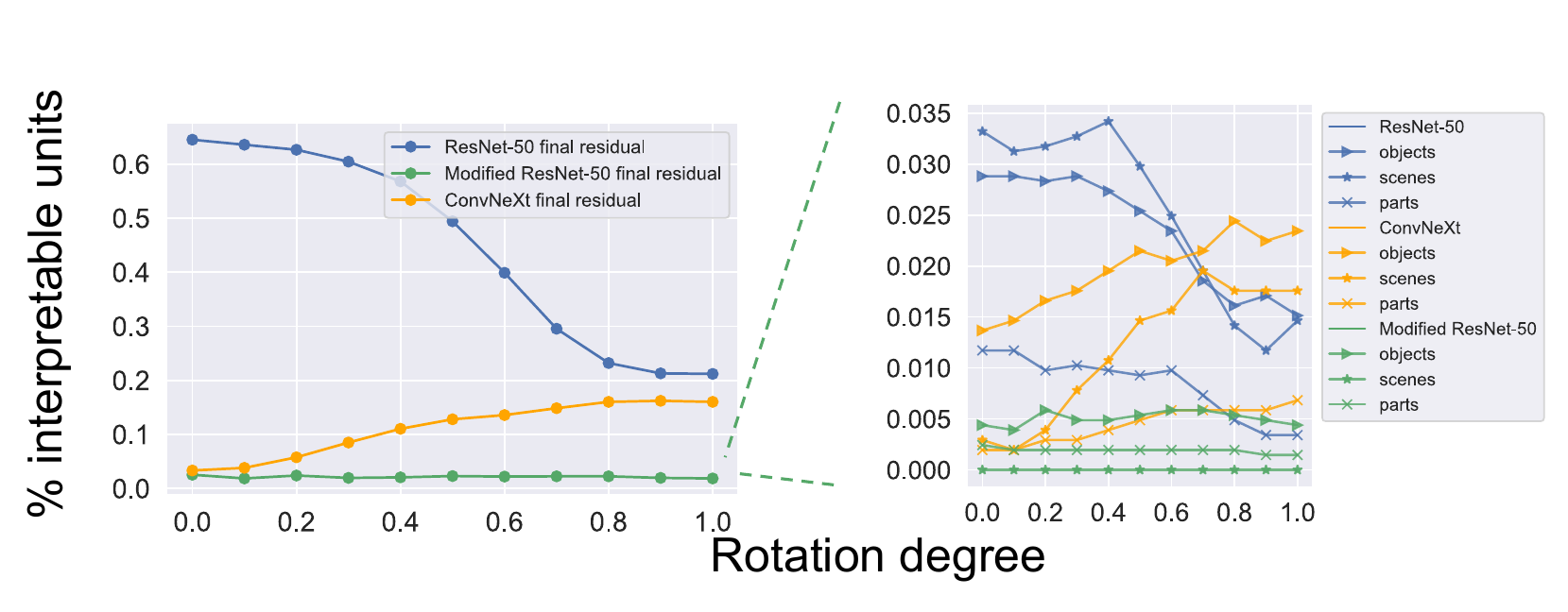}
  \caption[]{Supplement for the same network dissection experiment for the ResNet-50, modified ResNet-50, and ConvNeXt models in \cref{fig:net-d1}, highlighting the categories of interpretable units for each model and basis on the right. The y-axis for the plot on the right is distinct concepts.}\label{fig:net-d3}
\end{figure}

\begin{figure}
  \centering
  \includegraphics[width=4in]{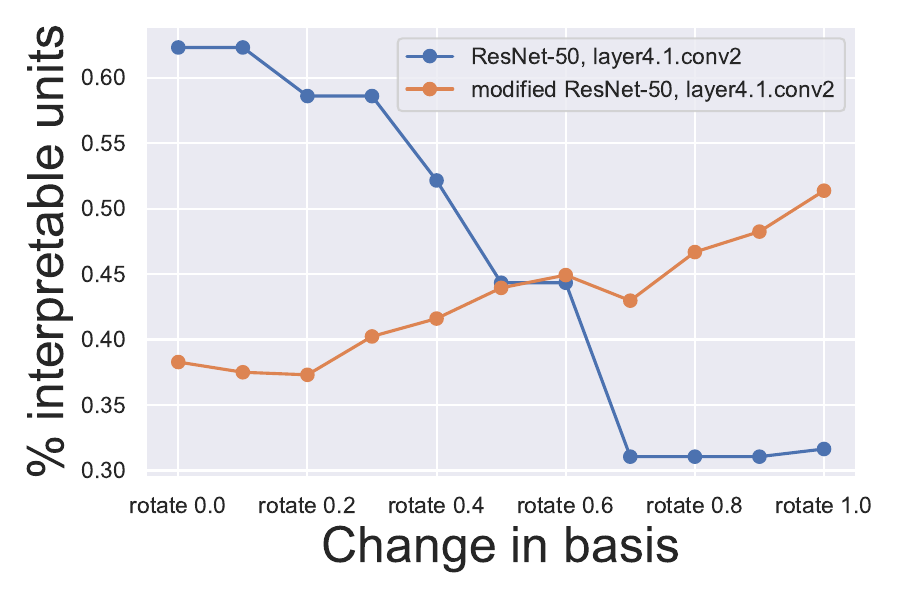}
  \caption[]{Fraction of network dissection interpretable units under rotations of
  the representation basis for a ResNet-50 and a modified ResNet-50 without an activation function on the residual output.}\label{fig:net-d2}
\end{figure}

We also perform an analogous network dissection experiment to \cref{sec:net-dissect} within the residual blocks for the normal and modified ResNet-50 in \cref{fig:net-d2}. Like the ConvNeXt model in \cref{fig:net-d1} we find that, surprisingly, the percentage of interpretable units actually tends to increase away from the activation basis.

Per-concept breakdowns produced by network dissection for three different rotation powers in the experiment in \cref{fig:net-d1} are given for the ResNet-50 in \cref{fig:r50-concepts}, the modified ResNet-50 (without a ReLU activation function on the residual output) in \cref{fig:r50-modified-concepts}, and the ConvNeXt in in \cref{fig:convnext-concepts}. We also include the units with the highest concept intersection over union scores for the the three representative rotations for the ResNet-50 in \cref{fig:r50-highest-iou}, the modified ResNet-50 (without a ReLU activation function on the residual output) in \cref{fig:r50-modified-highest-iou}, and the ConvNeXt in in \cref{fig:convnext-highest-iou}

\begin{figure}
  \includegraphics[width=5.5in]{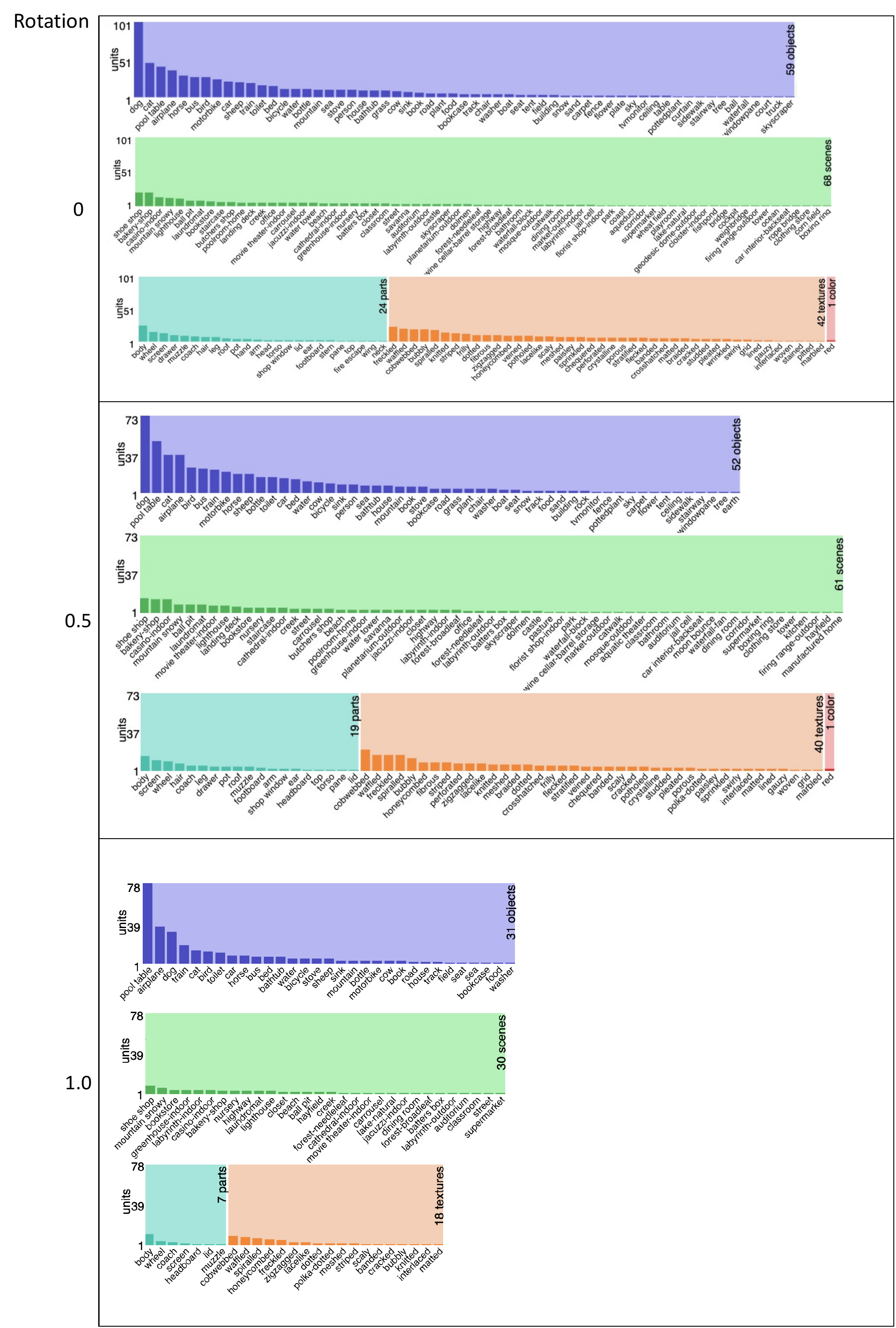}
  \caption[]{Network dissection bar graph of categories of unique concepts at three different rotation powers for the ResNet-50 model in \cref{fig:net-d1}.}\label{fig:r50-concepts}
\end{figure}

\begin{figure}
  \includegraphics[width=6.5in]{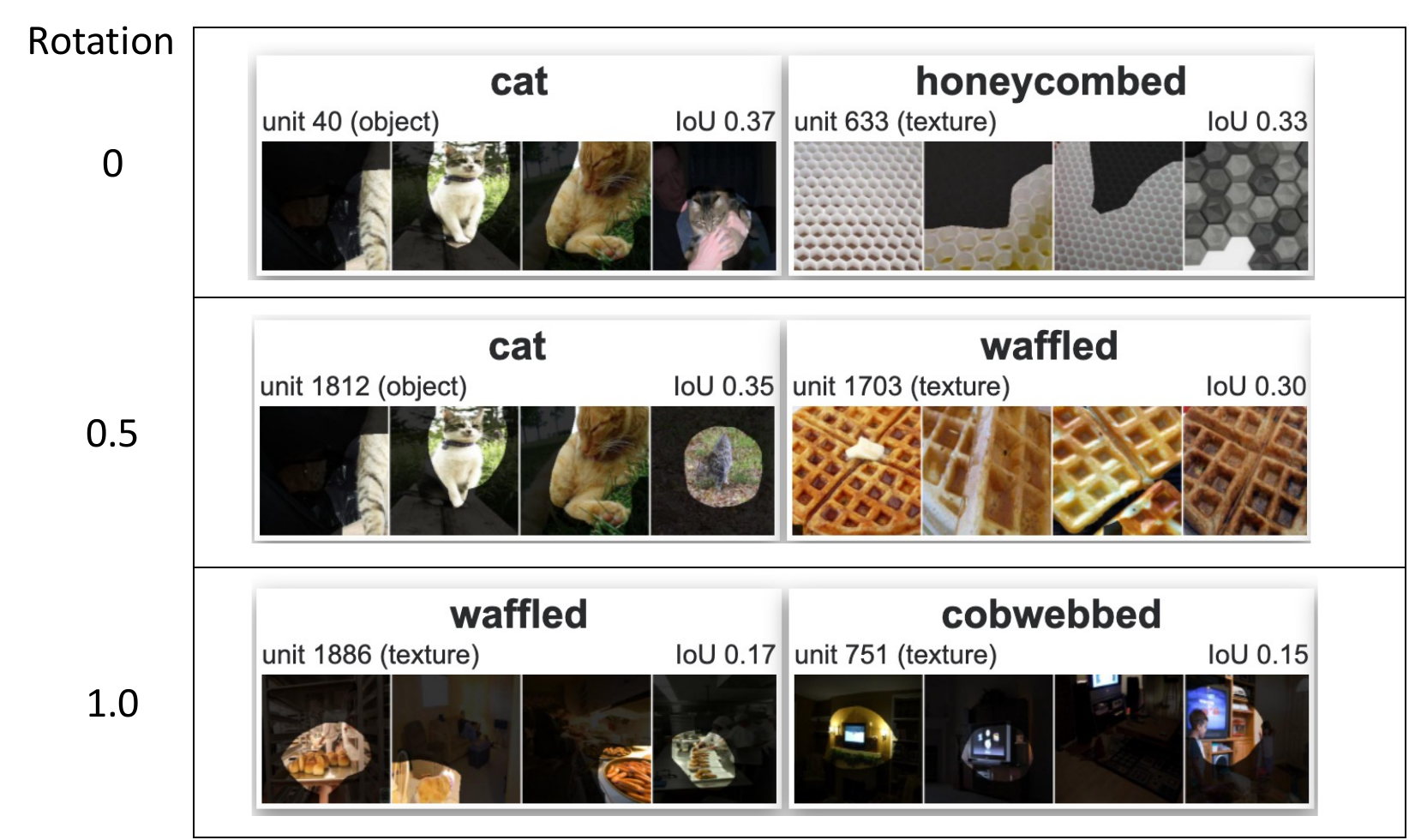}
  \caption[]{Top two highest scoring units for network dissection at three different rotation powers for the ResNet-50 model in \cref{fig:net-d1}.}\label{fig:r50-highest-iou}
\end{figure}

\begin{figure}
  \centering
  \includegraphics[width=3in]{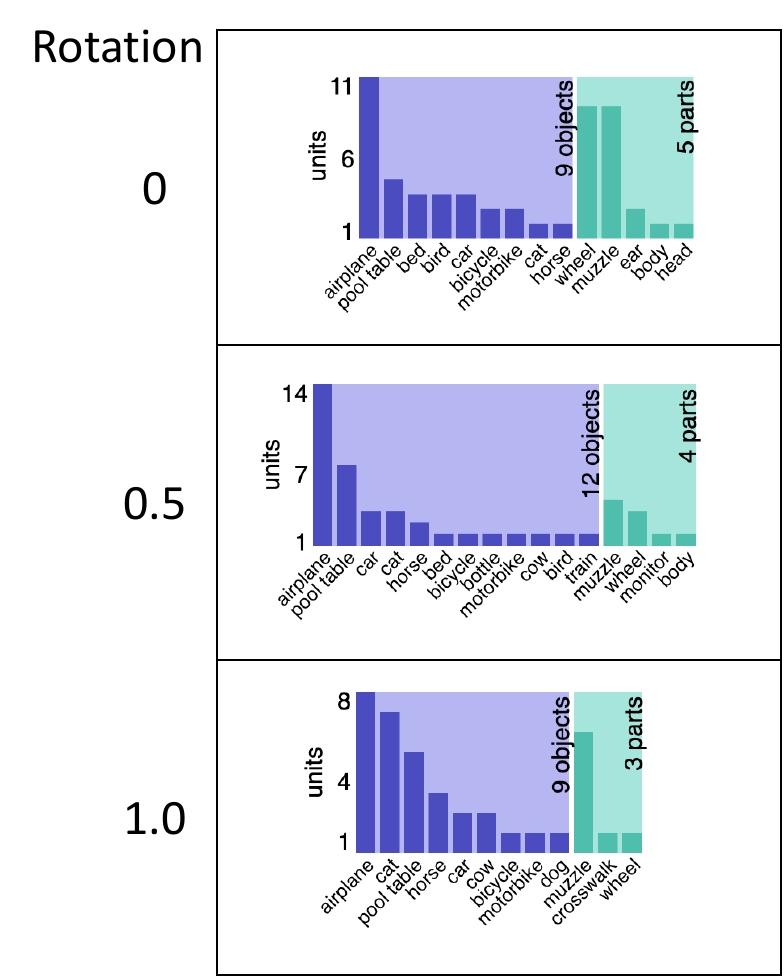}
  \caption[]{Network dissection bar graph of categories of unique concepts at three different rotation powers for the modified ResNet-50 model (without a ReLU activation function on the residual output) in \cref{fig:net-d1}.}\label{fig:r50-modified-concepts}
\end{figure}

\begin{figure}
  \centering
  \includegraphics[width=6.5in]{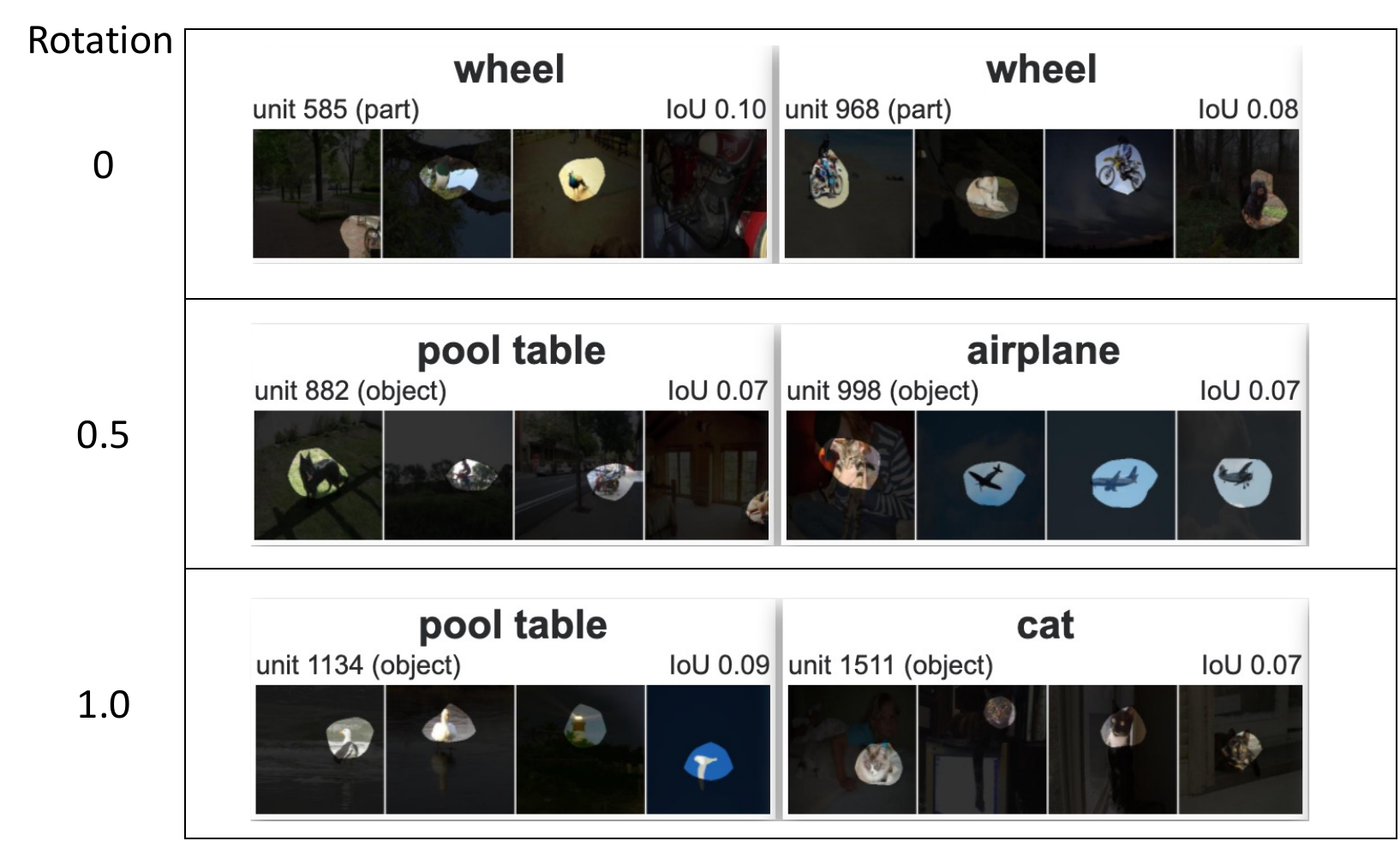}
  \caption[]{Top two highest scoring units for network dissection at three different rotation powers for the modified ResNet-50 model (without a ReLU activation function on the residual output) in \cref{fig:net-d1}.}\label{fig:r50-modified-highest-iou}
\end{figure}

\begin{figure}
  \centering
  \includegraphics[width=5in]{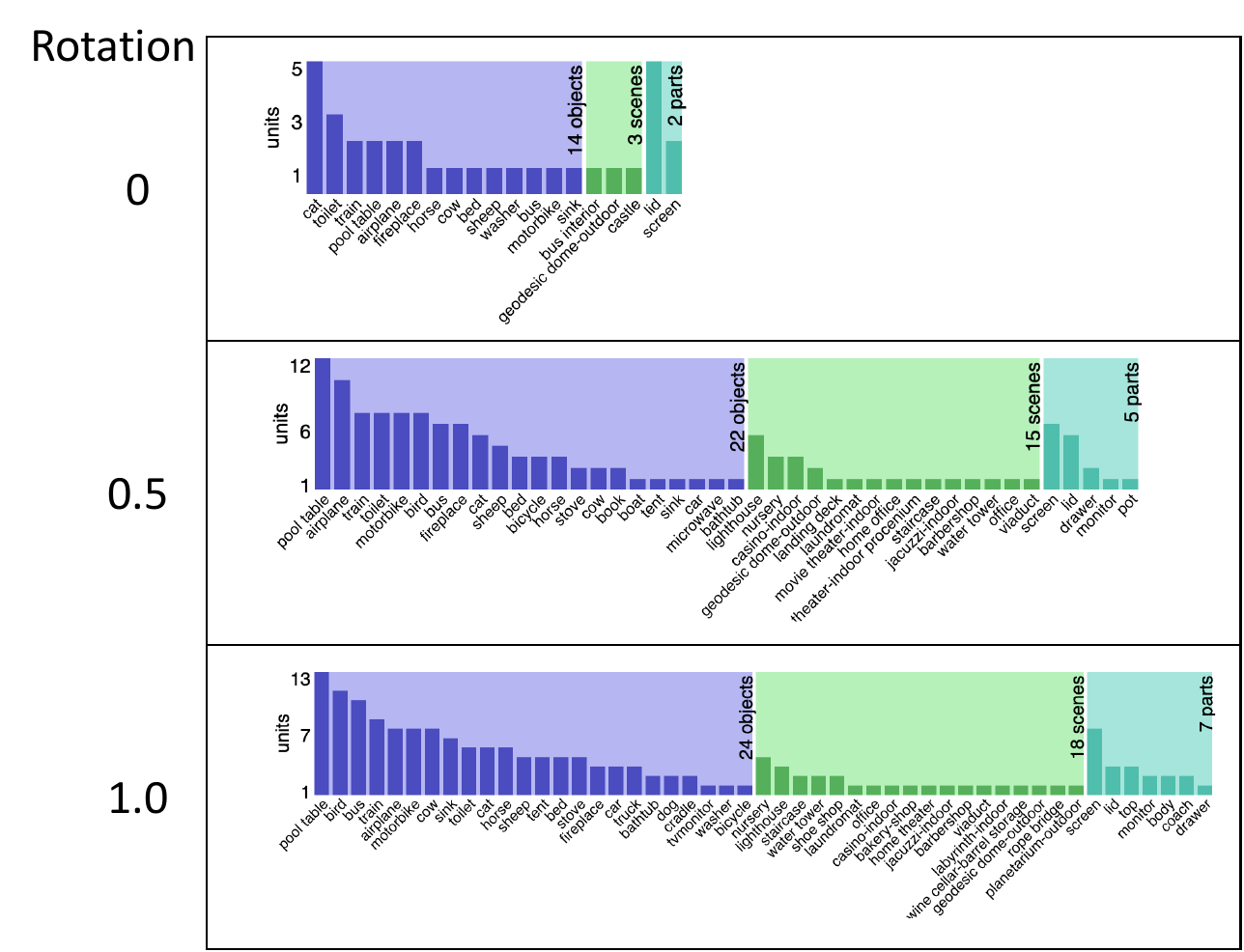}
  \caption[]{Network dissection bar graph of categories of unique concepts at three different rotation powers for the ConvNeXt model in \cref{fig:net-d1}.}\label{fig:convnext-concepts}
\end{figure}

\begin{figure}
  \centering
  \includegraphics[width=6.5in]{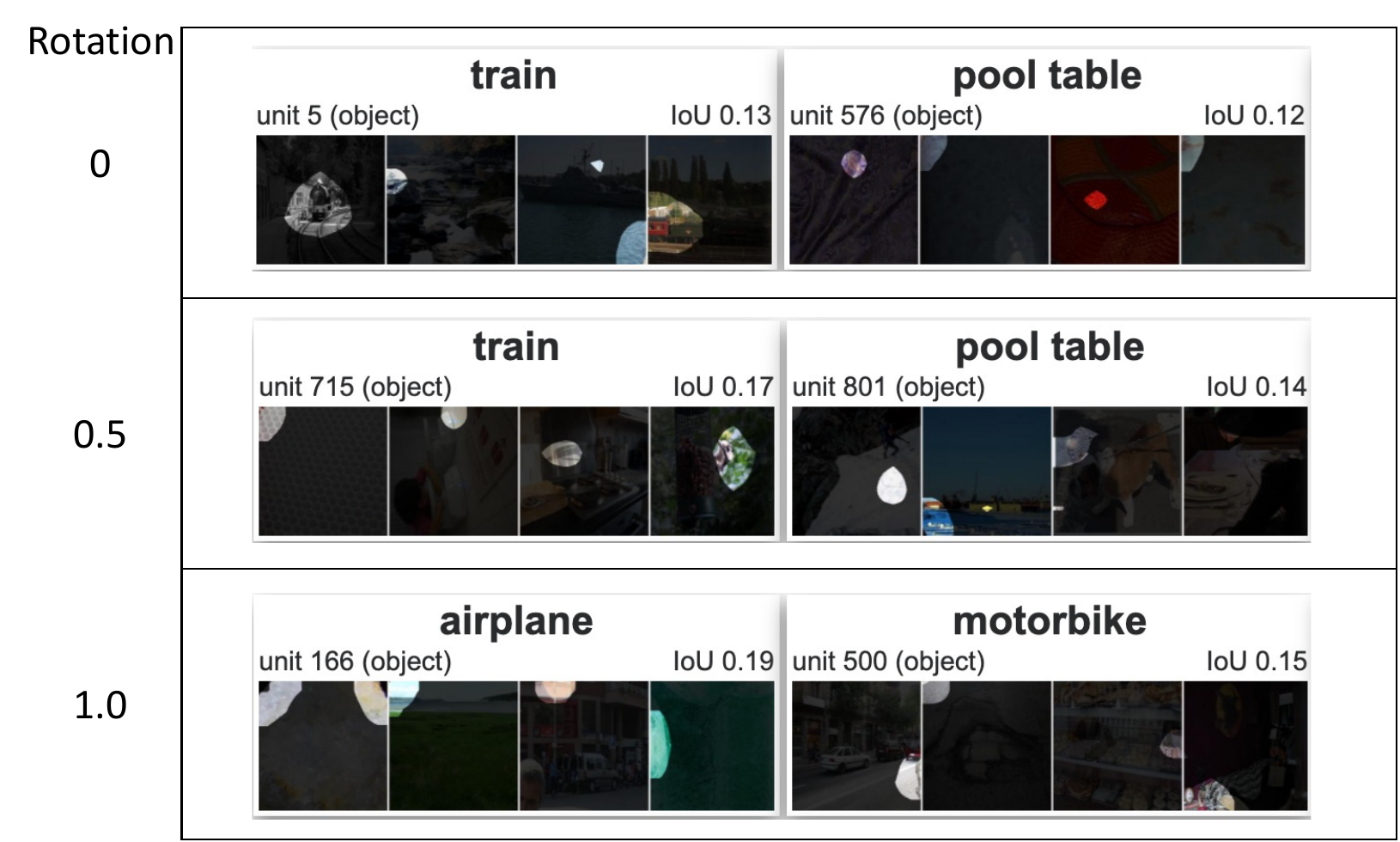}
  \caption[]{Top two highest scoring units for network dissection at three different rotation powers for the ConvNeXt model in \cref{fig:net-d1}.}\label{fig:convnext-highest-iou}
\end{figure}

\subsection{Model training details}
\label{sec:net-d-train}
We train a ResNet-50 without ReLU (or any activation function) on the residual blocks in PyTorch using \cite{leclerc2022ffcv} on ImageNet \cite{deng2009imagenet}. We train with SGD with momentum for 88 epochs with a cyclic learning rate rate of $1.7$, label smoothing of $0.1$, a batch size of 512, and weight decay of $10^{-4}$. The model achieves 76.1\% top-1 accuracy. We use pretrained weights for the ResNet-50 (unmodified) and ConvNeXt models from \cite{marcel2010torchvision} and \cite{rw2019timm} respectively.

\section{Dataset Details}
\label{sect-dataset-details}

CIFAR-10: CIFAR-10 is covered by the MIT License (MIT). We use canonical train/test splits (imported using \href{https://pytorch.org/vision/stable/index.html}{torchvision}).

Broden: the \href{https://github.com/CSAILVision/NetDissect}{code used to generate the dataset} is covered by the MIT license.

ImageNet: ImageNet is covered by CC-BY 4.0. We use canonical train/test splits.

\begin{figure}[h]
    \begin{subfigure}{0.5\linewidth}
    \centering
    \includegraphics[width=\linewidth]{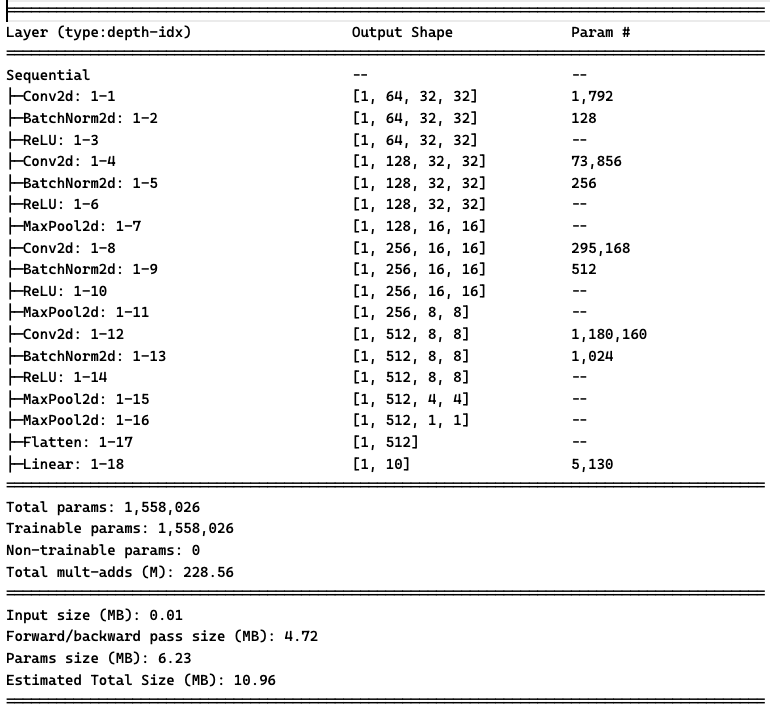}
    \caption{Myrtle CNN architecture, summary courtesy of \href{https://github.com/tyleryep/torchinfo}{torchinfo}.}\label{fig:arch-mCNN}
    \end{subfigure}
    \begin{subfigure}{0.5\linewidth}
    \centering
    \includegraphics[width=\linewidth]{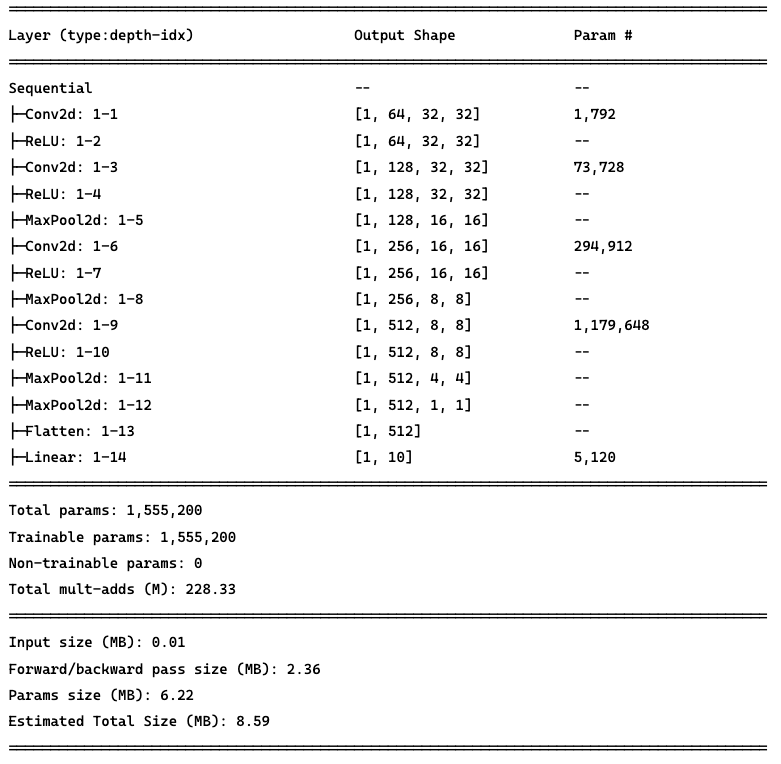}
    \caption{Myrtle CNN architecture without batch norm (only used in \cref{sec:sanity-rotations}), summary courtesy of \href{https://github.com/tyleryep/torchinfo}{torchinfo}.}
    \end{subfigure}
\end{figure}

\begin{figure}[h]
   \begin{subfigure}{0.5\linewidth}
    \centering
    \includegraphics[width=\linewidth]{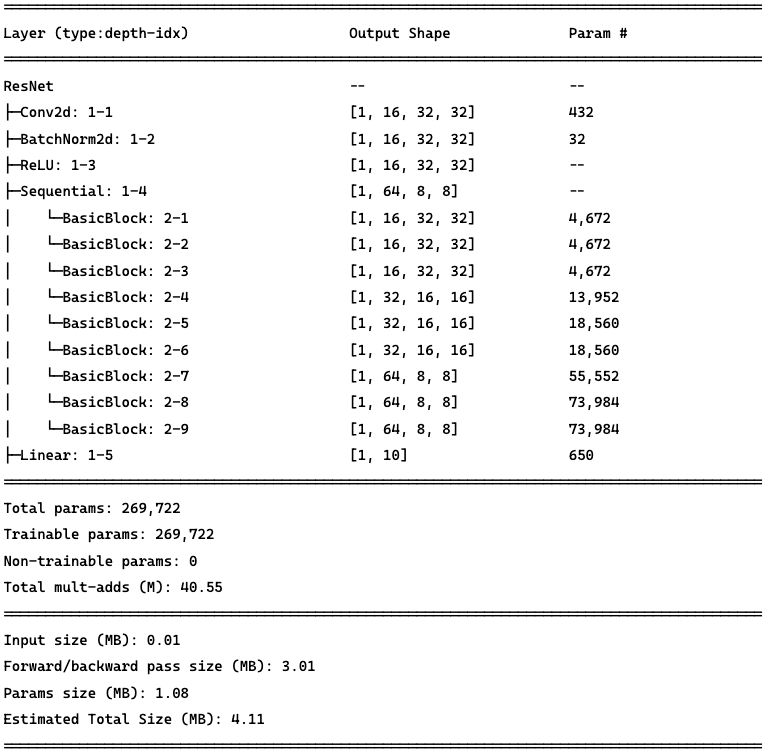}
    \caption{Our ResNet20 architecture, summary courtesy of \href{https://github.com/tyleryep/torchinfo}{torchinfo}.}\label{fig:arch-rn20}
   \end{subfigure}
   \begin{subfigure}{0.5\linewidth}
    \centering
    \includegraphics[width=\linewidth]{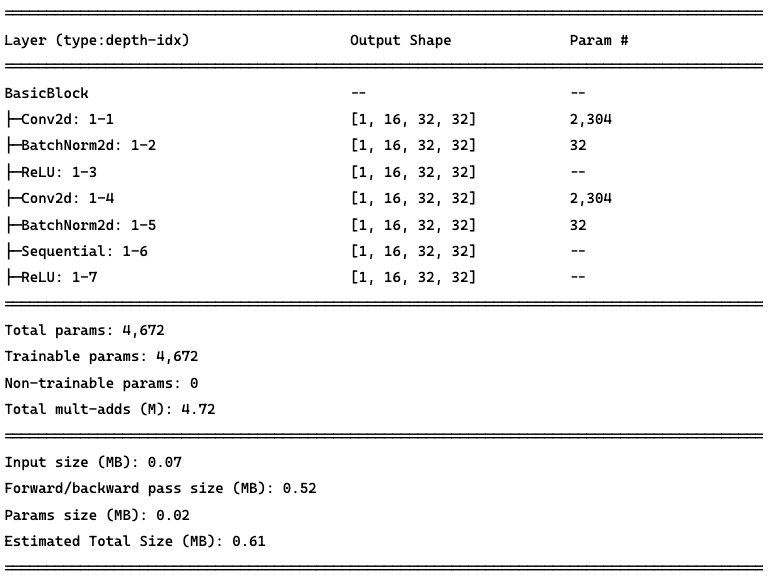}
    \caption{Internals of the 1st \lstinline{BasicBlock} (the sequential contains the residual connection).}
   \end{subfigure}
    
\end{figure}

\begin{figure}[h]
   \begin{subfigure}{0.5\linewidth}
    \centering
    \includegraphics[width=\linewidth]{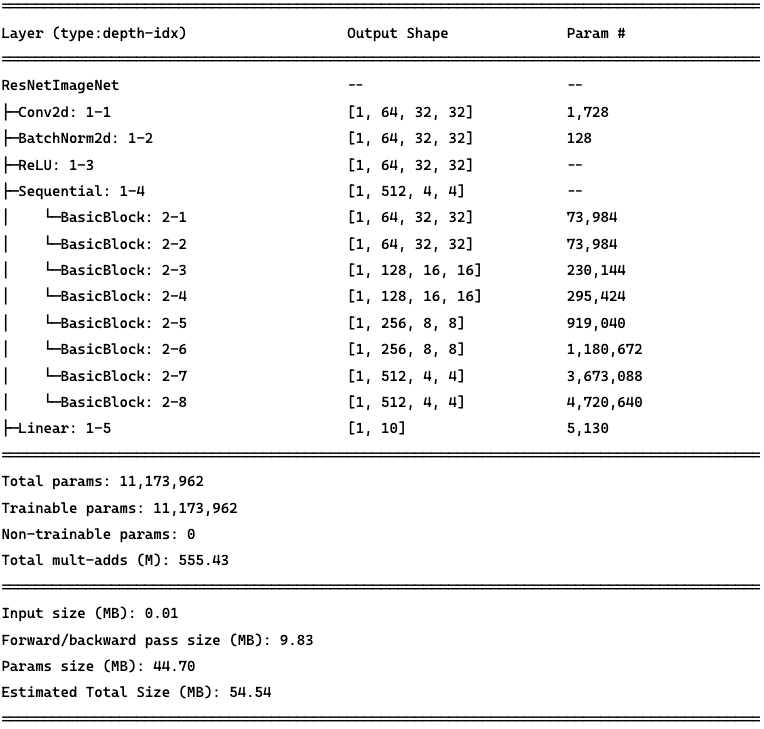}
    \caption{Our ResNet18 architecture, summary courtesy of \href{https://github.com/tyleryep/torchinfo}{torchinfo}.}\label{fig:arch-rn18}
   \end{subfigure}
   \begin{subfigure}{0.5\linewidth}
    \centering
    \includegraphics[width=\linewidth]{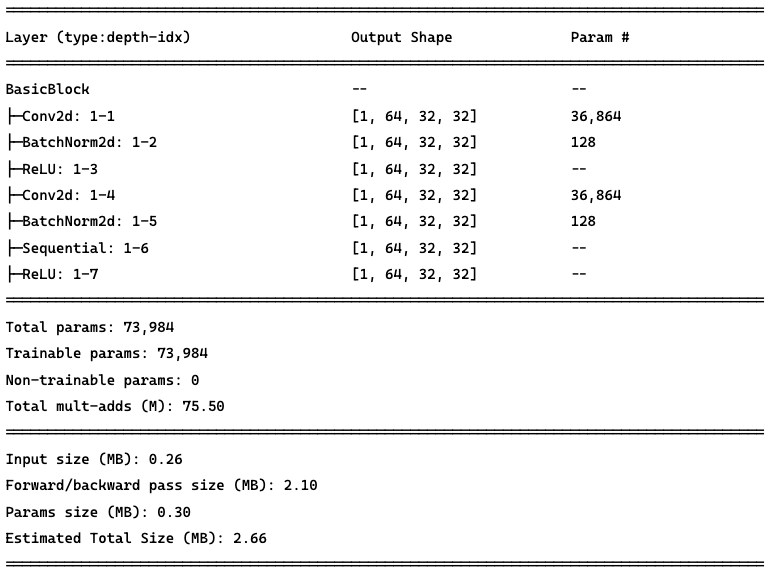}
    \caption{Internals of the 1st \lstinline{BasicBlock} (the sequential contains the residual connection).}
   \end{subfigure}
    
\end{figure}

\end{document}